\documentclass{article}

\usepackage{microtype}
\usepackage{graphicx}
\usepackage{subfigure}
\usepackage{booktabs} 
\usepackage{url}
\usepackage{multirow}
\usepackage{nicefrac}
\usepackage{lscape}
\usepackage{enumitem}
\usepackage{bbm}

\usepackage{amsthm}

\newtheorem{theorem}{Theorem}

\newtheorem{proposition}[theorem]{Proposition}

\usepackage{amsmath,amsfonts,bm}
\usepackage{mathtools}
\usepackage{amssymb}

\DeclareMathAlphabet{\mathsfit}{\encodingdefault}{\sfdefault}{m}{sl}
\SetMathAlphabet{\mathsfit}{bold}{\encodingdefault}{\sfdefault}{bx}{n}

\def\gI{{\mathcal{I}}}

\newcommand{\E}{\mathbb{E}}

\newcommand{\R}{\mathbb{R}}

\DeclareMathOperator*{\argmax}{arg\,max}

\usepackage{tikz}
\usetikzlibrary{bayesnet}
\usetikzlibrary{intersections, calc, arrows}
\usetikzlibrary{positioning}

\DeclarePairedDelimiterX{\infdivx}[2]{(}{)}{%
  #1\;\delimsize|\delimsize|\;#2%
}

\usepackage{hyperref}
\usepackage{xcolor}

\usepackage[accepted]{icml2021}

\icmltitlerunning{Policy Information Capacity}

\begin{document}

\twocolumn[
\icmltitle{Policy Information Capacity: \\ Information-Theoretic Measure for Task Complexity \\ in Deep Reinforcement Learning}



\icmlsetsymbol{equal}{*}

\begin{icmlauthorlist}
\icmlauthor{Hiroki Furuta}{ut}
\icmlauthor{Tatsuya Matsushima}{ut}
\icmlauthor{Tadashi Kozuno}{ua}
\icmlauthor{Yutaka Matsuo}{ut}\\
\icmlauthor{Sergey Levine}{goo}
\icmlauthor{Ofir Nachum}{goo}
\icmlauthor{Shixiang Shane Gu}{goo}
\end{icmlauthorlist}

\icmlaffiliation{ut}{The University of Tokyo, Tokyo, Japan}
\icmlaffiliation{ua}{University of Alberta, Edmonton, Canada}  
\icmlaffiliation{goo}{Google Research, Mountain View, USA}

\icmlcorrespondingauthor{Hiroki Furuta}{furuta@weblab.t.u-tokyo.ac.jp}

\icmlkeywords{Policy Information Capacity, Policy-Optimal Information Capacity, Task Complexity, Empowerment, Mutual Information, Reinforcement Learning}

\vskip 0.3in
]



\printAffiliationsAndNotice{}  

\begin{abstract}
Progress in deep reinforcement learning (RL) research is largely enabled by benchmark task environments.
However, analyzing the nature of those environments is often overlooked. In particular, we still do not have agreeable ways to measure the difficulty or solvability of a task, given that each has fundamentally different actions, observations, dynamics, rewards, and can be tackled with diverse RL algorithms.
In this work, we propose \textit{policy information capacity (PIC)} -- the mutual information between policy parameters and episodic return -- and \textit{policy-optimal information capacity (POIC)} -- between policy parameters and episodic optimality -- as two environment-agnostic, algorithm-agnostic quantitative metrics for task difficulty.
Evaluating our metrics across toy environments as well as continuous control benchmark tasks from OpenAI Gym and DeepMind Control Suite, we empirically demonstrate that these information-theoretic metrics have higher correlations with normalized task solvability scores than a variety of alternatives.
Lastly, we show that these metrics can also be used for fast and compute-efficient optimizations of key design parameters such as reward shaping, policy architectures, and MDP properties for better solvability by RL algorithms without ever running full RL experiments\footnote{The code is available at~\url{https://github.com/frt03/pic}.}.
\end{abstract}

\section{Introduction}
\label{sec:introduction}
The myriad recent successes of reinforcement learning~(RL) have arguably been enabled by the proliferation of deep neural network function approximators applied to rich observational inputs~\cite{mnih2015human, hessel2018rainbow, kalashnikov2018qtopt}, enabling an agent to adeptly manage long-term sequential decision making in complex environments. 
While in the past much of the empirical RL research has focused on tabular or linear function approximation case~\cite{Dietterich1998taxi, McGovernB2001bottleneck, Konidaris2009SkillChaining}, the impressive successes of recent years (and anticipation of domains ripe for subsequent successes) has spurred the creation of \emph{non-tabular} benchmarks -- i.e., continuous control and/or continuous observation -- in which neural network function approximators are effectively a prerequisite~\cite{Bellemare2013ALE, gym2016openai, tassa2018deepmind}. 
Accordingly, empirical RL research is presently heavily focused on the use of neural network function approximators, spurring new algorithmic developments in both model-free~\cite{mnih2015human, Schulman:2015uk, Lillicrap2016, gu2016continuous,gu2017interpolated,haarnoja2018sacapps} and model-based~\cite{chua2018deep, janner2019trust, Hafner2020Dream} RL.

Despite the impressive progress of RL algorithms, the analysis of the RL \textit{environments} has been difficult and stagnant, precisely due to the complexity of modern benchmarks and neural network architectures required to solve them.
Most prior works analyzing sample complexity (a common measure of task complexity) focus on tabular MDPs with finite state and action dimensionalities~\citep{strehl2006pac,Strehl2009pac} or MDPs with simple dynamics~\citep{recht2018tour}, and are not applicable or measurable for typical deep RL benchmark tasks.
Besides the fact that the components of the environments -- observations, actions, dynamics, and rewards -- are drastically different in typical benchmarks, the investigation into task solvability is further complicated by the diversity of deep RL algorithms used in practice~\cite{Schulman:2015uk, Lillicrap2016, gu2016continuous, gu2016q, gu2017interpolated, haarnoja2018sacapps,chua2018deep, salimans2017evolution}, where each algorithm has distinct convergence behaviors.
Coming up with a universal, scalable, and measurable definition for task complexity of an RL environment appears an impossible task. 

In this work, we propose \textit{policy information capacity (PIC)} and \textit{policy-optimal information capacity (POIC)} as practical metrics for task or environment complexity, taking inspiration from information-theoretic RL research -- particularly mutual information maximization or \textit{empowerment}~\citep{Klyubin2005empower,tishby2011information,Tobias2012empower,Mohamed2015VariationalIM,eysenbach2019diyan,warde-farley2019discern,sharma2020dynamics,hafner2020action}\footnote{``Empowerment'' classically measures MI between actions and states, but since rewards can be thought as an additional state dimension, we might regard PIC and POIC as a type of empowerment.}. 
Policy information capacity measures mutual information between policy parameters and cumulative episodic rewards. As with standard decomposition in mutual information, maximizing policy information capacity corresponds to maximizing reward marginal entropy through policies (diversity maximization) while minimizing reward conditional entropy conditioned on any given policy parameter (predictability maximization), and effectively corresponds to maximizing reward \textit{controllability}. 
Intuitively, if an agent can easily control rewards or relevant states that affect the cumulative rewards, then most RL algorithms can generally\footnote{See Section~\ref{sec:controllability_vs_maximizability} for a thorough discussion.} maximize the rewards easily and the environment should be classified as ``easy''. Alternatively, policy-optimal information capacity (POIC), a variant of PIC drawing on the control as inference literature~\citep{todorov2006linearly, toussaint2006probabilistic, rawlik2012psi, fox2015taming, jaques2017sequence, haarnoja2017reinforcement, levine2018reinforcement}, measures mutual information between policy parameters and whether an episode is optimal or not, and even more closely relates to the optimizability of an RL environment. 

We compute policy and policy-optimal information capacities across a range of benchmark environments~\cite{gym2016openai,tassa2018deepmind}, and show that, in practice, especially POIC has a higher correlation with normalized task scores (computed as a brute-force proxy for task complexity by executing many runs of RL algorithms) than other sensible alternatives. Considering the simplicity of our metrics and the drastically varied nature of the benchmarks, our result shows that PIC and POIC can serve as practical and measurable metrics for task complexity. 

In summary, our work makes the following contributions:
\begin{itemize}[leftmargin=0.4cm,topsep=0pt,itemsep=0pt]
    \item We define and propose \textit{policy information capacity (PIC)} and its variant, \textit{policy-optimal information capacity (POIC)} as algorithm-agnostic quantitative metrics for measuring task complexity, and show that, \textit{POIC} in particular corresponds well with empirical task solvability computed across diverse benchmark environments~\citep{gym2016openai,tassa2018deepmind}. 
    \item We set up the first quantitative experimental protocol to evaluate the correctness of a task complexity metric. 
    \item We show that both PIC and POIC can be used as fast proxies for tuning experimental parameters to improve learning progress, such as reward shaping, or policy architectures and initialization parameters, without running any full RL experiment.
\end{itemize}

\section{Related Work}
\label{sec:related_work}
We provide a brief overview of related works, first of previously proposed proxy metrics for assessing the properties of RL algorithms or environments, and then of instances of mutual information (MI) in the context of RL.

\paragraph{Analysis of RL Environments}
A large body of prior work has sought to theoretically analyze RL algorithms, as opposed to RL environments. For example, \citet{kearns2002near,Strehl2009pac,christoph2015ucfh} derive sample complexity bounds under a finite MDP setting, while \citet{Jaksch2010regret,azar2017ucbvi,jin2018qlearning} and \citet{jin2020linear} prove regret bounds under finite MDP and linear function approximation settings respectively.
Some recent works extend these previous results to non-linear function approximation \citep{du2019gradient,wang2020eluder,yang2020neuralrl}, but they require strong assumptions on function approximators, such as a low Eluder dimension or infinitely-wide neural networks.
All these works, however, are algorithm-specific and do not study the properties of RL environments or MDPs.

Asides from algorithms, there are theoretical works that directly study the properties of MDPs. \citet{Jaksch2010regret} consider the diameter of an MDP, which is the maximum over distinct state pairs $(s, s')$ of expected steps to reach $s'$ from $s$.
\citet{jiang2017rank} propose \textit{Bellman rank} and show that an MDP with a low Bellman rank can be provably-efficiently solved.
\citet{Maillard2014distnorm} propose the environmental norm, the one-step variance of an optimal state-value function. 
However, those metrics are often intractable to estimate in practical problems, where state or action dimensions are high dimensional~\citep{Jaksch2010regret,pong2018temporal}, neural network function approximations are used~\citep{jiang2017rank,dann2018oraclepac}, or oracle Q-functions are not computable~\citep{Jaksch2010regret,jiang2017rank, Maillard2014distnorm}.
Orthogonally to all these works, we propose tractable metrics that can be approximated numerically for complex RL environments with high-dimensional states and actions and, crucially, complex function approximators such as neural networks.

The recent work of \citet{oller2020analyzing} is the closest to ours, where they qualitatively visualize marginal reward distributions and show how their variances are intuitively related to environment-specific task difficulty scores estimated from a random search algorithm.
While they present very promising early results for tackling this ambitious problem, ours has a few critical differences from their work, which we detail in Section~\ref{sec:pre_rwg}.

\paragraph{Mutual Information}
Mutual information has been widely used in RL algorithm development, as a mechanism to encourage emergence of diverse behaviors, sometimes known as \textit{empowerment}~\cite{Klyubin2005empower,tishby2011information,Tobias2012empower, Mohamed2015VariationalIM}.
\citet{gregor2016variational} employ such diverse behaviors as intrinsic-motivation-based exploration methods, which intend to reach diverse states per option, maximizing the lower bound of mutual information between option and trajectory.
Related to exploration~\citep{Leibfried2019rme,pong2020skewfit}, recently MI-based skill discovery has become a popular topic~\citep{florensa2017stochastic,eysenbach2019diyan,warde-farley2019discern,nachum2019near-optimal, sharma2020dynamics,sharma2020emergent,Campos2020ExploreDA,Hansen2020Fast}, and these previous works are sources of inspiration for our own metrics, PIC and POIC.
For instance, \citet{eysenbach2019diyan} and \citet{warde-farley2019discern} learn diverse behaviors through maximizing a lower bound on mutual information between skills and future states, which encourages the agent to learn many distinct skills. In other words, the agents learn how to control the environments (future states) via maximization of mutual information.
This intuition -- that mutual information is related to the controllability of the environments -- motivates our own MI-based task solvability metrics, where our metrics PIC and POIC can be seen as \textit{reward} and \textit{optimality empowerments} respectively.

\section{Preliminaries}
\label{sec:preliminaries}
We consider standard RL settings with a Markov Decision Process~(MDP) $\mathcal{M}$ defined by state space $\mathcal{S}$, action space $\mathcal{A}$, transition probability $p(s_{t+1}|s_t, a_t)$, initial state distribution $p(s_1)$, and reward function $r(s_t, a_t)$.
A policy\footnote{While we denote Markovian policies in our derivations, our metrics are also valid for non-Markovian policies.} $\pi(a|s) \in \Pi$ maps from states to probability distributions over actions. With function approximation, this policy is parameterized by $\theta \in \mathbb{R}^d$, initialized by sampling from a prior distribution of the parameter $p(\theta)$\footnote{For the familiarity of notations, we introduce $\theta$ as parameters of a parametric function. However, in general, $\theta$ can represent the function itself. Since our methods do not require estimations of  $\mathcal{H}(\Theta|\cdot)$, any distribution over functions is applicable, e.g. $p(\theta)$ can represent a distribution over different network architectures.}. 
We use the upper case $\Theta$ to represent this random variable. 
We also denote a trajectory as $\tau := (s_1, a_1, s_2, a_2, ..., s_T)$, and a cumulative reward as $r(\tau) := \sum_{(s, a) \in \tau} r(s,a)$; when clear from the context, we slightly abuse notation and simply use $r$ for $r(\tau)$. 
We use the upper case $R$ to represent the random variable taking on value $r(\tau)$. 
Since we focus on evaluation of the environments, we omit a discount factor $\gamma \in [0, 1)$.

The distributions $p(r)$ and $p(r|\theta)$ may be factored as,
\begin{equation}
p(r) = \E_{p(\tau|\theta)p(\theta)}\left[ p(r|\tau) \right],~p(r|\theta) = \E_{p(\tau|\theta)}\left[ p(r|\tau) \right], \nonumber
\end{equation}
where $p(r|\tau)$ is the reward distribution over trajectory, which, for simplicity, we assume is a deterministic delta distribution, and the marginal distribution of the trajectory conditioned on $\theta$ is $p(\tau|\theta) = \textstyle p(s_1) \prod_{t=1}^{T} p(s_{t+1}|s_t, a_t) \pi(a_t|s_t, \theta)$\,.

\subsection{Optimality Variable}

RL concerns with not only characterization of reward distribution, but also its maximization. Information-theoretic perspective on RL, or control as inference~\citep{todorov2006linearly,toussaint2006probabilistic,fox2015taming,jaques2017sequence,levine2018reinforcement}, connects such maximization with probabilistic inference through the notion of ``optimality'' variable, a binary random variable $\mathcal{O}_t \in \{0, 1\}$ in MDP, where $\mathcal{O}_t=1$ means the agent behaves ``optimally'' at time-step $t$, and $\mathcal{O}_t=0$ means not optimal.
For simplicity, we denote $\mathcal{O}_{1:T}$ as $\mathcal{O}$, which is also a binary random variable, representing whether the agent behaves optimally during the entire episode.
We define the distribution of this variable as: $p(\mathcal{O}=1 | \tau) = \exp\left( (r-r_{\max})/\eta \right)$\,,
where $\eta>0$ is a temperature and $r_{\max}$ is the maximum return on the MDP. Note that we subtract $r_{\max}$ to ensure $p(\mathcal{O}|\tau)$ is an appropriate probability distribution.

\subsection{Random Weight Guessing}
\label{sec:pre_rwg}
\citet{oller2020analyzing} recently proposed a qualitative analysis protocol of environment complexity with function approximation via random weight guessing.
It obtains $N$ particles of $\theta$ from prior $p(\theta)$ and runs the deterministic policy $\pi(a_t|s_t, \theta)$ with $M$ episodes per parameter \textit{without any training}. They qualitatively observe that the mean, percentiles, and variance of episodic returns have certain relations with an approximate difficulty of finding a good policy through random search.

However, our work has a number of key differences from their work: (1) we propose a detailed \textit{quantitative} evaluation protocol for verifying task difficulty metrics while they focus on qualitative discussions; (2) we derive our main task difficulty metric based on a mixture of SoTA RL algorithms, instead of random search, to better reflect the diversity of algorithm choices in practice; (3) we estimate reward entropies non-parametrically with many particles to reduce approximation errors, while their variance metric assumes Gaussianity of reward distributions and poorly approximates in the case of multi-modality; and (4) we verify the metrics on more diverse set of benchmark environments including OpenAI MuJoCo~\citep{gym2016openai} and DeepMind Control Suite~\citep{tassa2018deepmind} while they evaluate on classic control problems only.

\section{Policy and Policy-Optimal Information Capacity}
We now introduce our own proposed task complexity metrics. 
We begin with formal definitions for both metrics, and then provide details on how to estimate them.

\subsection{Formal Definitions}

\paragraph{Policy Information Capacity (PIC)} We define PIC as the mutual information $\mathcal{I}(\cdot;\cdot)$ between cumulative reward $R$ and policy parameter $\Theta$:
\begin{equation}
\mathcal{I}(R;\Theta) = \mathcal{H}(R) - \E_{p(\theta)}\left[\mathcal{H}(R|\Theta=\theta)\right],
\label{eq:mi_based_metric}
\end{equation}
where $\mathcal{H}(\cdot)$ is Shannon entropy. The intuitive interpretation is that when the environment gives a more diverse reward signal (first term in \autoref{eq:mi_based_metric}) and a more consistent reward signal per parameter (second term), it enables the agent to learn better behaviors.

\paragraph{Policy-Optimal Information Capacity (POIC)} We introduce the variant of PIC, termed \textit{Policy-Optimal Information Capacity (POIC)}, defined as the mutual information between the optimality variable and the policy parameter:
\begin{equation}
\mathcal{I}(\mathcal{O};\Theta) = \mathcal{H}(\mathcal{O}) - \E_{p(\theta)} \left[ \mathcal{H}(\mathcal{O}|\Theta=\theta) \right].
\label{eq:opt_metric}
\end{equation}

\subsection{Estimating Policy Information Capacity}
\label{sec:methods}
In this section, we describe a practical procedure for measuring PIC. In general, it is intractable to compute \autoref{eq:mi_based_metric} directly. The typical approach to estimate mutual information is to consider the lower bound~\cite{barber2004vmi,belghazi2018mine,poole2019vbmi}; however, if we estimate entropies in the one-dimensional reward space, we can use simpler techniques based on discretization~\citep{bellemare2017distributional}. 

We employ random policy sampling to measure mutual information between cumulative reward and parameter. Given an environment, a policy network $\pi_\theta$, and a prior distribution of the policy parameter $p(\theta)$, we generate $N$ particles of $\theta_i~(i = 1, \dots, N)\sim p(\theta)$ randomly and run the policy $\pi_{\theta_i}$ for $M$ episodes per particle (without any training). 
In total, we collect $NM$ trajectories and their corresponding cumulative rewards. We use $r_{ij}$ to denote the cumulative rewards of the $j$-th trajectory using $\theta_i$. 

We then empirically estimate \autoref{eq:mi_based_metric} via discretization of the empirical cumulative reward distribution for $p(r)$ and each $p(r|\theta_i)$ using the same $B$ bins. We set min and max values observed in sampling as the limit, and divide it into $B~(>M)$ equal parts: 

\begin{equation}
\begin{split}
\hat{\mathcal{I}}(R;\Theta) = & - \textstyle\sum_{b=1}^B \hat{p}(r_{b})\log\hat{p}(r_{b}) \\
& + \frac{1}{N} \textstyle\sum^{N}_{i=1} \sum_{b=1}^B \hat{p}(r_b|\theta_i) \log\hat{p}(r_b|\theta_i). 
\end{split}
\label{eq:reward_empowerment_sample}
\end{equation}
While there is an unavoidable approximation error when applying this estimator to continuous probability distributions, this approximation error can be reduced with sufficiently large $N,M,B$.
The sketch of this procedure is described in Algorithm \ref{alg:mu_rwg}.

\begin{algorithm}[tb]
\caption{PIC/POIC Estimation}
\label{alg:mu_rwg}
\renewcommand{\algorithmicrequire}{\textbf{Input:}}
\begin{algorithmic}[1]
\REQUIRE MDP $\mathcal{M}$, Policy $\pi$, Prior distribution of the parameter $p(\theta)$, Number of parameter $N$, Number of episodes $M$. Number of bins $B$.
\FOR{$i=1$ {\bfseries to} $N$}
 \STATE Generate parameter $\theta_i \sim p(\theta)$ and set it to $\pi$.
 \FOR{$j=1$ {\bfseries to} $M$}
    \STATE Initialize MDP $\mathcal{M}$.
    \STATE Run $\pi_{\theta_i}$ and Collect cumulative reward $r_{ij} \sim p(r|\theta_i)$.
 \ENDFOR
\ENDFOR
\STATE (for PIC) Approximate $p(r)$ and all $p(r|\theta_i)$ by the same discretization using $B$ bins.
\STATE Estimate PIC via~\autoref{eq:reward_empowerment_sample}, and/or POIC via~\autoref{eq:optimality_empowerment_sample}.
\end{algorithmic}
\end{algorithm}

\subsection{Estimating Policy-Optimal Information Capacity}
\autoref{eq:opt_metric} can be approximated by using the same samples from Algorithm~\ref{alg:mu_rwg}:
\begin{equation}
\begin{split}
& p(\mathcal{O}=1|\theta_i) \approx \frac{1}{M} \textstyle{\sum_{j=1}^{M}} \exp \left(\dfrac{r_{ij} - r_{\max}}{\eta}\right) := \hat{p}_{1i}; \\ \nonumber
&p(\mathcal{O}=1) \approx \frac{1}{N} \textstyle{\sum_{i=1}^{N}} \hat{p}_{1i} := \hat{p}_1, \\
\end{split}
\end{equation}
then,
\begin{equation}
\begin{split}
\hat{\mathcal{I}}(&\mathcal{O};~\Theta) = -\hat{p}_1\log\hat{p}_1 - (1 - \hat{p}_1)\log(1 - \hat{p}_1) \\
& + \frac{1}{N} \left(\textstyle\sum_{i=1}^{N} \hat{p}_{1i}\log\hat{p}_{1i} + (1-\hat{p}_{1i})\log(1-\hat{p}_{1i}) \right).\\
\end{split}
\label{eq:optimality_empowerment_sample}
\end{equation}
Compared to PIC, the exponential reward transform in POIC is more likely to favor reward maximization rather than minimization, which is preferred in a task solvability metric.
Additionally, since~\autoref{eq:optimality_empowerment_sample} is reduced to the entropies of discrete Bernoulli distributions, we can avoid reward discretization that is necessary for PIC~(\autoref{eq:reward_empowerment_sample}). However, since $\log p(\mathcal{O}=1|\theta)=\log \int_r p(\mathcal{O}=1|r)p(r|\theta)$, the sample-based estimators in~\autoref{eq:optimality_empowerment_sample} are biased (but asymptotically consistent).

\paragraph{Tuning temperature}
One disadvantage of POIC is that the choice of temperature $\eta$ can be arbitrary. To circumvent this, we choose the temperature which maximizes the mutual information: $\eta^* := \argmax_\eta \mathcal{I}(\mathcal{O};\Theta)$. In practice, we employ a black-box optimizer to numerically find it.

\subsection{Policy Selection and Policy Information Capacity}

Before proceeding to experiments, we theoretically explain why a high PIC might imply the ease of solving an MDP. Concretely, we consider how the PIC is related to the ease of choosing a better policy among two policies. Such a situation naturally occurs when random search or evolutionary algorithms are used.

We have the following proposition about relation between the PIC and the ease of policy selection. We say that a policy parameter $\theta_1$ is better than another policy parameter $\theta_2$ if its expected return $E_{p(r|\theta_1)}[r]~(:=\mu_{\theta_1})$ is larger than or equal to $E_{p(r|\theta_2)}[r]~(:=\mu_{\theta_2})$.

\begin{proposition}\label{prop:better policy determination}
Consider a situation where which policy parameter, $\theta_1$ or $\theta_2$, is better based on the order of $N$-sample-average estimates of expected returns, $\hat{\mu}_1 \approx E_{p(r|\theta_1)}[r]$ and $\hat{\mu}_2 \approx E_{p(r|\theta_2)}[r]$. Assume that $p(r | \theta) = \mathcal{N}(\mu_\theta, \sigma_\theta^2)$ for any $\theta \in \R^d$. Then the probability that we wrongly determine the order of $\theta_1$ and $\theta_2$ is at most
\begin{align*}
\E \left[ \exp \left( - \pi e N \left(\frac{\mu_{\theta_1} - \mu_{\theta_2}}{\exp (\mathcal{H}_1) + \exp( \mathcal{H}_2) } \right)^2 \right) \right],
\end{align*}
where the expectation is with respect to $\theta_1, \theta_2$, and $\mathcal{H}_i := \mathcal{H}(R|\Theta=\theta_i)$.
\end{proposition}

\begin{proof}
See Appendix~\ref{sec:appendix_proof_of_reward_empowerment_and_policy_selection}.
\end{proof}

Regarding the ease of policy selection, this proposition tells us that policy selection becomes easier when each $\mathcal{H}_i$ is small, and $\mu_{\theta_1}$ and $\mu_{\theta_2}$ are distant.
A high PIC $\gI(R;\Theta)$ suggests small $\mathcal{H}_i$ and a large distance between $\mu_{\theta_1}$ and $\mu_{\theta_2}$. Indeed, to keep $\gI(R;\Theta)$ high, the unconditional distribution of $r$ must be broad (high $\mathcal{H}(R)$), and a distribution of $r$ conditioned by each parameter must be narrow (low $\mathcal{H}_i$). To simultaneously achieve these two requirements, what can be done is narrowing each conditional distribution (low $\mathcal{H}_i$), and evenly scattering the conditional distributions over $\R$, resulting in a large distance between $\mu_{\theta_1}$ and $\mu_{\theta_2}$.

\section{Synthetic Experiments}
\label{sec:controllability_vs_maximizability}
As a way of motivating PIC and POIC, we introduce a simple setting in which both metrics correlate with task difficulty.

We aim to investigate the following two questions using simple MDPs in~\autoref{fig:multi_step_mdp} to build intuitions:
(1) Do PIC and POIC decrease as the conceptual difficulty of the MDP increases?
(2) How much do PIC and POIC change as the parameters of $p(\theta)$ are optimized during training (e.g. via evolutionary strategy)?
Additionally, we present comparisons between our information capacity measures and marginal entropies in \autoref{sec:appendix_synthetic}.

We assume the following simple MDP: 
\begin{itemize}[leftmargin=0.4cm,topsep=0pt,itemsep=0pt]
    \item The set of states is given by $\mathcal{S}=\{s_1, s_2, s_3, s_4, s_5\}$. The initial state is $s_1 = [1, 0, 0]$ while the other state vectors are $s_2=[0,1,0], s_3=[0,0,1], s_4=[1,1,1], s_5=[0,0,0]$. 
    \item The action space is $\mathcal{A}=\{a_1, a_2\}$, and the parameterized policy $\pi_{\theta}(a|s)$ for $\theta \in \R^3$ is given by
\begin{equation}
\pi_{\theta}(a|s) = \begin{cases}
\text{sigmoid}(\theta^{\mathrm{T}}s) & (a=a_1)\\
1 - \text{sigmoid}(\theta^{\mathrm{T}}s) & (a=a_2).
\nonumber
\end{cases}
\end{equation}
    \item The transitions are deterministic as illustrated in \autoref{fig:multi_step_mdp}.
    \item We consider three possible reward functions: $r_A,r_B,r_C$. For $r_A$, we have $r_A(s_2)=1$ and $r_A(s)=0$ otherwise. For $r_B$, we have $r_B(s_4)=1$, and $r_B(s)=0$ otherwise. For $r_C$, we have $r_C(s_5)=1$, and $r_C(s)=0$ otherwise.
\end{itemize}
We consider variants of this MDP according to horizon $T\in\{1, 2, 3\}$ and we pair each choice of horizon with a reward function; i.e. horizon $T=1$ is associated with $r_A$, horizon $T=2$ is associated with $r_B$, and horizon $T=3$ is associated with $r_C$.
We can describe this MDP as $\mathcal{M}=\{\mathcal{S}, \mathcal{A}, r(s), T\}$.
We take the policy parameter prior to be a Gaussian distribution $p(\theta)=\mathcal{N}(\mu, \sigma^2I)$, where $\mu,\sigma$ are hyper-parameters.

\begin{figure}[ht]
\centering
\begin{tikzpicture}[auto,font=\small]
\node(A2_0)[blue] at (-2.1,0.2) {$a_{2}$};
\node(A1_0)[red] at (-0.85,0.2) {$a_{1}$};
\node(A2_1)[blue] at (0.65,0.2) {$a_{2}$};
\node(A1_1)[red] at (0.0, 1.2) {$a_{1}$};
\node(A2_2)[blue] at (1.5, 1.2) {$a_{2}$};
\node(A1_2)[red] at (2.2, 0.2) {$a_{1}$};
\node(R0)[scale=0.8] at (-3,-0.7) {$r(s_3)=0$};
\node(R1_1)[scale=0.8,align=center] at (0,-0.7) {$r_A(s_2)=1$\\($T=1$)};
\node(R1_2)[scale=0.8,align=center] at (1.5,-0.7) {$r_B(s_4)=1$\\($T=2$)};
\node(R1_3)[scale=0.8,align=center] at (3.0,-0.7) {$r_C(s_5)=1$\\($T=3$)};
\tikzstyle{round}=[draw=black,circle,minimum size=0.5cm]
\node[round] (S1) at (-1.5,0) {$s_{1}$};
\node[round] (S2) at (0,0) {$s_{2}$};
\node[round] (S3) at (-3.0,0) {$s_{3}$};
\node[round] (S4) at (1.5,0) {$s_{4}$};
\node[round] (S5) at (3.0,0) {$s_{5}$};
\draw[->] (S3) [out=140,in=210,loop] to node {} node [swap] {} (S3);
\draw[->] (S2) [out=55,in=125,loop] to node {} node [swap] {} (S2);
\draw[->] (S4) [out=55,in=125,loop] to node {} node [swap] {} (S4);
\draw[->] (S1) to node {} node [swap] {} (S3);
\draw[->] (S1) to node {} node [swap] {} (S2);
\draw[->] (S2) to node {} node [swap] {} (S4);
\draw[->] (S4) to node {} node [swap] {} (S5);
\end{tikzpicture}
\vskip -0.15in
\caption{Multi-step discrete MDP. The transition is deterministic and $s_3$ is an absorbing state. The reward function is determined by horizon $T$. Intuitively, this MDP becomes difficult when the horizon is longer.}
\label{fig:multi_step_mdp}
\end{figure}
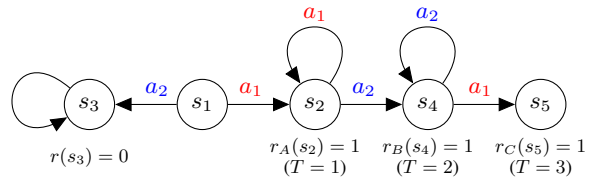

\paragraph{Answer to (1)}
We measure both information capacities and normalized score in multi-step MDPs with different horizon $T=1,2,3$ via Algorithm~\ref{alg:mu_rwg}.
The normalized score is defined as: $\frac{r_{\text{ave}} - r_{\min}}{r_{\max}- r_{\min}}$, where $r_{\text{ave}}, r_{\min}, r_{\max}$ are the average, maximum, and minimum return over parameters.
Intuitively, the MDP in~\autoref{fig:multi_step_mdp} becomes more difficult when the horizon gets longer.
We set prior parameter as $\mu=0$ and $\sigma=1$.
We sample 1,000 parameters from the prior, and evaluate each of them with 1,000 episode per parameter~(i.e. $N=1000$ and $M=1000$).
The results appear in~\autoref{table:multi_mi_normalized}.
We observe that PIC and POIC get lower when the horizon gets longer. The longer horizon MDP leads to a lower normalized score, which means that PIC and POIC properly reflect the task solvability of the MDP.

\begin{table}[htb]
\begin{center}
\begin{small}
\begin{tabular}{l|ccc}
\toprule
Horizon & Normalized Score & PIC & POIC \\
\midrule
$T=1$ & 0.451 & 0.087 & 0.087 \\
$T=2$ & 0.253 & 0.064 & 0.062 \\
$T=3$ & 0.112 & 0.050 & 0.049 \\
\bottomrule
\end{tabular}
\end{small}
\end{center}
\vskip -0.15in
\caption{The relations between normalized score by random behavior and PIC $\hat{\mathcal{I}}(R;\Theta)$ or POIC $\hat{\mathcal{I}}(\mathcal{O};\Theta)$ on simple MDP with different horizon~(\autoref{fig:multi_step_mdp}). All of them use Gaussian prior $\mathcal{N}(0,I)$. The lower normalized score, the lower information capacity metrics are.}
\label{table:multi_mi_normalized}
\end{table}

\paragraph{Answer to (2)}
Our metrics are intrinsically \textit{local}, in that it assumes some $p(\theta)$ for estimation. A natural question is, what happens to these metrics throughout a realistic learning process?
To answer this, we optimize $\mu$ in $p(\theta)$ to solve the MDP by evolution strategy~\cite{salimans2017evolution}, and observe PIC, POIC, and the agent performance during training. 
We assume the Gaussian prior $\mathcal{N}(\mu, I)$, and vary $\mu$ initializations using $[-10, -5, -4, -3, -2, -1, 0, 1, 2, 3, 4, 5, 10]$.
We set $N=1000$ and $M=1000$, and horizon is $T=3$.
\autoref{fig:opt_2_three_step_mdp_ps10_learning_curve2} presents the results of $\mu=0, -3.0, -4.0, -5.0$ (for the visibility); the rest of results and the case of PIC are shown in \autoref{sec:appendix_synthetic}.
These results confirm that the initial prior with high POIC ($\mu=0$) actually solves the environment faster than those with low POIC ($\mu=-3.0, -4.0, -5.0$).
Interestingly, \autoref{fig:opt_2_three_step_mdp_ps10_learning_curve2} also shows that high POIC effectively corresponds to regions of fastest learning. This allows us to build an intuition about what happens during learning: in parameter regions with high POIC learning accelerates, and in those with low POIC learning slows down, if $p(\theta)$ corresponds approximately to each local search region per update step.
As for PIC, the same trends can be observed in \autoref{sec:appendix_synthetic}.
We additionally provide the further examples on more complex environments like classic control or MuJoCo tasks in \autoref{sec:appendix_es_empowerments}.

\begin{figure}[ht]
\centering
\includegraphics[width=\linewidth]{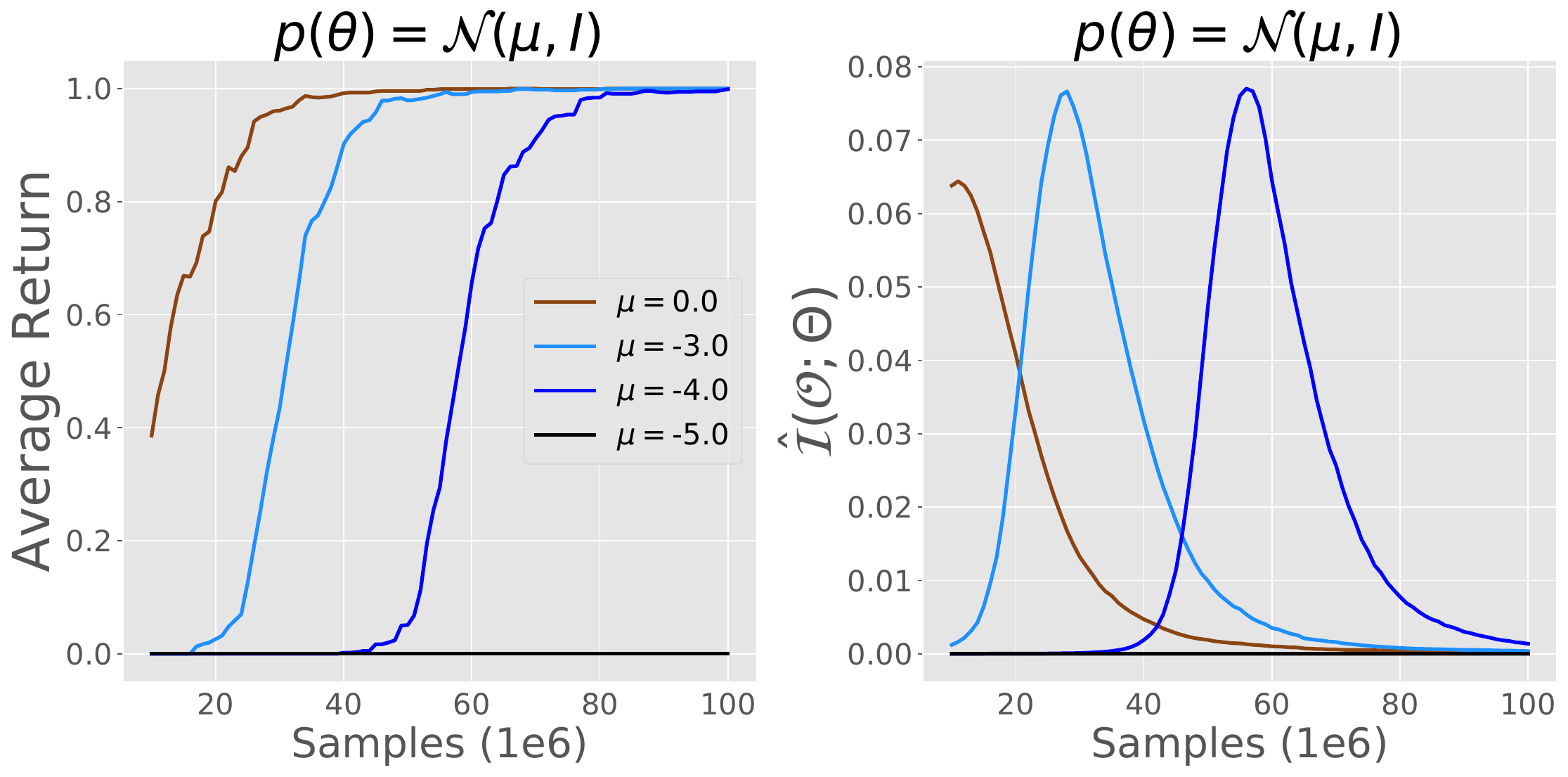}  
\vskip -0.05in
\caption{Average return (left) and POIC (right) during the training of evolution strategies. We vary the parameter of initial prior distribution $\mu$ and optimize it. High POIC correctly corresponds to regions of fastest learning.
}
\label{fig:opt_2_three_step_mdp_ps10_learning_curve2}
\end{figure}

\section{Deep RL Experiments}
\label{sec:Experiments}
In this section, we begin by elaborating on how we derive a brute-force task complexity metric to serve as a ground-truth metric to compare with. Then, we study the following questions: (1) Do PIC and POIC metrics correlate well with task complexity across standard deep RL benchmarks such as OpenAI Gym~\cite{gym2016openai,todorov2012mujoco} and DM Control~\citep{tassa2018deepmind}?
(2) Are PIC and POIC more correlated with task complexity than other possible metrics (entropy or variance of returns~\cite{oller2020analyzing})?
(3) Can both PIC and POIC be used to evaluate and tune goodness of reward shaping, network architectures, and parameter initialization without requiring running full RL training?

\subsection{Defining and Estimating Brute-Force Task Complexity Measure}
\label{sec:norm_task_scores}
While an oracle metric for task solvability on complex RL environments seems intractable, one possible (but costly) alternative is to run a large set of diverse RL algorithms and evaluate their normalized average performances.
On any given environment, some of these algorithms may completely and efficiently solve the task while others may struggle to learn; an appropriate averaging of the performances of the algorithms can serve as a ``ground-truth'' task complexity score.
As the preparation for the following experiments, we will pre-compute these normalized average performances.

\paragraph{Collecting Raw Algorithm Performances}
First, we prepare a bag of algorithms (and hyper-parameters) for learning and execute them all. We treat three types of environments separately; classic control, MuJoCo~\citep{gym2016openai}, and DM Control~\citep{tassa2018deepmind}.
For classic control, we run 23 algorithms, based on PPO~\citep{Schulman2017PPO}, DQN~\citep{mnih2015human} and Evolution Strategy with different hyper-parameters for discrete-action, and PPO, DDPG~\citep{Lillicrap2016}, SAC~\cite{haarnoja2018sacapps}, and Evolution Strategy with different hyper-parameters for continuous-action space environments.
For MuJoCo and DM Control, we run SAC, MPO~\citep{Abdolmaleki2018mpo} and AWR~\citep{peng2019awr}, and to simulate more diverse set of algorithms, we additionally incorporated the leaderboard scores reported in previous SoTA works~\cite{fujimoto2018td3,peng2019awr,laskin_lee2020rad}.
See Appendix~\ref{sec:appendix_normalized_score_algo} for further details.

\paragraph{Computing Normalized Scores (Algorithm)}
After collecting raw performances, we compute the average return over the all algorithms $r_{\text{ave}}^{\text{algo}}$.
For the comparison, we need to align the range of reward that is different from each environment.
To normalize average return over the environments, we take the maximum between this algorithm-based and random-sampling-based maximum scores (explained below), and use the minimum return obtained by random policy sampling:
\begin{equation}
\text{Normalized Score} := \frac{r_{\text{ave}}^{\text{algo}} - r_{\min}^{\text{rand}}}{\max(r_{\max}^{\text{rand}},~r_{\max}^{\text{algo}}) - r_{\min}^{\text{rand}}}.  \nonumber
\end{equation}
As a sanity check, we checked that the task scores do not trivially correlate with obvious properties of MDP or policy networks, such as state, action dimensionalities; episodic horizon in Appendix~\ref{sec:appendix_obvious_propaties}.

\paragraph{Computing Normalized Scores (Random Sampling)}
In addition to ``bag-of-algorithms'' task scores, we also compute random-sampling-based task scores that are considered in~\citet{oller2020analyzing} for characterizing task difficulties; see Appendix~\ref{sec:appendix_normalized_score_rs} for more details. While this metric is easier to compute, this only measures the task difficulty of an environment with respect to random search algorithm, and ignores the availability of more advanced RL algorithms. 

\subsection{Evaluating PIC and POIC as Task Complexity Measures}
\label{sec:mi_on_benchmark}
To verify that our MI-based metrics perform favorably as task solvability metrics in practical settings, we measure both PIC and POIC, along with other alternative metrics, following the random-sampling protocol in Algorithm~\ref{alg:mu_rwg} on the standard RL benchmarks: CartPole, Pendulum, MountainCar, MountainCarContinuous, and Acrobot from classic control in Open AI Gym; HalfCheetah, Walker2d, Humanoid, and Hopper from MuJoCo tasks; cheetah run, reacher easy, and ball\_in\_cup catch in DM Control~(see~\autoref{sec:env_details}).

We prepare a ``bag-of-policy-architectures'' to model a realistic prior over policy functions practitioners would use: ([0] layers $+$ [1, 2] layers $\times$ [4, 32, 64] hidden units) $\times$ [Gaussian prior $\mathcal{N}(0, I)$, Uniform prior $\textit{Unif}(-1,1)$, Xavier Normal, Xavier Uniform] $\times$ [w/ bias, w/o bias]; totally 56 variants of architectures\footnote{\citet{oller2020analyzing} by contrast only studied smaller networks, such as 2 layers with 4 units, which were sufficient for their classic control experiments, but certainly would not be for our MuJoCo environments.}.
We sample 1,000 parameters from the prior for each architecture, and evaluate each of them with 1,000 episode per parameter ($N=56\times1000=56000$ and $M=1000$) for random policy sampling.
The number of bins for discretization is set to $B=10^5$ for surely maximizing PIC.
To compare the suitability of our information capacity metrics and Shannon entropy or variance as task solvability metrics, we compute Pearson correlation coefficients between these measures and the normalized scores for the quantitative evaluation.

\begin{figure*}[ht]
\centering
\includegraphics[width=\linewidth]{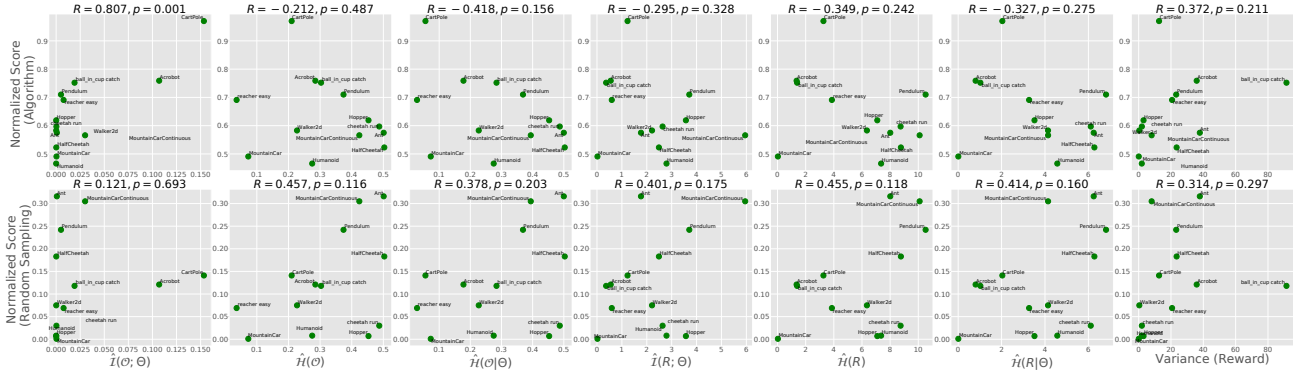}  
\vskip -0.05in
\caption{2D-Scatter plots between each metric (x-axis) and normalized scores (algorithm-based (top) and random-sampling-based (bottom) ; y-axis); see~\autoref{table:empowerment_full} for labeled values with environment names. Variance (last column) approximately corresponds to the metric proposed by~\citet{oller2020analyzing}.
\textbf{POIC ($\hat{\mathcal{I}}(\mathcal{O};\Theta)$) positively correlates with algorithm-based normalized score} ($\bf{R=0.807}$; statistically significant with $p<0.01$), the more realistic of the two task difficulty scores, more than all other alternatives including variance of returns~\citep{oller2020analyzing}. Note that the two normalized task scores have a low correlation of $0.139$ (see~\autoref{fig:peason_correlation_coefficient_sub}), and therefore a high correlation in one means a low correlation in the other.}
\label{fig:peason_correlation_coefficient}
\end{figure*}

\autoref{fig:peason_correlation_coefficient} visualizes the relation between metrics computed via random sampling (Algorithm \ref{alg:mu_rwg}) and normalized scores (see \autoref{table:empowerment_full} in \autoref{sec:appendix_full_results} for the detailed scores). Note that variance of returns is scaled by $r_{\max}^{\text{rand}} - r_{\min}^{\text{rand}}$ for a normalized comparison among different environments.
The results suggest that POIC seems to positively correlate better with algorithm-based normalized score ($R=0.807$; statistically significant with $p<0.01$) compared to any other alternatives, such as reward marginal entropy ($R=-0.349$) or variance of returns ($R=0.372$)~\citep{oller2020analyzing}.
Here, we can see that POIC seems to work as a task solvability metric in standard RL benchmarks.
In contrast, PIC is correlated with random-sampling-based normalized score ($R=0.401$) and superior to variance of returns ($R=0.314$). However, it is less correlated with algorithm-based task scores, which seems closer to actual task difficulty.
These differences are possibly due to maximization bias in optimality variable from exponential transform and estimation in Bernoulli space\footnote{The actual behavior of PIC and POIC during RL training might be also related to it (see \autoref{sec:appendix_es_empowerments}).}.

\subsection{Evaluating the Goodness of Reward Shaping and Other ``Hyper-Parameters'' without RL}
\label{sec:exp_reward_shaping}
We additionally test whether both PIC and POIC can directly evaluate the goodness of reward shaping properly. We investigate two goal-oriented tasks, Reacher~\citep{todorov2012mujoco} and Pointmaze~\citep{fu2020d4rl}, where both tasks are defined based on distance functions. We prepare the following four families of distance reward functions: L1 norm, L2 norm, Fraction, and Sparse, each with 1 or 2 hyper-parameters.
We select 4 hyper-parameter values for each, totaling 16 different reward functions.
To get the normalized scores, we run PPO with 500k steps and 5 seeds for each (see \autoref{sec:effect_on_reward_shaping} for the details).

\autoref{fig:reward_shaping} reveals that Pearson correlation coefficients between the normalized score and information capacity metrics have positive correlations (statistically significant with $p<0.05$). Coincidentally, clustering of L1 and L2 results reveals that these two reward families are much more robust to ill-specified reward hyper-parameters (i.e. require much less hyper-parameter tuning) than Fraction and Sparse, which is expected given the smoothness and low curvatures of L1 and L2.
The results show that \textit{both} PIC and POIC can evaluate the goodness of reward shaping for optimizability\footnote{Typically, separately from reward shaping, there is a true task reward  (success metric). However, in our definition of task difficulty, we only measure how easy to optimize the given reward (shaped reward). In practice, one should choose a shaped reward that is both easy to optimize (e.g. based on our metric) and accurately reflecting the true task success.}.
We run additional experiments to evaluate the goodness of architectures, initializations, and dynamics noises. See \autoref{sec:mi_with_different_archi} and \ref{sec:appendix_dynamics_initialization} for the details.

\label{sec:mi_with_reward_shaping}
\begin{figure*}[ht]
\centering
\includegraphics[width=\linewidth]{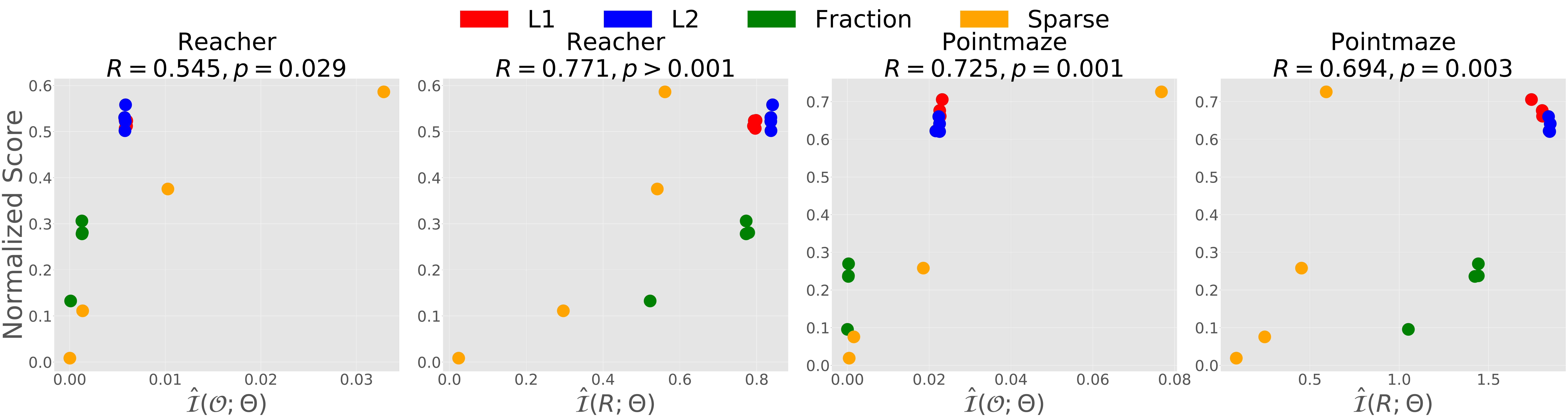}  
\vskip -0.05in
\caption{2D-Scatter plots between PIC or POIC~(x-axis) and normalized score (y-axis), after 500k training steps and averaged with 5 seeds. We test 4 family of reward functions (L1 norm, L2 norm, Fraction, Sparse) and 4 variants each (see \autoref{sec:effect_on_reward_shaping} for the details).
The results indicate that \textit{both} PIC and POIC can be used to evaluate the goodness of reward shaping proposals. The clusterings of L1 and L2 rewards indicate that learning difficulty with those reward functions does not vary much with hyper-parameters, and suggest that they do not require much hyper-parameter tuning in practice for these tasks.}
\label{fig:reward_shaping}
\end{figure*}


\section{Limitation and Future Work}
We tackle a seemingly intractable problem: \textit{to quantify the difficulty of an RL environment irrespective of learning algorithms}. While our empirical evaluations present many encouraging positive results, our metrics  have obvious limitations. The biggest is the dependence on $p(\theta)$. As discussed in Section~\ref{sec:controllability_vs_maximizability}, $p(\theta)$ intuitively defines an effective search area (in function space or function parameter space) and our two information capacity metrics approximately measure the easiness of searchability and maximizability respectively within this area weighted by a prior search distribution. Since these metrics in our experiments are only measured at initialization (except results in \autoref{fig:opt_2_three_step_mdp_ps10_learning_curve2} and \autoref{sec:appendix_synthetic}), they may not correspond well with overall task solvability: (1) if the optimization landscape drastically changes in the later stages of learning, particularly likely for neural networks~\citep{li2017visualizing} or MDPs with discontinuous rewards or dynamics; or (2) if $p(\theta)$ poorly approximates actual search directions given by SGD, natural gradients, Q-learning, etc., during learning. Exploring our metrics throughout the dynamics of optimization and learning to adapt $p(\theta)$ are some of the exciting future directions, along with making connections to works in supervised learning that relate convergence to signal-to-noise ratios in gradient estimators~\citep{saxe2019information,smith2017don,smith2017bayesian}. Another important direction is to scale the evaluations to problems requiring larger neural networks, like ALE with image observations~\citep{Bellemare2013ALE}.     

\section{Conclusion}
\label{sec:conclusion}
We defined \textit{policy information capacity (PIC)} and \textit{policy-optimal information capacity (POIC)} as information-theoretic metrics for numerically analyzing the \textit{generic} task complexities of RL environments.
These metrics are simple and practical: estimating these metrics only requires a prior distribution over policy functions $p(\theta)$ and trajectory sampling.
We formalized a quantitative evaluation protocol for verifying the correctness of task difficulty metrics, that properly accounts for both the richness of available RL algorithms and the complexities of high-dimensional benchmark environments.  
Through careful experimentation, we successfully identified POIC as the only metric that exhibited high correlations with a brute-force measurement of environment complexity, and demonstrated that these PIC and POIC metrics can be used for tuning task parameters such as reward shaping, MDP dynamics, network architecture and initialization for best learnability without running RL experiments.  
We hope our work can inspire future research to further explore these long overdue questions of analyzing, measuring, and categorizing the properties of RL \textit{environments}, which can guide us to developing even better learning algorithms.

\section*{Acknowledgements}
We thank Yusuke Iwasawa, Danijar Hafner, and Karol Hausman for helpful feedback on this work.

\bibliography{icml2021}

\begin{thebibliography}{79}
\providecommand{\natexlab}[1]{#1}
\providecommand{\url}[1]{\texttt{#1}}
\expandafter\ifx\csname urlstyle\endcsname\relax
  \providecommand{\doi}[1]{doi: #1}\else
  \providecommand{\doi}{doi: \begingroup \urlstyle{rm}\Url}\fi

\bibitem[Abdolmaleki et~al.(2018)Abdolmaleki, Springenberg, Tassa, Munos,
  Heess, and Riedmiller]{Abdolmaleki2018mpo}
Abdolmaleki, A., Springenberg, J.~T., Tassa, Y., Munos, R., Heess, N., and
  Riedmiller, M.
\newblock {Maximum a posteriori policy optimisation}.
\newblock In \emph{International Conference on Learning Representations}, 2018.

\bibitem[Andrychowicz et~al.(2021)Andrychowicz, Raichuk, Sta{\'n}czyk, Orsini,
  Girgin, Marinier, Hussenot, Geist, Pietquin, Michalski, Gelly, and
  Bachem]{andrychowicz2021what}
Andrychowicz, M., Raichuk, A., Sta{\'n}czyk, P., Orsini, M., Girgin, S.,
  Marinier, R., Hussenot, L., Geist, M., Pietquin, O., Michalski, M., Gelly,
  S., and Bachem, O.
\newblock What matters for on-policy deep actor-critic methods? a large-scale
  study.
\newblock In \emph{International Conference on Learning Representations}, 2021.

\bibitem[Azar et~al.(2017)Azar, Osband, and Munos]{azar2017ucbvi}
Azar, M.~G., Osband, I., and Munos, R.
\newblock Minimax regret bounds for reinforcement learning.
\newblock In \emph{International Conference on Machine Learning}, 2017.

\bibitem[Barber \& Agakov(2004)Barber and Agakov]{barber2004vmi}
Barber, D. and Agakov, F.
\newblock The im algorithm : A variational approach to information
  maximization.
\newblock In \emph{Advances in neural information processing systems}, 2004.

\bibitem[Belghazi et~al.(2018)Belghazi, Baratin, Rajeshwar, Ozair, Bengio,
  Courville, and Hjelm]{belghazi2018mine}
Belghazi, M.~I., Baratin, A., Rajeshwar, S., Ozair, S., Bengio, Y., Courville,
  A., and Hjelm, D.
\newblock Mutual information neural estimation.
\newblock In \emph{International Conference on Machine Learning}, 2018.

\bibitem[Bellemare et~al.(2013)Bellemare, Naddaf, Veness, and
  Bowling]{Bellemare2013ALE}
Bellemare, M.~G., Naddaf, Y., Veness, J., and Bowling, M.
\newblock The arcade learning environment: An evaluation platform for general
  agents.
\newblock \emph{Journal of Artificial Intelligence Research}, 2013.

\bibitem[Bellemare et~al.(2017)Bellemare, Dabney, and
  Munos]{bellemare2017distributional}
Bellemare, M.~G., Dabney, W., and Munos, R.
\newblock A distributional perspective on reinforcement learning.
\newblock In \emph{International Conference on Machine Learning}, 2017.

\bibitem[Boucheron et~al.(2013)Boucheron, Lugosi, and
  Massart]{boucheron2013concentration}
Boucheron, S., Lugosi, G., and Massart, P.
\newblock \emph{Concentration Inequalities: A Nonasymptotic Theory of
  Independence}.
\newblock Oxford University Press, Oxford, UK, 2013.

\bibitem[Brockman et~al.(2016)Brockman, Cheung, Pettersson, Schneider,
  Schulman, Tang, and Zaremba]{gym2016openai}
Brockman, G., Cheung, V., Pettersson, L., Schneider, J., Schulman, J., Tang,
  J., and Zaremba, W.
\newblock Openai gym.
\newblock \emph{arXiv preprint arXiv:1606.01540}, 2016.

\bibitem[Campos et~al.(2020)Campos, Trott, Xiong, Socher, i~Nieto, and
  Torres]{Campos2020ExploreDA}
Campos, V.~A., Trott, A., Xiong, C., Socher, R., i~Nieto, X.~G., and Torres, J.
\newblock Explore, discover and learn: Unsupervised discovery of state-covering
  skills.
\newblock In \emph{International Conference on Machine Learning}, 2020.

\bibitem[Chua et~al.(2018)Chua, Calandra, McAllister, and Levine]{chua2018deep}
Chua, K., Calandra, R., McAllister, R., and Levine, S.
\newblock Deep reinforcement learning in a handful of trials using
  probabilistic dynamics models.
\newblock In \emph{Advances in Neural Information Processing Systems}, 2018.

\bibitem[Dann \& Brunskill(2015)Dann and Brunskill]{christoph2015ucfh}
Dann, C. and Brunskill, E.
\newblock Sample complexity of episodic fixed-horizon reinforcement learning.
\newblock In \emph{Advances in Neural Information Processing Systems}, 2015.

\bibitem[Dann et~al.(2018)Dann, Jiang, Krishnamurthy, Agarwal, Langford, and
  Schapire]{dann2018oraclepac}
Dann, C., Jiang, N., Krishnamurthy, A., Agarwal, A., Langford, J., and
  Schapire, R.~E.
\newblock On oracle-efficient pac rl with rich observations.
\newblock In \emph{Advances in Neural Information Processing Systems}, 2018.

\bibitem[Dietterich(1998)]{Dietterich1998taxi}
Dietterich, T.~G.
\newblock The maxq method for hierarchical reinforcement learning.
\newblock In \emph{International Conference on Machine Learning}, 1998.

\bibitem[Du et~al.(2019)Du, Zhai, Poczos, and Singh]{du2019gradient}
Du, S.~S., Zhai, X., Poczos, B., and Singh, A.
\newblock Gradient descent provably optimizes over-parameterized neural
  networks.
\newblock In \emph{International Conference on Learning Representations}, 2019.

\bibitem[Engstrom et~al.(2019)Engstrom, Ilyas, Santurkar, Tsipras, Janoos,
  Rudolph, and Madry]{engstrom2019implementation}
Engstrom, L., Ilyas, A., Santurkar, S., Tsipras, D., Janoos, F., Rudolph, L.,
  and Madry, A.
\newblock Implementation matters in deep rl: A case study on ppo and trpo.
\newblock In \emph{International Conference on Learning Representations}, 2019.

\bibitem[Eysenbach et~al.(2019)Eysenbach, Gupta, Ibarz, and
  Levine]{eysenbach2019diyan}
Eysenbach, B., Gupta, A., Ibarz, J., and Levine, S.
\newblock Diversity is all you need: Learning diverse skills without a reward
  function.
\newblock In \emph{international conference on learning representations}, 2019.

\bibitem[Florensa et~al.(2017)Florensa, Duan, and
  Abbeel]{florensa2017stochastic}
Florensa, C., Duan, Y., and Abbeel, P.
\newblock Stochastic neural networks for hierarchical reinforcement learning.
\newblock In \emph{International conference on learning representations}, 2017.

\bibitem[Fox et~al.(2016)Fox, Pakman, and Tishby]{fox2015taming}
Fox, R., Pakman, A., and Tishby, N.
\newblock Taming the noise in reinforcement learning via soft updates.
\newblock In \emph{Conference on Uncertainty in Artificial Intelligence}, 2016.

\bibitem[Fu et~al.(2020)Fu, Kumar, Nachum, Tucker, and Levine]{fu2020d4rl}
Fu, J., Kumar, A., Nachum, O., Tucker, G., and Levine, S.
\newblock D4rl: Datasets for deep data-driven reinforcement learning.
\newblock \emph{arXiv preprint arXiv:2004.07219}, 2020.

\bibitem[Fujimoto et~al.(2018)Fujimoto, van Hoof, and Meger]{fujimoto2018td3}
Fujimoto, S., van Hoof, H., and Meger, D.
\newblock Addressing function approximation error in actor-critic methods.
\newblock In \emph{International Conference on Machine Learning}, 2018.

\bibitem[Glorot \& Bengio(2010)Glorot and Bengio]{glorot2010}
Glorot, X. and Bengio, Y.
\newblock Understanding the difficulty of training deep feedforward neural
  networks.
\newblock In \emph{International Conference on Artificial Intelligence and
  Statistics}, 2010.

\bibitem[Gregor et~al.(2016)Gregor, Rezende, and
  Wierstra]{gregor2016variational}
Gregor, K., Rezende, D.~J., and Wierstra, D.
\newblock Variational intrinsic control.
\newblock \emph{arXiv preprint arXiv:1611.07507}, 2016.

\bibitem[Gu et~al.(2016{\natexlab{a}})Gu, Lillicrap, Ghahramani, Turner, and
  Levine]{gu2016q}
Gu, S., Lillicrap, T., Ghahramani, Z., Turner, R.~E., and Levine, S.
\newblock Q-prop: Sample-efficient policy gradient with an off-policy critic.
\newblock \emph{arXiv preprint arXiv:1611.02247}, 2016{\natexlab{a}}.

\bibitem[Gu et~al.(2016{\natexlab{b}})Gu, Lillicrap, Sutskever, and
  Levine]{gu2016continuous}
Gu, S., Lillicrap, T., Sutskever, I., and Levine, S.
\newblock Continuous deep q-learning with model-based acceleration.
\newblock In \emph{International Conference on Machine Learning},
  2016{\natexlab{b}}.

\bibitem[Gu et~al.(2017)Gu, Lillicrap, Turner, Ghahramani, Sch{\"o}lkopf, and
  Levine]{gu2017interpolated}
Gu, S.~S., Lillicrap, T., Turner, R.~E., Ghahramani, Z., Sch{\"o}lkopf, B., and
  Levine, S.
\newblock Interpolated policy gradient: Merging on-policy and off-policy
  gradient estimation for deep reinforcement learning.
\newblock In \emph{Advances in neural information processing systems}, 2017.

\bibitem[Haarnoja et~al.(2017)Haarnoja, Tang, Abbeel, and
  Levine]{haarnoja2017reinforcement}
Haarnoja, T., Tang, H., Abbeel, P., and Levine, S.
\newblock Reinforcement learning with deep energy-based policies.
\newblock In \emph{International Conference on Machine Learning}, 2017.

\bibitem[Haarnoja et~al.(2018)Haarnoja, Zhou, Hartikainen, Tucker, Ha, Tan,
  Kumar, Zhu, Gupta, Abbeel, and Levine]{haarnoja2018sacapps}
Haarnoja, T., Zhou, A., Hartikainen, K., Tucker, G., Ha, S., Tan, J., Kumar,
  V., Zhu, H., Gupta, A., Abbeel, P., and Levine, S.
\newblock Soft actor-critic algorithms and applications.
\newblock \emph{arXiv preprint arXiv:1812.05905}, 2018.

\bibitem[Hafner et~al.(2020{\natexlab{a}})Hafner, Lillicrap, Ba, and
  Norouzi]{Hafner2020Dream}
Hafner, D., Lillicrap, T., Ba, J., and Norouzi, M.
\newblock Dream to control: Learning behaviors by latent imagination.
\newblock In \emph{International Conference on Learning Representations},
  2020{\natexlab{a}}.

\bibitem[Hafner et~al.(2020{\natexlab{b}})Hafner, Ortega, Ba, Parr, Friston,
  and Heess]{hafner2020action}
Hafner, D., Ortega, P.~A., Ba, J., Parr, T., Friston, K., and Heess, N.
\newblock Action and perception as divergence minimization.
\newblock \emph{arXiv preprint arXiv:2009.01791}, 2020{\natexlab{b}}.

\bibitem[Hansen et~al.(2020)Hansen, Dabney, Barreto, Warde-Farley, de~Wiele,
  and Mnih]{Hansen2020Fast}
Hansen, S., Dabney, W., Barreto, A., Warde-Farley, D., de~Wiele, T.~V., and
  Mnih, V.
\newblock Fast task inference with variational intrinsic successor features.
\newblock In \emph{International Conference on Learning Representations}, 2020.

\bibitem[Hessel et~al.(2018)Hessel, Modayil, Van~Hasselt, Schaul, Ostrovski,
  Dabney, Horgan, Piot, Azar, and Silver]{hessel2018rainbow}
Hessel, M., Modayil, J., Van~Hasselt, H., Schaul, T., Ostrovski, G., Dabney,
  W., Horgan, D., Piot, B., Azar, M., and Silver, D.
\newblock Rainbow: Combining improvements in deep reinforcement learning.
\newblock In \emph{AAAI Conference on Artificial Intelligence}, 2018.

\bibitem[Jaksch et~al.(2010)Jaksch, Ortner, and Auer]{Jaksch2010regret}
Jaksch, T., Ortner, R., and Auer, P.
\newblock Near-optimal regret bounds for reinforcement learning.
\newblock \emph{Journal of Machine Learning Research}, 2010.

\bibitem[Janner et~al.(2019)Janner, Fu, Zhang, and Levine]{janner2019trust}
Janner, M., Fu, J., Zhang, M., and Levine, S.
\newblock When to trust your model: Model-based policy optimization.
\newblock In \emph{Advances in Neural Information Processing Systems}, 2019.

\bibitem[Jaques et~al.(2017)Jaques, Gu, Bahdanau, Hern{\'a}ndez-Lobato, Turner,
  and Eck]{jaques2017sequence}
Jaques, N., Gu, S., Bahdanau, D., Hern{\'a}ndez-Lobato, J.~M., Turner, R.~E.,
  and Eck, D.
\newblock Sequence tutor: Conservative fine-tuning of sequence generation
  models with kl-control.
\newblock In \emph{International Conference on Machine Learning}, 2017.

\bibitem[Jiang et~al.(2017)Jiang, Krishnamurthy, Agarwal, Langford, and
  Schapire]{jiang2017rank}
Jiang, N., Krishnamurthy, A., Agarwal, A., Langford, J., and Schapire, R.~E.
\newblock Contextual decision processes with low {B}ellman rank are
  {PAC}-learnable.
\newblock In \emph{International Conference on Machine Learning}, 2017.

\bibitem[Jin et~al.(2018)Jin, Allen-Zhu, Bubeck, and Jordan]{jin2018qlearning}
Jin, C., Allen-Zhu, Z., Bubeck, S., and Jordan, M.~I.
\newblock Is q-learning provably efficient?
\newblock In \emph{Advances in Neural Information Processing Systems}, 2018.

\bibitem[Jin et~al.(2020)Jin, Yang, Wang, and Jordan]{jin2020linear}
Jin, C., Yang, Z., Wang, Z., and Jordan, M.~I.
\newblock Provably efficient reinforcement learning with linear function
  approximation.
\newblock In \emph{Conference on Learning Theory}, 2020.

\bibitem[Jung et~al.(2011)Jung, Polani, and Stone]{Tobias2012empower}
Jung, T., Polani, D., and Stone, P.
\newblock Empowerment for continuous agent—environment systems.
\newblock \emph{Adaptive Behavior}, 2011.

\bibitem[Kalashnikov et~al.(2018)Kalashnikov, Irpan, Pastor, Ibarz, Herzog,
  Jang, Quillen, Holly, Kalakrishnan, Vanhoucke, and
  Levine]{kalashnikov2018qtopt}
Kalashnikov, D., Irpan, A., Pastor, P., Ibarz, J., Herzog, A., Jang, E.,
  Quillen, D., Holly, E., Kalakrishnan, M., Vanhoucke, V., and Levine, S.
\newblock {QT-Opt}: Scalable deep reinforcement learning for vision-based
  robotic manipulation.
\newblock In \emph{Conference on Robot Learning}, 2018.

\bibitem[Kearns \& Singh(2002)Kearns and Singh]{kearns2002near}
Kearns, M. and Singh, S.
\newblock Near-optimal reinforcement learning in polynomial time.
\newblock \emph{Machine Learning}, 2002.

\bibitem[Klyubin et~al.(2005)Klyubin, Polani, and Nehaniv]{Klyubin2005empower}
Klyubin, A.~S., Polani, D., and Nehaniv, C.~L.
\newblock All else being equal be empowered.
\newblock In \emph{Advances in Artificial Life}, 2005.

\bibitem[Konidaris \& Barto(2009)Konidaris and
  Barto]{Konidaris2009SkillChaining}
Konidaris, G. and Barto, A.
\newblock Skill discovery in continuous reinforcement learning domains using
  skill chaining.
\newblock In \emph{Advances in Neural Information Processing Systems}, 2009.

\bibitem[Laskin et~al.(2020)Laskin, Lee, Stooke, Pinto, Abbeel, and
  Srinivas]{laskin_lee2020rad}
Laskin, M., Lee, K., Stooke, A., Pinto, L., Abbeel, P., and Srinivas, A.
\newblock Reinforcement learming with augmented data.
\newblock In \emph{Advances in Neural Information Processing Systems}, 2020.

\bibitem[Leibfried et~al.(2019)Leibfried, Pascual-D\'{\i}az, and
  Grau-Moya]{Leibfried2019rme}
Leibfried, F., Pascual-D\'{\i}az, S., and Grau-Moya, J.
\newblock A unified bellman optimality principle combining reward maximization
  and empowerment.
\newblock In \emph{Advances in Neural Information Processing Systems}, 2019.

\bibitem[Levine(2018)]{levine2018reinforcement}
Levine, S.
\newblock {Reinforcement Learning and Control as Probabilistic Inference}:
  Tutorial and review.
\newblock \emph{arXiv preprint arXiv:1805.00909}, 2018.

\bibitem[Li et~al.(2017)Li, Xu, Taylor, Studer, and
  Goldstein]{li2017visualizing}
Li, H., Xu, Z., Taylor, G., Studer, C., and Goldstein, T.
\newblock Visualizing the loss landscape of neural nets.
\newblock \emph{arXiv preprint arXiv:1712.09913}, 2017.

\bibitem[Lillicrap et~al.(2016)Lillicrap, Hunt, Pritzel, Heess, Erez, Tassa,
  Silver, and Wierstra]{Lillicrap2016}
Lillicrap, T.~P., Hunt, J.~J., Pritzel, A., Heess, N., Erez, T., Tassa, Y.,
  Silver, D., and Wierstra, D.
\newblock Continuous control with deep reinforcement learning.
\newblock In \emph{International Conference on Learning Representations}, 2016.

\bibitem[Maillard et~al.(2014)Maillard, Mann, and Mannor]{Maillard2014distnorm}
Maillard, O.-A., Mann, T.~A., and Mannor, S.
\newblock How hard is my mdp?" the distribution-norm to the rescue".
\newblock In \emph{Advances in Neural Information Processing Systems}, 2014.

\bibitem[McGovern \& Barto(2001)McGovern and Barto]{McGovernB2001bottleneck}
McGovern, A. and Barto, A.~G.
\newblock Automatic discovery of subgoals in reinforcement learning using
  diverse density.
\newblock In \emph{International Conference on Machine Learning}, 2001.

\bibitem[Mnih et~al.(2015)Mnih, Kavukcuoglu, Silver, Rusu, Veness, Bellemare,
  Graves, Riedmiller, Fidjeland, Ostrovski, et~al.]{mnih2015human}
Mnih, V., Kavukcuoglu, K., Silver, D., Rusu, A.~A., Veness, J., Bellemare,
  M.~G., Graves, A., Riedmiller, M., Fidjeland, A.~K., Ostrovski, G., et~al.
\newblock Human-level control through deep reinforcement learning.
\newblock \emph{Nature}, 2015.

\bibitem[Mohamed \& Rezende(2015)Mohamed and Rezende]{Mohamed2015VariationalIM}
Mohamed, S. and Rezende, D.~J.
\newblock Variational information maximisation for intrinsically motivated
  reinforcement learning.
\newblock In \emph{Advances in Neural Information Processing Systems}, 2015.

\bibitem[Nachum et~al.(2019)Nachum, Gu, Lee, and
  Levine]{nachum2019near-optimal}
Nachum, O., Gu, S., Lee, H., and Levine, S.
\newblock Near-optimal representation learning for hierarchical reinforcement
  learning.
\newblock In \emph{international conference on learning representations}, 2019.

\bibitem[Oller et~al.(2020)Oller, Glasmachers, and Cuccu]{oller2020analyzing}
Oller, D., Glasmachers, T., and Cuccu, G.
\newblock Analyzing reinforcement learning benchmarks with random weight
  guessing.
\newblock In \emph{International Conference on Autonomous Agents and
  Multi-Agent Systems}, 2020.

\bibitem[Peng et~al.(2019)Peng, Kumar, Zhang, and Levine]{peng2019awr}
Peng, X.~B., Kumar, A., Zhang, G., and Levine, S.
\newblock {Advantage-Weighted Regression}: Simple and scalable off-policy
  reinforcement learning.
\newblock \emph{arXiv preprint arXiv:1910.00177}, 2019.

\bibitem[Pong et~al.(2018)Pong, Gu, Dalal, and Levine]{pong2018temporal}
Pong, V., Gu, S., Dalal, M., and Levine, S.
\newblock Temporal difference models: Model-free deep rl for model-based
  control.
\newblock \emph{arXiv preprint arXiv:1802.09081}, 2018.

\bibitem[Pong et~al.(2020)Pong, Dalal, Lin, Nair, Bahl, and
  Levine]{pong2020skewfit}
Pong, V.~H., Dalal, M., Lin, S., Nair, A., Bahl, S., and Levine, S.
\newblock Skew-fit: State-covering self-supervised reinforcement learning.
\newblock In \emph{International Conference on Machine Learning}, 2020.

\bibitem[Poole et~al.(2019)Poole, Ozair, Van Den~Oord, Alemi, and
  Tucker]{poole2019vbmi}
Poole, B., Ozair, S., Van Den~Oord, A., Alemi, A., and Tucker, G.
\newblock On variational bounds of mutual information.
\newblock In \emph{International Conference on Machine Learning}, 2019.

\bibitem[Rajeswaran et~al.(2017)Rajeswaran, Lowrey, Todorov, and
  Kakade]{Rajeswaran2017}
Rajeswaran, A., Lowrey, K., Todorov, E.~V., and Kakade, S.~M.
\newblock Towards generalization and simplicity in continuous control.
\newblock In \emph{Advances in Neural Information Processing Systems}, 2017.

\bibitem[Rawlik et~al.(2012)Rawlik, Toussaint, and Vijayakumar]{rawlik2012psi}
Rawlik, K., Toussaint, M., and Vijayakumar, S.
\newblock On stochastic optimal control and reinforcement learning by
  approximate inference.
\newblock In \emph{International Joint Conference on Artificial Intelligence},
  2012.

\bibitem[Recht(2018)]{recht2018tour}
Recht, B.
\newblock A tour of reinforcement learning: The view from continuous control.
\newblock \emph{arXiv preprint arXiv:1806.09460}, 2018.

\bibitem[Salimans et~al.(2017)Salimans, Ho, Chen, Sidor, and
  Sutskever]{salimans2017evolution}
Salimans, T., Ho, J., Chen, X., Sidor, S., and Sutskever, I.
\newblock Evolution strategies as a scalable alternative to reinforcement
  learning.
\newblock \emph{arXiv preprint arXiv:1703.03864}, 2017.

\bibitem[Saxe et~al.(2019)Saxe, Bansal, Dapello, Advani, Kolchinsky, Tracey,
  and Cox]{saxe2019information}
Saxe, A.~M., Bansal, Y., Dapello, J., Advani, M., Kolchinsky, A., Tracey,
  B.~D., and Cox, D.~D.
\newblock On the information bottleneck theory of deep learning.
\newblock \emph{Journal of Statistical Mechanics: Theory and Experiment}, 2019.

\bibitem[Schulman et~al.(2015)Schulman, Levine, Moritz, Jordan, and
  Abbeel]{Schulman:2015uk}
Schulman, J., Levine, S., Moritz, P., Jordan, M., and Abbeel, P.
\newblock Trust region policy optimization.
\newblock In \emph{International Conference on Machine Learning}, 2015.

\bibitem[Schulman et~al.(2017)Schulman, Wolski, Dhariwal, Radford, and
  Klimov]{Schulman2017PPO}
Schulman, J., Wolski, F., Dhariwal, P., Radford, A., and Klimov, O.
\newblock Proximal policy optimization algorithms.
\newblock \emph{arXiv preprint arXiv:1707.06347}, 2017.

\bibitem[Sharma et~al.(2020{\natexlab{a}})Sharma, Ahn, Levine, Kumar, Hausman,
  and Gu]{sharma2020emergent}
Sharma, A., Ahn, M., Levine, S., Kumar, V., Hausman, K., and Gu, S.
\newblock Emergent real-world robotic skills via unsupervised off-policy
  reinforcement learning.
\newblock In \emph{Robotics: Science and Systems}, 2020{\natexlab{a}}.

\bibitem[Sharma et~al.(2020{\natexlab{b}})Sharma, Gu, Levine, Kumar, and
  Hausman]{sharma2020dynamics}
Sharma, A., Gu, S., Levine, S., Kumar, V., and Hausman, K.
\newblock Dynamics-aware unsupervised discovery of skills.
\newblock In \emph{International conference on learning representations},
  2020{\natexlab{b}}.

\bibitem[Smith \& Le(2017)Smith and Le]{smith2017bayesian}
Smith, S.~L. and Le, Q.~V.
\newblock A bayesian perspective on generalization and stochastic gradient
  descent.
\newblock \emph{arXiv preprint arXiv:1710.06451}, 2017.

\bibitem[Smith et~al.(2017)Smith, Kindermans, Ying, and Le]{smith2017don}
Smith, S.~L., Kindermans, P.-J., Ying, C., and Le, Q.~V.
\newblock Don't decay the learning rate, increase the batch size.
\newblock \emph{arXiv preprint arXiv:1711.00489}, 2017.

\bibitem[Strehl et~al.(2006)Strehl, Li, Wiewiora, Langford, and
  Littman]{strehl2006pac}
Strehl, A.~L., Li, L., Wiewiora, E., Langford, J., and Littman, M.~L.
\newblock Pac model-free reinforcement learning.
\newblock In \emph{International Conference on Machine Learning}, 2006.

\bibitem[Strehl et~al.(2009)Strehl, Li, and Littman]{Strehl2009pac}
Strehl, A.~L., Li, L., and Littman, M.~L.
\newblock Reinforcement learning in finite mdps: Pac analysis.
\newblock \emph{Journal of Machine Learning Research}, 2009.

\bibitem[Tassa et~al.(2018)Tassa, Doron, Muldal, Erez, Li, de~Las~Casas,
  Budden, Abdolmaleki, Merel, Lefrancq, Lillicrap, and
  Riedmiller]{tassa2018deepmind}
Tassa, Y., Doron, Y., Muldal, A., Erez, T., Li, Y., de~Las~Casas, D., Budden,
  D., Abdolmaleki, A., Merel, J., Lefrancq, A., Lillicrap, T., and Riedmiller,
  M.
\newblock Deepmind control suite.
\newblock \emph{arXiv preprint arXiv:1801.00690}, 2018.

\bibitem[Tishby \& Polani(2011)Tishby and Polani]{tishby2011information}
Tishby, N. and Polani, D.
\newblock Information theory of decisions and actions.
\newblock In \emph{Perception-action cycle}. Springer, 2011.

\bibitem[Todorov et~al.(2012)Todorov, Erez, and Tassa]{todorov2012mujoco}
Todorov, E., Erez, T., and Tassa, Y.
\newblock Mujoco: A physics engine for model-based control.
\newblock In \emph{International Conference on Intelligent Robots and Systems},
  2012.

\bibitem[Todorov et~al.(2006)]{todorov2006linearly}
Todorov, E. et~al.
\newblock Linearly-solvable markov decision problems.
\newblock In \emph{Advances in Neural Information Processing Systems}, 2006.

\bibitem[Toussaint \& Storkey(2006)Toussaint and
  Storkey]{toussaint2006probabilistic}
Toussaint, M. and Storkey, A.
\newblock Probabilistic inference for solving discrete and continuous state
  markov decision processes.
\newblock In \emph{Proceedings of the 23rd international conference on Machine
  learning}, pp.\  945--952, 2006.

\bibitem[Wang et~al.(2020)Wang, Salakhutdinov, and Yang]{wang2020eluder}
Wang, R., Salakhutdinov, R.~R., and Yang, L.
\newblock Reinforcement learning with general value function approximation:
  Provably efficient approach via bounded eluder dimension.
\newblock In \emph{Advances in Neural Information Processing Systems}, 2020.

\bibitem[Warde{-}Farley et~al.(2019)Warde{-}Farley, de~Wiele, Kulkarni,
  Ionescu, Hansen, and Mnih]{warde-farley2019discern}
Warde{-}Farley, D., de~Wiele, T.~V., Kulkarni, T.~D., Ionescu, C., Hansen, S.,
  and Mnih, V.
\newblock Unsupervised control through non-parametric discriminative rewards.
\newblock In \emph{international conference on learning representations}, 2019.

\bibitem[Yang et~al.(2020)Yang, Jin, Wang, Wang, and Jordan]{yang2020neuralrl}
Yang, Z., Jin, C., Wang, Z., Wang, M., and Jordan, M.
\newblock Provably efficient reinforcement learning with kernel and neural
  function approximations.
\newblock In \emph{Advances in Neural Information Processing Systems}, 2020.

\end{thebibliography}
\bibliographystyle{icml2021}

\clearpage
\onecolumn

\appendix
\section*{Appendix}
\setcounter{section}{0}
\renewcommand{\thesection}{\Alph{section}}

\section{Details of Environments}
\label{sec:env_details}
In this section, we explain the details of the environments used in our experiments. The properties of each environment are summarized in \autoref{table:environment_details}.

\begin{table*}[ht]
\vskip 0.15in
\begin{center}
\begin{small}
\begin{tabular}{l|cccc}
\toprule
Environment & State dim & Action dim & Control & Episode Length \\
\midrule
CartPole-v0 & 4 & 2 & Discrete & 200 \\
Pendulum-v0 & 3 & 1 & Continuous & 200 \\
MountainCar-v0 & 2 & 3 & Discrete & 200 \\
MountainCarContinuous-v0 & 2 & 1 & Continuous & 999 \\
Acrobot-v1 & 6 & 3 & Discrete & 500 \\
\midrule
Ant-v2 & 111 & 8 & Continuous & 1000 \\
HalfCheetah-v2 & 17 & 6 & Continuous & 1000 \\
Hopper-v2 & 11 & 3 & Continuous & 1000 \\
Walker2d-v2 & 17 & 6 & Continuous & 1000 \\
Humanoid-v2 & 376 & 17 & Continuous & 1000 \\
\midrule
cheetah run & 17 & 6 & Continuous & 1000 \\
reacher easy & 6 & 2 & Continuous & 1000 \\
ball\_in\_cup catch & 8 & 2 & Continuous & 1000 \\
\midrule
Reacher & 11 & 2 & Continuous & 50 \\
Pointmaze & 4 & 2 & Continuous & 150 \\
\bottomrule
\end{tabular}
\end{small}
\end{center}
\vskip -0.15in
\caption{The details of Open AI Gym and DeepMind Control Suite Environments used in this paper.} 
\label{table:environment_details}
\end{table*}

\paragraph{CartPole} The states of CartPole environment are composed of 4-dimension real values, cart position, cart velocity, pole angle, and pole angular velocity. The initial states are uniformly randomized.
The action space is discretized into 2-dimension.
The cumulative reward corresponds to the time steps in which the pole isn't fallen down. 

\paragraph{Pendulum} The states of Pendulum environment are composed of 3-dimension real values, cosine of pendulum angle $\cos{\varphi}$, sine of pendulum angle $\sin{\varphi}$, and pendulum angular velocity $\dot{\varphi}$. The initial states are uniformly randomized.
The action space is continuous and 1-dimension.
The reward is calculated by the following equation;
\begin{equation}
r_t = -(\varphi_t^2 + 0.1 * \dot{\varphi_t}^2 + 0.001 * \|\bm{a}_t\|_2^2) \nonumber.
\end{equation}

\paragraph{MountainCar} The states of MountainCar environment are composed of 2-dimension real values, car position, and its velocity. The initial states are uniformly randomized.
The action space is discretized into 3-dimension.
The cumulative reward corresponds to a negative value of the time steps in which the car doesn't reach the goal region. 

\paragraph{MountainCarContinuous} The 1-dimensional continuous action space version of MountainCar environment. The reward is calculated by the following equation;
\begin{equation}
r_t = - 0.1 * \|\bm{a}_t\|_2^2 + 100 * \mathbbm{1}[x_t \geq x_g~\text{and}~\dot{x}_t \geq \dot{x}_g] \nonumber.
\end{equation}

\paragraph{Acrobot} The states of Acrobot environment consist of 6-dimension real values, sine, and cosine of the two rotational joint angles and the joint angular velocities. The initial states are uniformly randomized.
The action space is continuous and 3-dimension.
The cumulative reward corresponds to a negative value of the time steps in which the end-effector isn't swung up at a height at least the length of one link above the base.

\paragraph{Ant}
The state space is 111-dimension, position and velocity of each joint, and contact forces. The initial states are uniformly randomized. The action is an 8-dimensional continuous space. The reward is calculated by the following equation;
\begin{equation}
r_t = \dot{x}_t - 0.5 * \|\bm{a}_t\|_2^2 - 0.0005 * \|\bm{s}_t^{\text{contact}}\|_2^2 + 1 \nonumber.
\end{equation}

\paragraph{HalfCheetah}
The state space is 17-dimension, position and velocity of each joint. The initial states are uniformly randomized. The action is a 6-dimensional continuous space. The reward is calculated by the following equation;
\begin{equation}
r_t = \dot{x}_t - 0.1 * \|\bm{a}_t\|_2^2 \nonumber.
\end{equation}

\paragraph{Hopper}
The state space is 11-dimension, position and velocity of each joint. The initial states are uniformly randomized. The action is a 3-dimensional continuous space. This environment is terminated when the agent falls down. The reward is calculated by the following equation;
\begin{equation}
r_t = \dot{x}_t - 0.01 * \|\bm{a}_t\|_2^2 + 1 \nonumber.
\end{equation}

\paragraph{Walker2d}
The state space is 17-dimension, position and velocity of each joint. The initial states are uniformly randomized. The action is a 6-dimensional continuous space. This environment is terminated when the agent falls down. The reward is calculated by the following equation;
\begin{equation}
r_t = \dot{x}_t - 0.01 * \|\bm{a}_t\|_2^2 + 1 \nonumber.
\end{equation}

\paragraph{Humanoid}
The state space is 376-dimension, position and velocity of each joint, contact forces, and friction of actuator. The initial states are uniformly randomized. The action is a 17-dimensional continuous space. This environment is terminated when the agent falls down. The reward is calculated by the following equation;
\begin{equation}
r_t = 1.25 * \dot{x}_t - 0.1 * \|\bm{a}_t\|_2^2 - \min (5e^{-7} * \|\bm{s}_t^{\text{contact}}\|_2^2,~10) + 5 \nonumber.
\end{equation}

\paragraph{cheetah run}
The state space is 17-dimension, position and velocity of each joint. The initial states are uniformly randomized. The action is a 6-dimensional continuous space. The reward is calculated by the following equation;
\begin{equation}
  r_t =\begin{cases}
    0.1 * \dot{x}_t & (0 \leq \dot{x}_t \leq 10) \\
    1 & (\dot{x}_t > 10). \nonumber
  \end{cases}
\end{equation}

\paragraph{reacher easy}
The state space is 6-dimension, position and velocity of each joint. The initial states are uniformly randomized. The action is a 2-dimensional continuous space. The reward is calculated by the following equation;
\begin{equation}
r_t = \mathbbm{1}[x_t - x_g \leq \epsilon] \nonumber.
\end{equation}

\paragraph{ball\_in\_cup catch}
The state space is 8-dimension, position and velocity of each joint. The initial states are uniformly randomized. The action is a 2-dimensional continuous space. The reward is calculated by the following equation;
\begin{equation}
r_t = \mathbbm{1}[x_t - x_g \leq \epsilon] \nonumber.
\end{equation}

\paragraph{Reacher}
The state space is 11-dimension, position and velocity of each joint. The initial states are uniformly randomized. The action is a 2-dimensional continuous space.
We modify the reward function for the experiments (see~\autoref{sec:effect_on_reward_shaping}).

\paragraph{Pointmaze}
We use the implementation provided by \citet{fu2020d4rl}~(\autoref{fig:pointmaze}). 
The state space is 4-dimension, position and velocity. The initial states are uniformly randomized. The action is a 2-dimensional continuous space.
We modify the reward function for the experiments (see~\autoref{sec:effect_on_reward_shaping}).

\clearpage
\section{Additional Details of Synthetic Experiments}
\label{sec:appendix_synthetic}
We provide the rest of the results, presented in Section~\ref{sec:controllability_vs_maximizability}, and then show additional comparisons between PIC or POIC and marginal entropies in this setting.

\subsection{Answer to (2) in Section~\ref{sec:controllability_vs_maximizability}}
First, \autoref{fig:opt_2_three_step_mdp_ps10_learning_curve} shows the full results of \autoref{fig:opt_2_three_step_mdp_ps10_learning_curve2}, which holds the same tendency overall.

Similar to the results of POIC, presented in Section~\ref{sec:controllability_vs_maximizability}, we also investigate what happens to PIC throughout a realistic learning process, optimizing $\mu$ in $p(\theta)$ to solve the MDP by evolution strategy and observing this metric and performance of the agent during training.
Following the experimental settings of POIC, we assume the Gaussian prior, $\mathcal{N}(\mu, I)$, and vary $\mu$ initializations using $[-10, -5, -4, -3, -2, -1, 0, 1, 2, 3, 4, 5, 10]$.
We set $N=M=1000$, and horizon is $T=3$.
\autoref{fig:three_step_mdp_ps10_learning_curve} reveals that the initial prior with high PIC (e.g. $\mu=0$) actually solves the environment faster than those with low PIC (e.g. $\mu=-3.0, -4.0, -5.0$), which seems the same as the case of POIC.
Interestingly, \autoref{fig:three_step_mdp_ps10_learning_curve} also shows that high PIC corresponds to regions of fastest learning.

\begin{figure}[ht]
\centering
\includegraphics[width=0.8\linewidth]{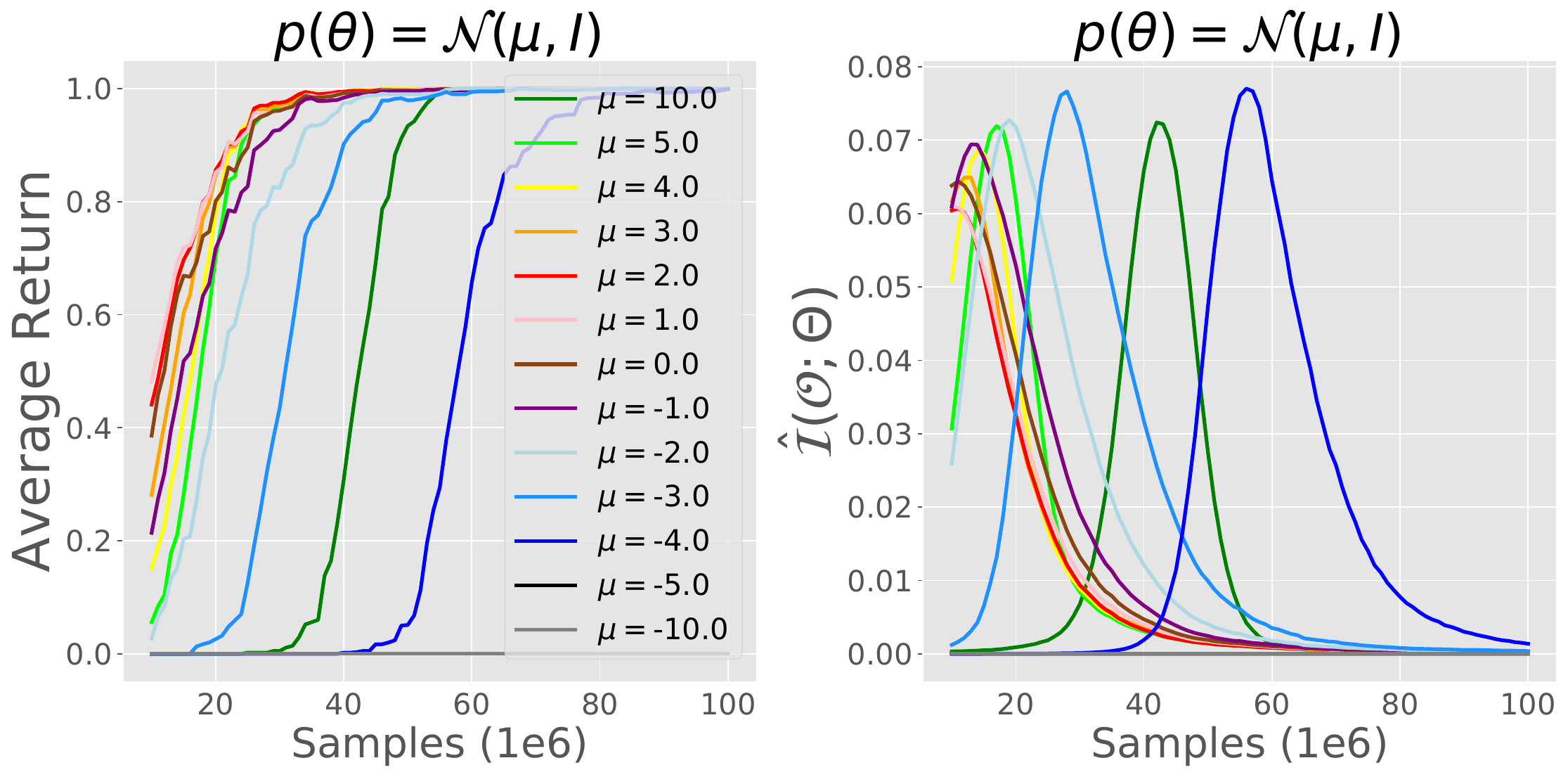}  
\vskip -0.05in
\caption{Average return (left) and POIC (right) during the training of evolution strategies. We change the parameter of initial prior distribution $\mu$ and optimize it.}
\label{fig:opt_2_three_step_mdp_ps10_learning_curve}
\end{figure}

\begin{figure*}[ht]
\begin{center}
\begin{tabular}{c}
\begin{minipage}{0.48\hsize}
\begin{center}
\includegraphics[width=\hsize]{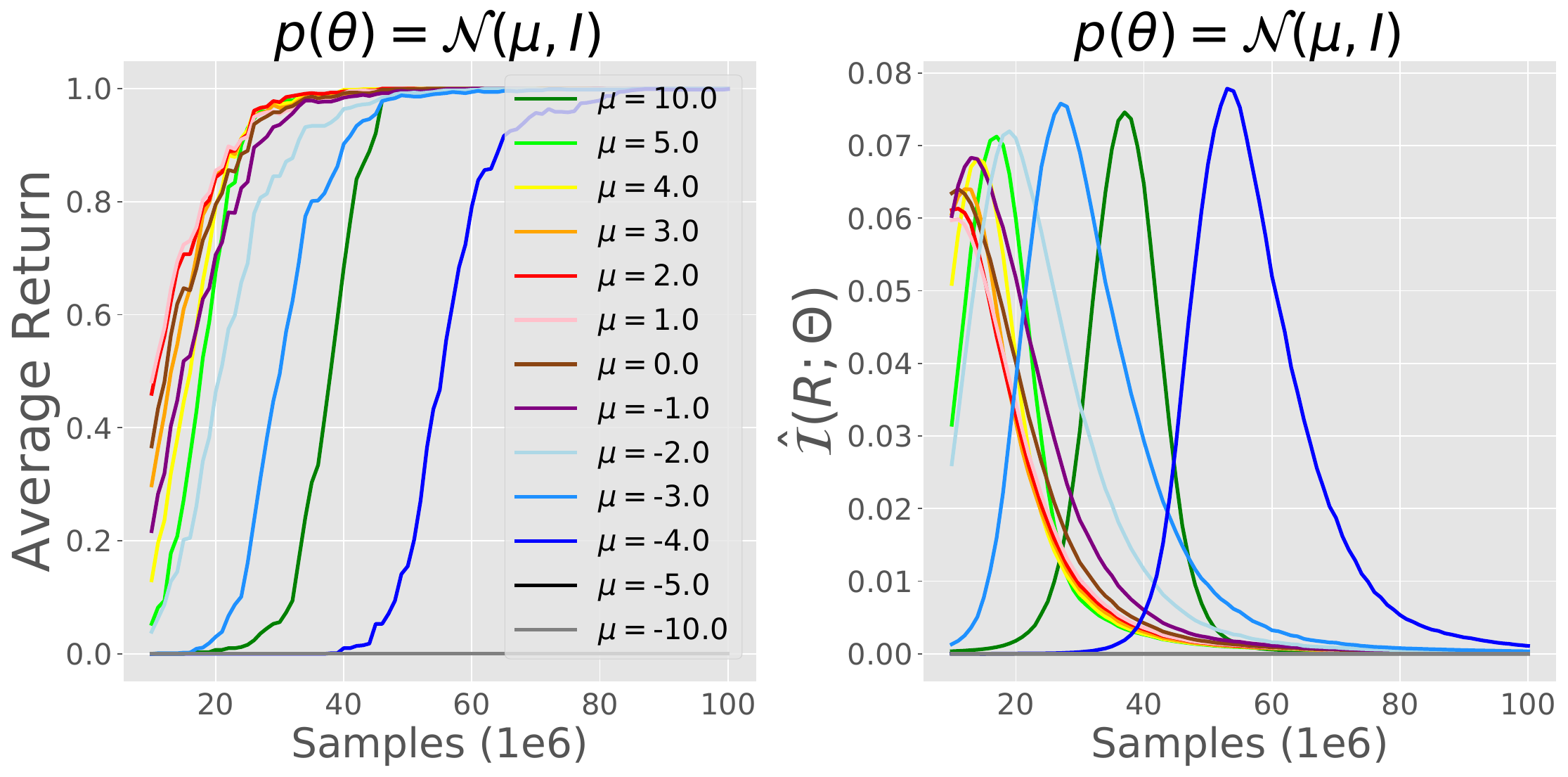}  
\hspace{1.5cm} \small{(a)\quad All results}
\end{center}
\end{minipage}
\begin{minipage}{0.48\hsize}
\begin{center}
\includegraphics[width=\hsize]{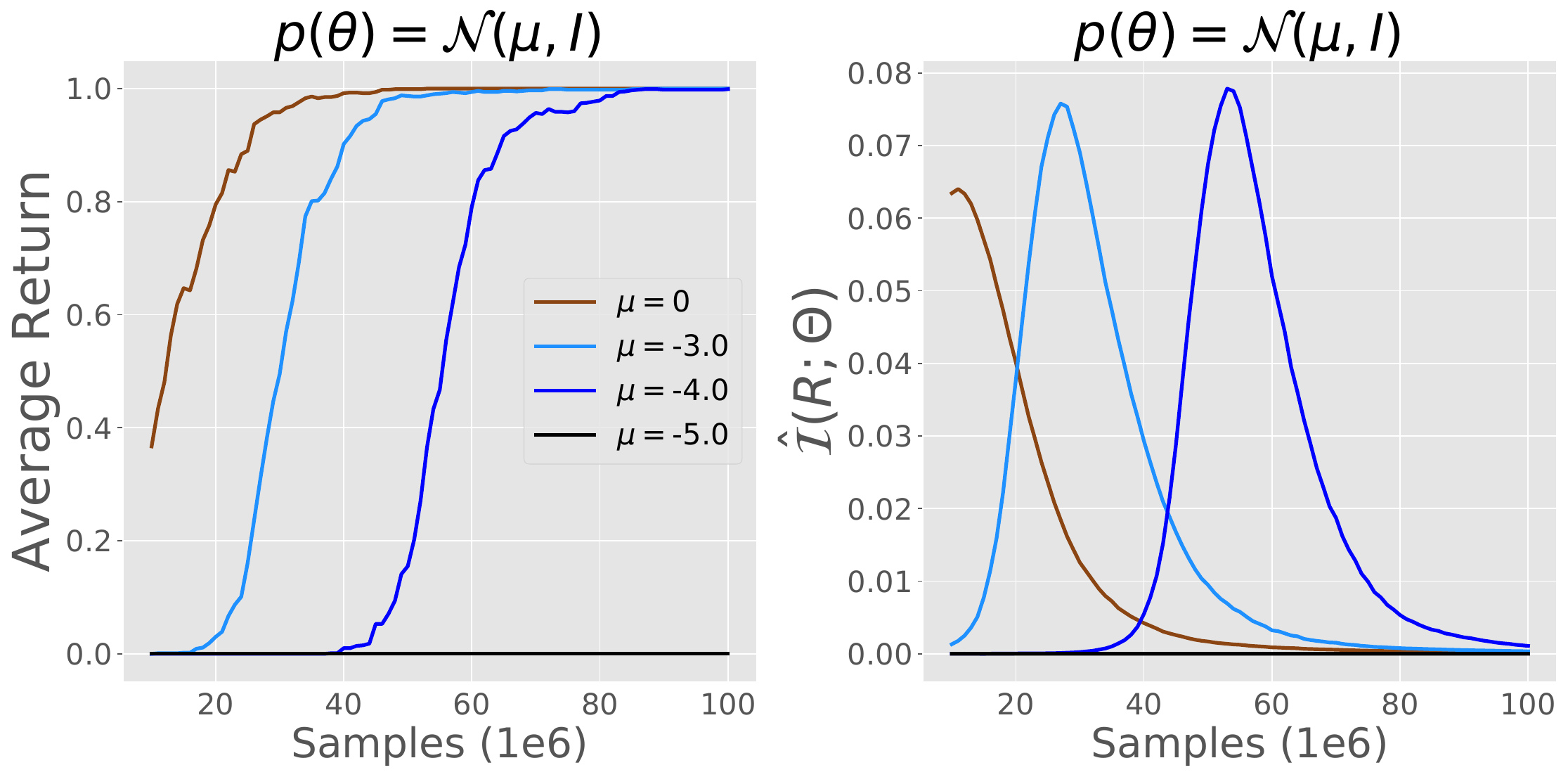}  
\hspace{1.5cm} \small{(b)\quad $\mu=0.0, -3.0, -4.0, -5.0$}
\end{center}
\end{minipage}
\end{tabular}
\vskip -0.05in
\caption{Average return and PIC during the training of evolution strategies. We vary the parameter of initial prior distribution $\mu$ and optimize it. (a) shows the all results, and in (b), we extract a few of them ($\mu=0.0, -3.0, -4.0, -5.0$; we separate these for the visibility). High PIC surely corresponds to regions of faster learning, which is similar trends of POIC.
}
\label{fig:three_step_mdp_ps10_learning_curve}
\end{center}
\end{figure*}

\subsection{Are PIC and POIC more suitable for evaluating task solvability than other alternatives?}

We consider the relations between the normalized score and each reward-based metric ($\mathcal{I}(R;\Theta)$, $\mathcal{H}(R)$, $\mathcal{H}(R|\Theta)$), or optimality-based metric ($\mathcal{I}(\mathcal{O};\Theta)$, $\mathcal{H}(\mathcal{O})$, $\mathcal{H}(\mathcal{O}|\Theta)$) on the MDP $\mathcal{M}$, for a variety of the prior distributions.
We test two variants of Gaussian prior, (a) $\mathcal{N}(\mu, I)$~(changing $\mu \in \{-10, -5, -4, -3, -2, -1, 0, 1, 2, 3, 4, 5, 10\}$), and (b) $\mathcal{N}(0, \sigma^2I)$~(changing $\sigma \in \{0.1, 0.2, ..., 0.9, 1.0\}$).
We set $N=M=1000$, and the horizon is $T=3$ (note that as far as we observed, $N=M=100$ could not estimate POIC metric properly).
\autoref{fig:three_step_mdp} indicates the relation between each metric of reward and normalized score. In (a), we change $\mu$, and in (b) we change $\sigma$.
We see that PIC shows positive correlation to normalized score both in (a) ($R=0.95$) and (b) ($R=0.78$).
The marginal and conditional entropies show, however, positive correlation in (a) (both $R=0.99$), but negative in (b) ($R=-0.83$ and $R=-0.80$). This suggests that there are cases where the marginal or conditional entropy alone might not reflect task solvability appropriately.
Similar to PIC, we see in \autoref{fig:opt_three_step_mdp} that POIC shows positive correlation to normalized score both in (a) ($R=0.94$) and (b) ($R=0.69$).
The marginal and conditional entropies show, however, positive correlation in (a) (both $R=0.99$), but negative in (b) ($R=-0.82$ and $R=-0.72$).

These observations suggest that there are cases where the marginal or conditional entropy alone might not reflect task solvability appropriately.

\begin{figure*}[ht]
\begin{center}
\begin{tabular}{c}
\begin{minipage}{0.5\hsize}
\begin{center}
\includegraphics[width=\hsize]{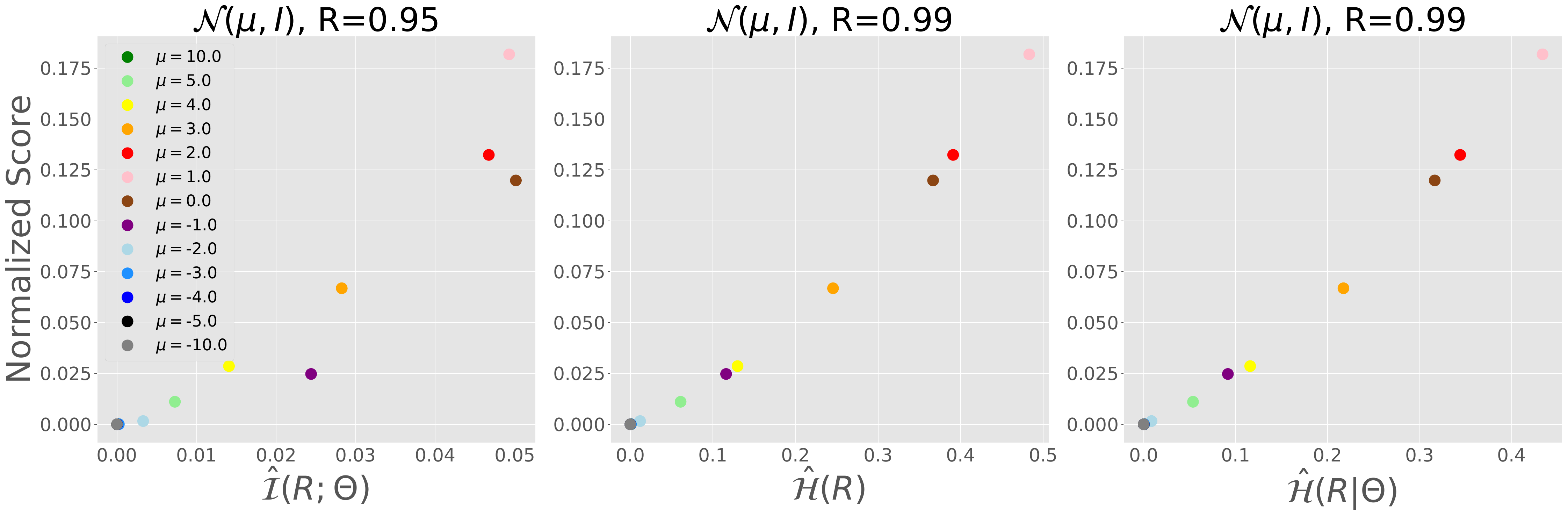}  
\hspace{1.5cm} \small{(a)\quad$p(\theta)=\mathcal{N}(\mu, I)$}
\end{center}
\end{minipage}
\begin{minipage}{0.5\hsize}
\begin{center}
\includegraphics[width=\hsize]{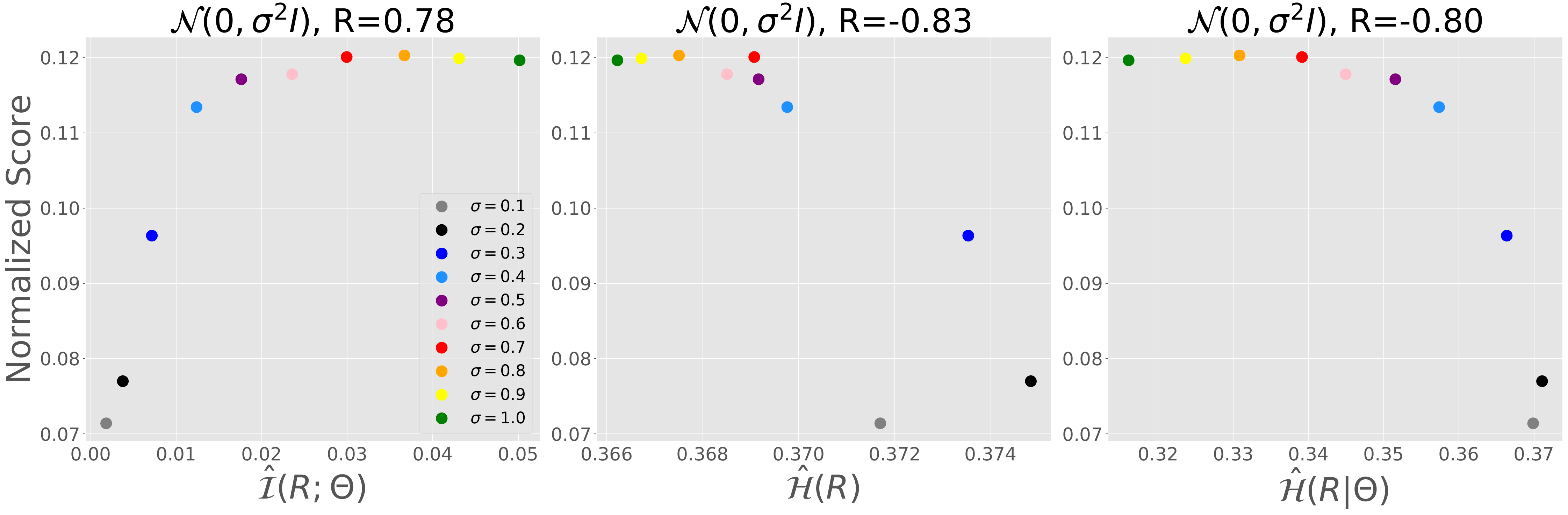}  
\hspace{1.5cm} \small{(b)\quad$p(\theta)=\mathcal{N}(0, \sigma^2I)$}
\end{center}
\end{minipage}
\end{tabular}
\vskip -0.05in
\caption{Relation between each metric (x-axis; PIC (left), marginal entropy (middle), conditional entropy (right)) and normalized score (y-axis). We compute the Pearson correlation coefficient (above the plots). In (a), we change $\mu$, and in (b) we change $\sigma$. PIC is the only metric that shows a consistently positive correlation. The marginal and conditional entropy doesn't have such consistency.}
\label{fig:three_step_mdp}
\end{center}
\end{figure*}

\begin{figure*}[ht]
\begin{center}
\begin{tabular}{c}
\begin{minipage}{0.5\hsize}
\begin{center}
\includegraphics[width=\hsize]{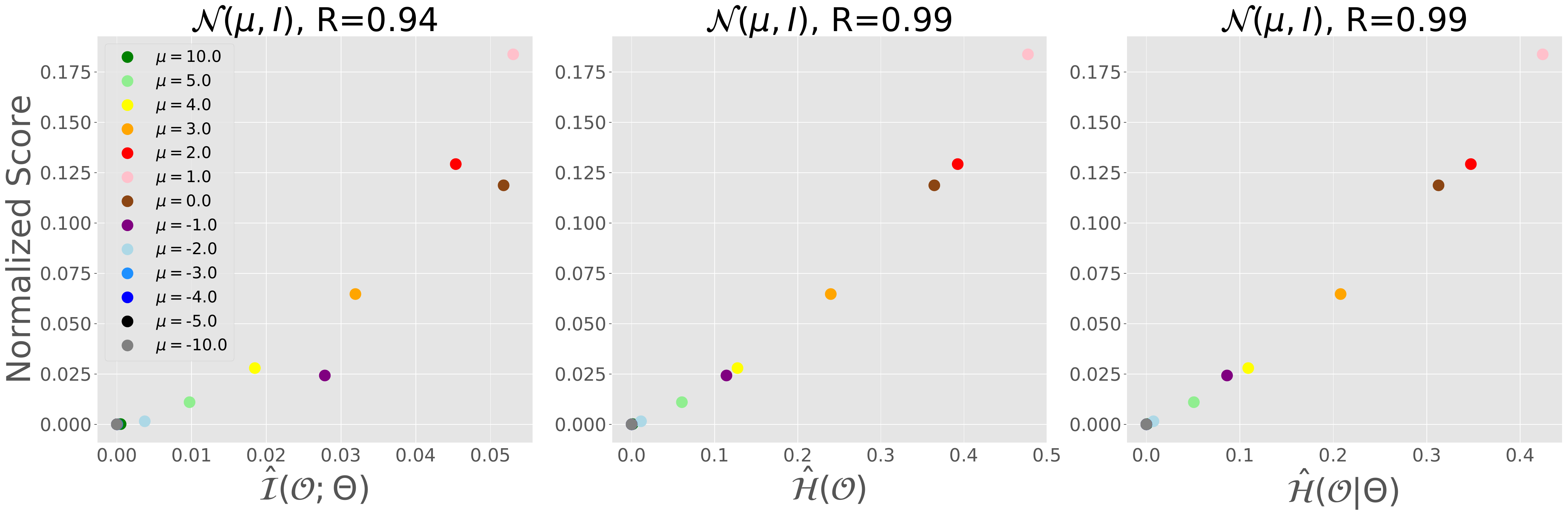}  
\hspace{1.5cm} \small{(a)\quad$p(\theta)=\mathcal{N}(\mu, I)$}
\end{center}
\end{minipage}
\begin{minipage}{0.5\hsize}
\begin{center}
\includegraphics[width=\hsize]{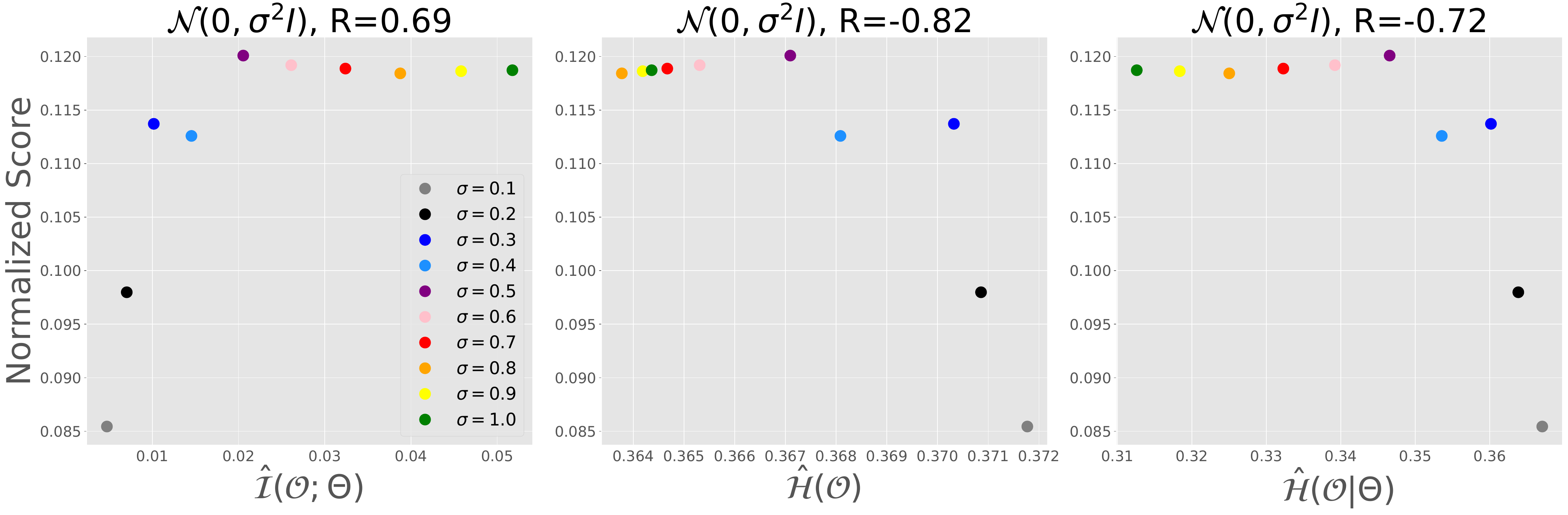}  
\hspace{1.5cm} \small{(b)\quad$p(\theta)=\mathcal{N}(0, \sigma^2I)$}
\end{center}
\end{minipage}
\end{tabular}
\vskip -0.05in
\caption{Relation between each metric (x-axis; POIC (left), marginal entropy (middle), conditional entropy (right)) and normalized score (y-axis). We compute the Pearson correlation coefficient (above the plots). In (a), we change $\mu$, and in (b) we change $\sigma$. Same as the case of PIC~(\autoref{fig:three_step_mdp}), POIC is the only metric that shows a consistently positive correlation. The marginal and conditional entropy doesn't have such consistency.}
\label{fig:opt_three_step_mdp}
\end{center}
\end{figure*}

\clearpage
\section{Normalized Score as Task Complexity Measure}
\label{sec:appendix_normalized_score}
In this section, we explain the details of normalized score. One intuitive and brute-force approach to measure the ground-truth task complexity in complex environments is to compare the performances of just an \textit{average} agent among the environments. When we try to realize this impractical method, we face two difficulties: (1) how can we obtain an \textit{average} agent? (2) how do we manage the different reward scale among the different environments?

To deal with (1), we have two options: substituting average score among diverse RL algorithms or random policy sampling score, for the performance of average agent.
We can also solve (2) by normalizing some average quantity divided by max-min value.
In the following section, we explain the details of each approach.

\subsection{Algorithm-based Normalized Score}
\label{sec:appendix_normalized_score_algo}
To compute the normalized score, we prepare the set of algorithms for evaluation and then execute them all. Since algorithms with different hyper-parameter behave differently, we treat them as different ``algorithms''.
After the evaluation, we compute the average return over the algorithms $r_{\text{ave}}^{\text{algo}}$.
This value, however, completely differs from all environments, due to the range of reward.
To normalize average return for the comparison over the environments, we compute the maximum return $r_{\max}$ and minimum return $r_{\min}$.
In practice, we take the maximum between this algorithm-based and random-sampling-based maximum scores, and use the minimum return obtained by random policy sampling:
\begin{equation}
\text{Normalized Score} = \frac{r_{\text{ave}}^{\text{algo}} - r_{\min}^{\text{rand}}}{\max(r_{\max}^{\text{rand}},~r_{\max}^{\text{algo}}) - r_{\min}^{\text{rand}}} \nonumber.
\end{equation}

We use three types of environments, classic control, MuJoCo, and DeepMind Control Suite~(DM Control).
For MuJoCo and DM Control, we test SAC, MPO and AWR. To simulate the diverse set of algorithms, we employ the leaderboard scores reported in previous SoTA works~\cite{fujimoto2018td3,peng2019awr,laskin_lee2020rad}.

\subsubsection{Classic Control}
For classic control, we run 23 algorithms, based on PPO, DQN and Evolution Strategy for discrete action space environments, and PPO, DDPG, SAC, and Evolution Strategy for continuous action space environments  with different hyper-parameters, such as network architecture or discount factor~(\autoref{table:algorithm_results_classic_control}).
We average each performance with 5 random seeds, and also average over algorithms.

\subsubsection{MuJoCo}
We test SAC, MPO and AWR, following hyper parameters in original implementations.
We average each performance with 10 random seeds and train each agent 1M steps for Hopper, 3M for Ant, HalfCheetah, Walker2d, and 10M for Humanoid (\autoref{table:appendix_mujoco_algo_raw}).
To simulate the diverse set of algorithms, we employ the leaderboard scores reported in previous works~\cite{fujimoto2018td3,peng2019awr} (\autoref{table:appendix_mujoco_algo_awr} and \autoref{table:appendix_mujoco_algo_td3}).
Totally, we use 17 algorithms for Ant, HalfCheetah, Hopper and Walker2d, and 10 algorithms for Humanoid to compute $r_{\text{ave}}^{\text{algo}}$ and $r_{\max}^{\text{algo}}$ (\autoref{table:appendix_mujoco_algo_ave_max}).

\subsubsection{DeepMind Control Suite}
We also test SAC, MPO and AWR, following hyper parameters in original implementations (\autoref{table:appendix_dm_algo_raw}).
We average each performance with 10 random seeds and train each agent 500k steps.
To simulate the diverse set of algorithms, we employ the leaderboard scores reported in previous work~\cite{laskin_lee2020rad} (\autoref{table:appendix_dm_algo_rad}).
Totally, we use 11 algorithms for cheetah run and ball\_in\_cup catch, and 10 algorithms for reacher easy to compute $r_{\text{ave}}^{\text{algo}}$ and $r_{\max}^{\text{algo}}$ (\autoref{table:appendix_dm_algo_ave_max}).

\subsection{Random-Sampling-based Normalized Score}
\label{sec:appendix_normalized_score_rs}
In \citet{oller2020analyzing}, they compute some representatives (e.g. 99.9 percentile score) obtained via random weight guessing and compare them qualitatively among the variety of environments.
We extend this idea to our settings -- quantitative comparison of the task difficulty.

Through the random policy sampling, we can compute the average cumulative reward $r_{\text{ave}}^{\text{rand}}$ and then normalize it using maximum return $r_{\max}$ and minimum return $r_{\min}$. 
In practice, we take the maximum between this algorithm-based and random-search-based maximum scores, and use the minimum return obtained by random policy search:
\begin{equation}
\text{Normalized Score} := \frac{r_{\text{ave}}^{\text{rand}} - r_{\min}^{\text{rand}}}{\max(r_{\max}^{\text{rand}},~r_{\max}^{\text{algo}}) - r_{\min}^{\text{rand}}} \nonumber.
\end{equation}

This method seems easy to use since we do not need extensive evaluations by a variety of RL algorithms. However, this random-sampling-based normalized score is highly biased towards the early stage of learning. It might not reflect the overall nature of environments properly.

\subsection{Correlation to Obvious Properties of MDP}
\label{sec:appendix_obvious_propaties}
In this section, we verify that these brute-force task solvability metrics do not just depend on obvious properties of MDP or policy networks, such as state and action dimensionalities, horizon, and the other type of normalized score.

\autoref{fig:peason_correlation_coefficient_sub} summarizes the correlation among those metrics.
While some properties such as action dimensions or episode length have negative correlations with algorithm-based normalized score, compared to \autoref{fig:peason_correlation_coefficient}, our proposed POIC seems much better than those metrics (see also \autoref{table:empowerment_full} for the details).

\begin{figure*}[ht]
\centering
\includegraphics[width=\linewidth]{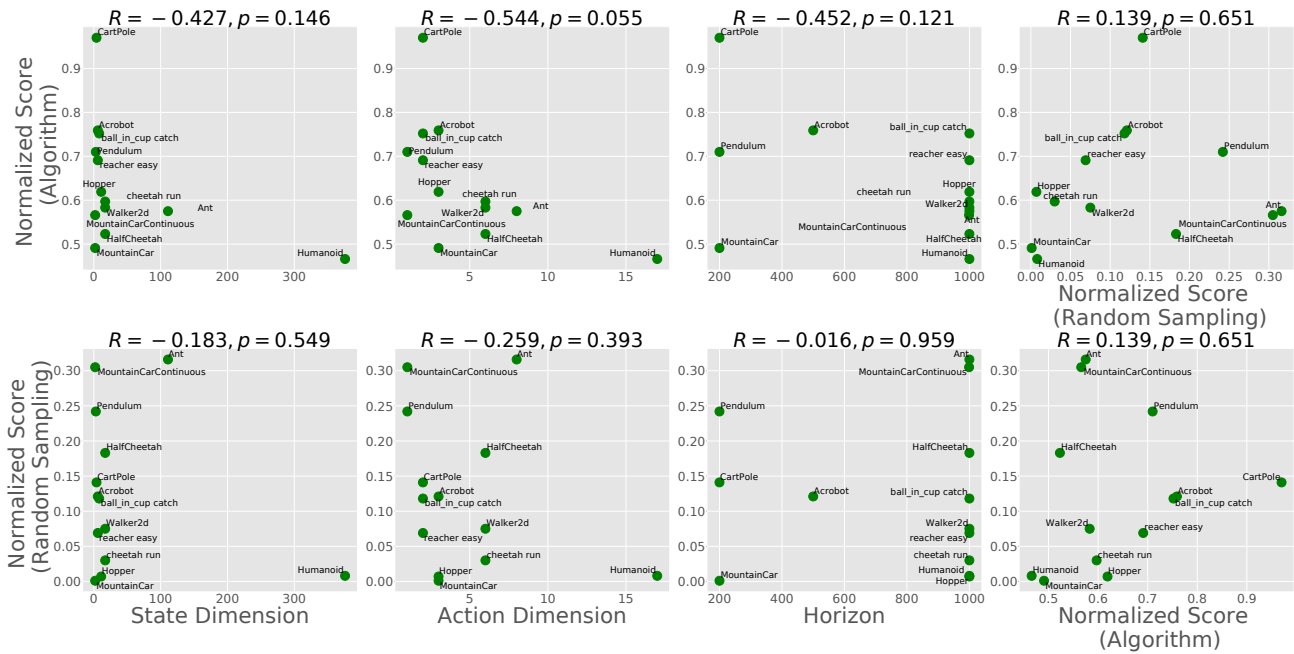}  
\vskip -0.05in
\caption{2D-Scatter plots between each metric (State dimension, action dimension, episode horizon, and normalized scores; x-axis) and each normalized score (algorithm-based (top) and random-sampling-based (bottom); y-axis).}
\label{fig:peason_correlation_coefficient_sub}
\end{figure*}

\begin{table*}[htb]
\begin{center}
\begin{small}
\begin{tabular}{l|l|ccccc}
\toprule
Algorithm & Hyper-Parameters & CartPole & Pendulum & MountainCar & MountainCarContinuous & Acrobot \\
\midrule
PPO & (64, 64), $\gamma=0.995$ & 200.00 & -1045.16 & -97.48 & 65.05 & -66.84 \\
PPO & (64, 64), $\gamma=0.99$ & 200.00 & -582.53 & -117.83 & 18.82 & -67.15 \\
PPO & (64, 64), $\gamma=0.95$ & 200.00 & -170.45 & -158.77 & 9.53 & -67.42 \\
PPO & (128, 64), $\gamma=0.995$ & 200.00 & -1152.09 & -118.02 & 95.89 & -67.42 \\
PPO & (128, 64), $\gamma=0.99$ & 200.00 & -467.42 & -104.73 & 0.00 & -72.83 \\
PPO & (128, 64), $\gamma=0.95$ & 199.75 & -151.43 & -116.44 & 0.00 & -68.09 \\
PPO & (128, 128), $\gamma=0.995$ & 200.00 & -1143.59 & -97.22 & 56.81 & -64.83 \\
PPO & (128, 128), $\gamma=0.99$ & 200.00 & -707.25 & -97.21 & 30.75 & -64.90 \\
PPO & (128, 128), $\gamma=0.95$ & 199.79 & -161.28 & -98.82 & 0.00 & -67.25 \\
\midrule
ES & (16, 16), $\sigma=0.1$ & 200.00 & -1062.54 & -128.82 & 0.00 & -80.30 \\
ES & (16, 16), $\sigma=0.1$, rand & 200.00 & -1059.30 & -127.67 & 0.00 & -79.57  \\
ES & (16, 16), $\sigma=0.3$, rand & 143.19 & -1175.70 & -200.00 & 0.00 & -182.21 \\
ES & (64, 64), $\sigma=0.1$ & 200.00 & -999.26 & -136.48 & 0.00 & -80.97 \\
ES & (64, 64), $\sigma=0.1$, rand & 200.00 & -1012.32 & -131.85 & 0.00 & -80.83 \\
ES & (64, 64), $\sigma=0.3$, rand & 158.43 & -1180.50 & -200.00 & 0.00 & -257.48 \\
\midrule
DQN & (100, 100), $\gamma=0.95$ & 188.80 &-- & -200.00 &-- & -360.57 \\
DQN & (100, 100), $\gamma=0.99$ & 200.00 &-- & -164.76 &-- & -250.53 \\
DQN & (200, 200), $\gamma=0.95$ & 176.69 &-- & -192.70 &-- & -296.48 \\
DQN & (200, 200), $\gamma=0.99$ & 200.00 &-- & -150.86 &-- & -381.99 \\
DQN & (50, 50), $\gamma=0.95$ & 200.00 &-- & -154.91 &-- & -318.57 \\
DQN & (50, 50), $\gamma=0.99$ & 200.00 &-- & -167.02 &-- & -329.15 \\
DQN & (50, 50, 50), $\gamma=0.95$ & 200.00 &-- & -167.91 &-- & -274.28 \\
DQN & (50, 50, 50), $\gamma=0.99$ &200.00 &-- & -167.22 &-- & -169.92 \\
\midrule
SAC & (128, 128), $\gamma=0.99$ & -- & -129.30 & -- & 0.00 & -- \\
SAC & (128, 128), $\gamma=0.95$ & -- & -138.26 & -- & 0.00 & -- \\
SAC & (256, 256), $\gamma=0.99$ & -- & -128.63 & -- & 0.00 & -- \\
SAC & (256, 256), $\gamma=0.95$ & -- & -138.27  & -- & 19.21 & -- \\
\midrule
DDPG & (150, 50), $\gamma=0.95$ & -- & -138.76 & -- & 0.00 & -- \\
DDPG & (150, 50), $\gamma=0.99$ & -- & -130.35 & -- & 0.00 & -- \\
DDPG & (400, 300), $\gamma=0.95$ & -- & -138.61 & -- & 0.00 & -- \\
DDPG & (400, 300), $\gamma=0.99$ & -- & -131.21 & -- & 0.00 & -- \\
\midrule
$r_{\text{ave}}^{\text{algo}}$ & -- & 194.20 & -571.49 & -143.12 & 12.87 & -162.92 \\
\midrule
$r_{\max}^{\text{algo}}$ & -- & 200.00 & -128.63 & -97.21 & 95.89 & -64.83 \\
\bottomrule
\end{tabular}
\end{small}
\end{center}
\vskip -0.15in
\caption{Performance of a variety of algorithms in classic control. The results are averaged over 5 random seeds. We change 2 hyper-parameters, architecture of neural networks and discount factor $\gamma$.}
\label{table:algorithm_results_classic_control}
\end{table*}

\begin{table}[htb]
\begin{center}
\begin{small}
\begin{tabular}{l|cc}
\toprule
Environments & Average & Maximum \\
\midrule
Ant & 2450.8 & 6584.2 \\
HalfCheetah & 6047.2 & 15266.5 \\
Hopper & 2206.7 & 3564.1 \\
Walker2d & 3190.8 & 5813.0 \\
Humanoid & 3880.8 & 8264.0 \\
\bottomrule
\end{tabular}
\end{small}
\end{center}
\vskip -0.15in
\caption{Average and maximum scores in MuJoCo environments, calculated from 17 algorithms (10 for Humanoid) (\autoref{table:appendix_mujoco_algo_raw}, \autoref{table:appendix_mujoco_algo_awr} and \autoref{table:appendix_mujoco_algo_td3}).}
\label{table:appendix_mujoco_algo_ave_max}
\end{table}

\begin{table}[htb]
\begin{center}
\begin{small}
\begin{tabular}{l|ccc}
\toprule
Environments & SAC & MPO & AWR \\
\midrule
Ant & 5526.4 & 6584.2 & 1126.7 \\
HalfCheetah & 15266.5 & 11769.6 & 5742.4 \\
Hopper & 2948.9 & 2135.5 & 3084.7 \\
Walker2d & 5771.8 & 3971.5 & 4716.6 \\
Humanoid & 8264.0 & 5708.7 & 5572.6 \\
\bottomrule
\end{tabular}
\end{small}
\end{center}
\vskip -0.15in
\caption{SAC, MPO and AWR results in MuJoCo environments. These results are averaged across 10 seeds.}
\label{table:appendix_mujoco_algo_raw}
\end{table}

\begin{table*}[htb]
\begin{center}
\begin{small}
\begin{tabular}{l|ccccccc}
\toprule
Environments & TRPO & PPO & DDPG & TD3 & SAC & RWR & AWR \\
\midrule
Ant & 2901 & 1161 & 72 & 4285 & 5909 & 181 & 5067  \\
HalfCheetah & 3302 & 4920 & 10563 & 4309 & 9297 & 1400 & 9136  \\
Hopper & 1880 & 1391 & 855 & 935 & 2769 & 605 & 3405  \\
Walker2d & 2765 & 2617 & 401 & 4212 & 5805 & 406 & 5813  \\
Humanoid & 552 & 695 & 4382 & 81 & 8048 & 509 & 4996  \\
\bottomrule
\end{tabular}
\end{small}
\end{center}
\vskip -0.15in
\caption{Performance of a variety of algorithms in MuJoCo environments, reported by \citet{peng2019awr}. These results are averaged across 5 seeds.}
\label{table:appendix_mujoco_algo_awr}
\end{table*}

\begin{table*}[htb]
\begin{center}
\begin{small}
\begin{tabular}{l|ccccccc}
\toprule
Environments & TD3 & DDPG(1) & DDPG(2) & PPO & TRPO & ACKTR & SAC \\
\midrule
Ant & 4372 & 1005 & 889 & 1083 & -76 & 1822 & 655  \\
HalfCheetah & 9637 & 3306 & 8577 & 1795 & -16 & 1450 & 2347  \\
Hopper & 3564 & 2020 & 1860 & 2165 & 2471 & 2428 & 2997  \\
Walker2d & 4683 & 1844 & 3098 & 3318 & 2321 & 1217 & 1284  \\
\bottomrule
\end{tabular}
\end{small}
\end{center}
\vskip -0.15in
\caption{Performance of a variety of algorithms in MuJoCo environments, reported by \citet{fujimoto2018td3} after 1M steps. These results are averaged across 10 seeds.}
\label{table:appendix_mujoco_algo_td3}
\end{table*}

\begin{table}[htb]
\begin{center}
\begin{small}
\begin{tabular}{l|cc}
\toprule
Environments & Average & Maximum \\
\midrule
cheetah run & 474.4 & 795.0 \\
reacher easy & 691.5 & 961.2 \\
ball\_in\_cup catch & 751.7 & 978.2 \\
\bottomrule
\end{tabular}
\end{small}
\end{center}
\vskip -0.15in
\caption{Average and maximum scores in DM Control environments, calculated from 11 algorithms (10 for reacher easy) (\autoref{table:appendix_dm_algo_raw} and \autoref{table:appendix_dm_algo_rad}).}
\label{table:appendix_dm_algo_ave_max}
\end{table}

\begin{table}[htb]
\begin{center}
\begin{small}
\begin{tabular}{l|ccc}
\toprule
Environments & SAC & MPO & AWR \\
\midrule
cheetah run & 536.0 & 253.9 & 125.2 \\
reacher easy & 961.2 & 841.5 & 530.2 \\
ball\_in\_cup catch & 971.9 & 957.3 & 135.2 \\
\bottomrule
\end{tabular}
\end{small}
\end{center}
\vskip -0.15in
\caption{SAC, MPO and AWR results in DM Control environments after 500k steps. These results are averaged across 10 seeds.}
\label{table:appendix_dm_algo_raw}
\end{table}

\begin{table*}[htb]
\begin{center}
\begin{small}
\begin{tabular}{l|cccccccc}
\toprule
Environments & RAD & CURL & PlaNet & Dreamer & SAC+AE & SLACv1 & Pixel SAC & State SAC \\
\midrule
cheetah run & 728 & 518 & 305 & 570 & 550 & 640 & 197 & 795 \\
reacher easy & 955 & 929 & 210 & 793 & 627 & -- & 145 & 923 \\
ball\_in\_cup catch & 974 & 959 & 460 & 879 & 794 & 852 & 312 & 974 \\
\bottomrule
\end{tabular}
\end{small}
\end{center}
\vskip -0.15in
\caption{Performance of a variety of algorithms in DM Control environments, reported by \citet{laskin_lee2020rad}, after 500k training steps. These results are averaged across 10 seeds.}
\label{table:appendix_dm_algo_rad}
\end{table*}

\clearpage
\section{Policy and Policy-Optimal Information Capacity during ES Training}
\label{sec:appendix_es_empowerments}
We evaluate how PIC and POIC behave during RL training (reward maximization) on complex benchmarking environments, in contrast to synthetic ones in Section \ref{sec:controllability_vs_maximizability}.
We train linear, (4, 4) and (16, 16) neural networks with evolution strategy (ES)~\citep{salimans2017evolution}.
100 parameters are sampled from the trainable prior distribution $p_{\mu}(\theta)=\mathcal{N}(\mu, \sigma=0.1)$ (ES optimizes $\mu$) in each epoch, and the agent runs 100 episodes per parameters to calculate both PIC and POIC.
Section \ref{fig:es_optimality_full} shows the learning curves (top row), corresponding POIC $\mathcal{I}(\mathcal{O};\Theta)$ (middle) and PIC $\mathcal{I}(R;\Theta)$ (bottom).
As we observed in Section \ref{sec:controllability_vs_maximizability}, both information capacity metrics increase during training, and after each performance converges, they gradually decrease.
It might be related to higher correlation of POIC to algorithm-based normalized score shown in Section \ref{sec:Experiments} that POIC seems to follow these trends better than PIC (e.g. Pendulum and HalfCheetah).

Another interesting observation of POIC can be seen in HalfCheetah, where (4, 4) network (green; top) converges to sub-optimal solution and (16, 16) network (blue; top) gets away from there.
The agents that can have multi-modal solutions keeps high POIC (green; middle) after sub-optimal convergence, while POIC decreases as improving performance (blue; middle).
This might suggests that a sub-optimal prior distribution $p_{\mu}(\theta)=\mathcal{N}(\mu, \sigma=0.1)$ still can be easy to minimize the rewards (green), though the further improvements (blue) make it lean towards maximization (i.e. less controllable).
Measuring PIC and POIC with more familiar on-policy algorithms such as PPO or TRPO remains as future work.

\begin{figure}[htb]
\centering
\includegraphics[width=\linewidth]{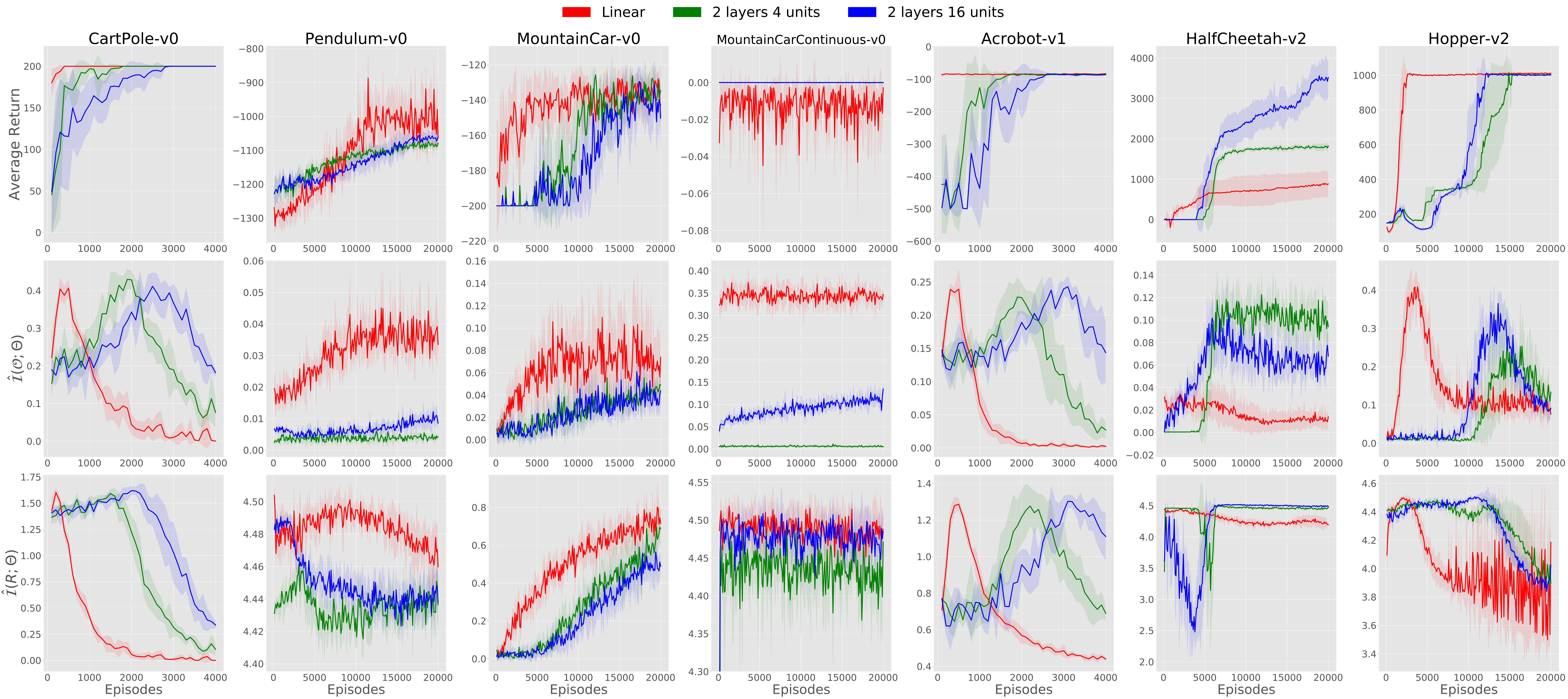}  
\vskip -0.05in
\caption{The average returns (top), POIC (middle) and PIC (bottom) during ES training. We use several classic control and MuJoCo tasks. Both information capacity metrics increase during training, and after each performance converges, they gradually decrease.
It might be related to higher correlation of POIC to algorithm-based normalized score shown in Section \ref{sec:Experiments} that POIC seems to follow these trends better than PIC (e.g. Pendulum and HalfCheetah).}
\label{fig:es_optimality_full}
\end{figure}

\clearpage
\section{A Proof of Proposition~\ref{prop:better policy determination}}
\label{sec:appendix_proof_of_reward_empowerment_and_policy_selection}
Assume that $\theta_1$ is better than $\theta_2$ without loss of generality.

The proof relies on the following Chernoff's bound tailored for a normal distribution \citep{boucheron2013concentration}: for $N$ independent samples $(X_n)_{n=1,\ldots,N}$ from a normal distribution $\mathcal{N}(\mu, \sigma^2)$, and $\mu$'s $N$-sample-estimate $\hat{\mu} := \sum_{n=1}^N X_n / N$,
\begin{align*}
    \Pr \left\{ \hat{\mu} \leq \mu - \varepsilon \right\} \leq \exp \left( - \frac{N \varepsilon^2}{2 \sigma^2} \right)
\end{align*}
holds. Applying this bound to our case with $\hat{\mu} = \hat{\mu}_1 - \hat{\mu}_2$, $\mu = \mu_{\theta_1} - \mu_{\theta_2}$, and $\varepsilon = \mu$, we obtain
\begin{align*}
\Pr \left\{ \hat{\mu}_1 - \hat{\mu}_2 \leq 0 \middle| \theta_1, \theta_2 \right\} \leq \exp \left( - \frac{N}{2} \left(\frac{\mu_{\theta_1} - \mu_{\theta_2}}{\sigma_{\theta_1} + \sigma_{\theta_2}} \right)^2 \right).
\end{align*}
Since the entropy of $\mathcal{N}(\mu_\theta, \sigma_\theta^2)$ is $\log (\sigma_{\theta} \sqrt{2 \pi e} )$, the upper bound can be rewritten as
\begin{align*}
\Pr \left\{ \hat{\mu}_1 - \hat{\mu}_2 \leq 0 \middle| \theta_1, \theta_2 \right\} \leq \exp \left( - \pi e N \left(\frac{\mu_{\theta_1} - \mu_{\theta_2}}{\exp (\mathcal{H}_1) + \exp( \mathcal{H}_2) } \right)^2 \right).
\end{align*}
Taking the expectation of both sides with respect to $\theta_1$ and $\theta_2$, the proof is concluded.

\section{Full Results on Deep RL Experiment}
\label{sec:appendix_full_results}
In this section, we provide the the full results on the experiment in Section~\ref{sec:mi_on_benchmark}.

We used max 40 CPUs for the experiments in Section \ref{sec:mi_on_benchmark} and it took about at most 2 hours per each random sampling (e.g. HalfCheetah-v2).
For estimating brute-force normalized score (Section~\ref{sec:norm_task_scores} and \autoref{sec:appendix_normalized_score}), we mainly used 4 GPUs (NVIDIA V100; 16GB) and it took about 4 hours per seed.


\paragraph{Correlation of Policy-Optimal Information Capacity}
As seen in \autoref{fig:peason_correlation_coefficient}, the relation between POIC and the algorithm-based normalized score might seem a bit concentrated or skewed with some outliers\footnote{We tried an early stopping when finding the $\argmax$ temperature with a black-box optimizer, but it didn't ease these concentrations.}.
To check the validity of correlation, we remove these outliers and recompute the correlation.
\autoref{fig:optimality_empowerment_rm_outlier} exhibits the strong positive correlation still holds ($R=0.780$, statistically significant with $p<0.01$.) after we remove top-3 outliers of POIC (CartPole, Acrobot, and MountainCarContinuous).

\begin{figure}[htb]
\centering
\includegraphics[width=0.4\linewidth]{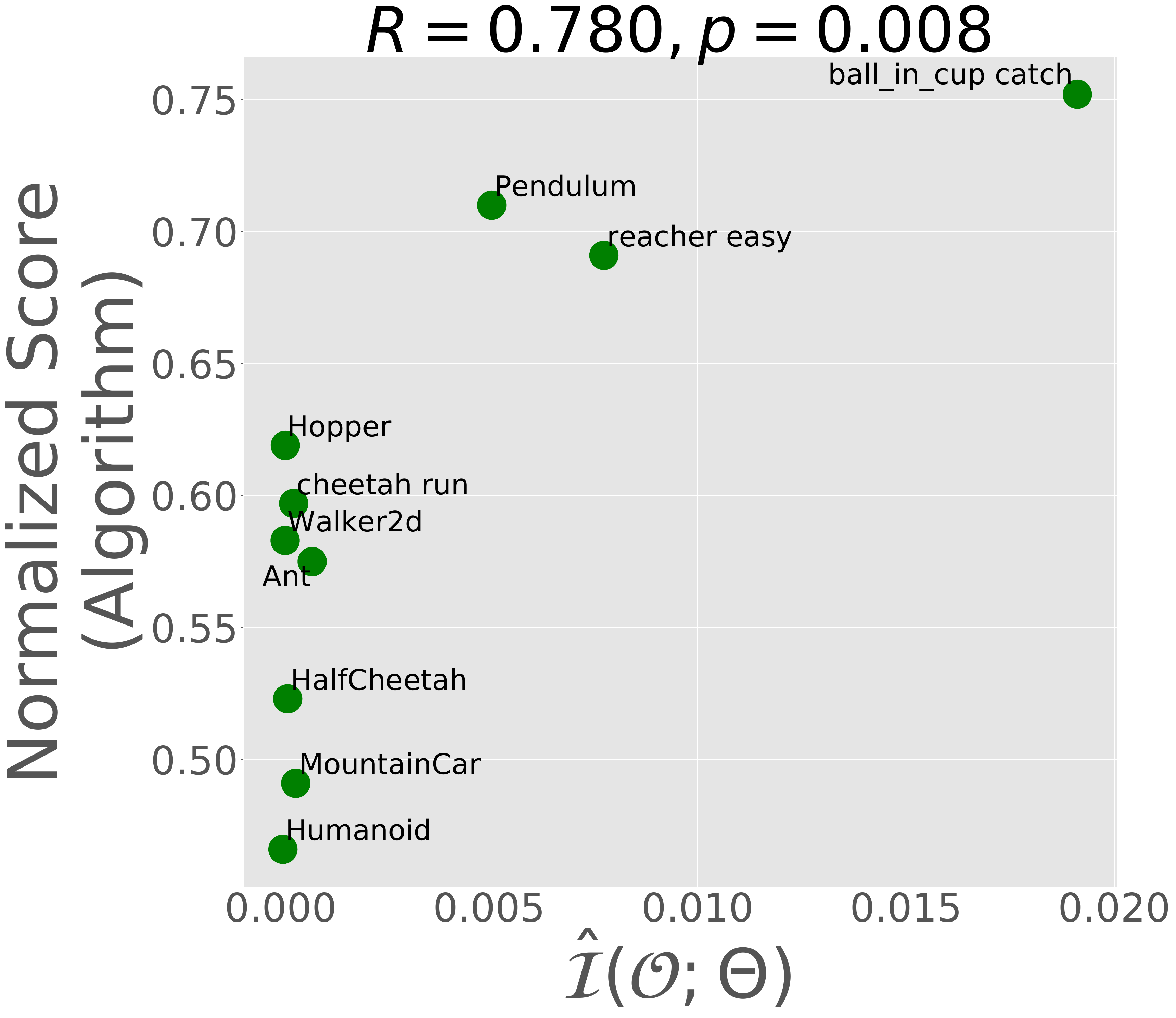}  
\vskip -0.05in
\caption{
The correlation between POIC and the algorithm-based normalized score.
The strong positive correlation still holds ($R=0.780$, statistically significant with $p<0.01$.) after we remove top-3 outliers of POIC (CartPole, Acrobot, and MountainCarContinuous).
}
\label{fig:optimality_empowerment_rm_outlier}
\end{figure}

\begin{landscape}
\begin{table}[p]
\setlength{\tabcolsep}{3pt}
\begin{center}
\begin{small}
\begin{tabular}{l|c|c|cccccccccc}
\toprule
Environment & Score(A) & Score(R) & $\hat{\mathcal{I}}(\mathcal{O};\Theta)$ & $\hat{\mathcal{H}}(\mathcal{O})$ & $\hat{\mathcal{H}}(\mathcal{O}|\Theta)$ & $\hat{\mathcal{I}}(R;\Theta)$ & $\hat{\mathcal{H}}(R)$ & $\hat{\mathcal{H}}(R|\Theta)$ & Variance & State dim & Action dim & Horizon \\
\midrule
CartPole & 0.970 & 0.141 & 0.153418 & 0.210 & 0.056 & 1.227 & 3.262 & 2.035 & 12.610 & 4 & 2 & 200 \\
Pendulum & 0.710 & 0.242 & 0.005060 & 0.374 & 0.369 & 3.708 & 10.520 & 6.812 & 23.223 & 3 & 1 & 200 \\
MountainCar & 0.491 & 0.001 & 0.000357 & 0.073 & 0.073 & 0.019 & 0.027 & 0.008 & 0.054 & 2 & 3 & 200 \\
MountainCarContinuous & 0.566 & 0.305 & 0.030092 & 0.424 & 0.394 & 5.953 & 10.095 & 4.142 & 8.022 & 2 & 1 & 999 \\
Acrobot & 0.759 & 0.121 & 0.106989 & 0.285 & 0.178 & 0.551 & 1.355 & 0.803 & 35.891 & 6 & 3 & 500 \\
Ant & 0.575 & 0.316 & 0.000751 & 0.501 & 0.500 & 1.767 & 8.010 & 6.243 & 37.830 & 111 & 8 & 1000 \\
HalfCheetah & 0.523 & 0.183 & 0.000165 & 0.503 & 0.503 & 2.488 & 8.764 & 6.276 & 23.468 & 17 & 6 & 1000 \\
Hopper & 0.619 & 0.007 & 0.000107 & 0.453 & 0.453 & 3.572 & 7.088 & 3.516 & 3.006 & 11 & 3 & 1000 \\
Walker2d & 0.583 & 0.075 & 0.000102 & 0.227 & 0.227 & 2.211 & 6.350 & 4.139 & 0.403 & 17 & 6 & 1000 \\
Humanoid & 0.466 & 0.008 & 5e-05 & 0.275 & 0.275 & 2.792 & 7.359 & 4.567 & 1.917 & 376 & 17 & 1000 \\
cheetah run & 0.597 & 0.030 & 0.000308 & 0.487 & 0.487 & 2.635 & 8.743 & 6.108 & 1.972 & 17 & 6 & 1000 \\
reacher easy & 0.691 & 0.069 & 0.007751 & 0.036 & 0.029 & 0.594 & 3.863 & 3.268 & 20.669 & 6 & 2 & 1000 \\
ball\_in\_cup catch & 0.752 & 0.118 & 0.019111 & 0.303 & 0.284 & 0.367 & 1.391 & 1.024 & 91.454 & 8 & 2 & 1000 \\
\midrule
Correlation Coefficient: Score(A) & -- & 0.139 & 0.807 & -0.212 & -0.418 & -0.295 & -0.349 & -0.327 & 0.372 & -0.427 & -0.544 & -0.452 \\
p-Value: Score(A) & -- & 0.651 & 0.001 & 0.487 & 0.156 & 0.328 & 0.242 & 0.275 & 0.211 & 0.146 & 0.055 & 0.121 \\
Correlation Coefficient: Score(R) & 0.139 & -- & 0.121 & 0.457 & 0.378 & 0.401 & 0.455 & 0.414 & 0.314 & -0.183 & -0.259 & -0.016 \\
p-Value: Score(R) & 0.651 & -- & 0.693 & 0.116 & 0.203 & 0.175 & 0.118 & 0.160 & 0.297 & 0.549 & 0.393 & 0.959 \\
\bottomrule
\end{tabular}
\end{small}
\end{center}
\vskip -0.15in
\caption{The raw results of POIC $\hat{\mathcal{I}}(\mathcal{O};\Theta)$, optimality marginal entropy $\hat{\mathcal{H}}(\mathcal{O})$, optimality conditional entropy $\hat{\mathcal{H}}(\mathcal{O}|\Theta)$, PIC $\hat{\mathcal{I}}(R;\Theta)$, reward marginal entropy $\hat{\mathcal{H}}(R)$, reward conditional entropy $\hat{\mathcal{H}}(R|\Theta)$, normalized variance of return in standard RL benchmark environments, Pearson correlation coefficient to the algorithm-based normalized score~(Score(A)), and the random-sampling-based normalized score~(Score(R)). We prepare bags of policy architectures, totally 56 variants of architectures; ([0] layers $+$ [1, 2] layers $\times$ [4, 32, 64] units) $\times$ [Gaussian prior $\mathcal{N}(0, I)$, Uniform prior $\textit{Unif}(-1,1)$, Xavier Normal, Xavier Uniform] $\times$ [w/ bias, w/o bias]. The results suggest that POIC seems to positively correlate well with algorithm-based normalized score better than any other alternatives, such as marginal reward entropy or variance of returns~\citep{oller2020analyzing}.}
\label{table:empowerment_full}
\end{table}
\end{landscape}

\section{Details of Reward Shaping Experiments}
\label{sec:effect_on_reward_shaping}
Here, we present hyper-parameters of the reward functions and raw experimental results (\autoref{table:reward_hyper_reacher} for Reacher and \autoref{table:reward_hyper_pointmaze} for Pointmaze~(\autoref{fig:pointmaze})).
In the experiment, we employ the following four families of goal-oriented reward function:
\begin{enumerate}
\setlength{\parskip}{0cm}
\setlength{\leftskip}{1.0cm}
\item L1 norm: $r(s, s_g) = -\alpha\|s-s_g\|_{1}$,
\item L2 norm: $r(s, s_g) =-\alpha\|s-s_g\|_{2}$,
\item Fraction: $r(s, s_g) =\frac{\beta}{\gamma + \|s-s_g\|_{2}}$,
\item Sparse: $r(s, s_g) = - \mathbbm{1}[\|s-s_g\|_{2} \geq \epsilon]$.
\end{enumerate}
This notations of hyper-parameter correspond to \autoref{table:reward_hyper_reacher} and \autoref{table:reward_hyper_pointmaze}.
To estimate information capacity metrics, we use the small neural networks (2 layers, 4 hidden units, Gaussian initialization $\mathcal{N}(0, I)$, and without bias term) for the simplicity.
For the following RL training, we employ the same policy architectures, while keeping original value networks sizes (2 layers, 64 hidden units).
We normalize the scores of PPO, trained 500 steps and averaged among 5 seeds for fair comparisons among different reward-scale environments.

\begin{table}[htb]
\begin{center}
\begin{small}
\scalebox{0.72}{
\begin{tabular}{l|c|ccc|cc}
\toprule
Reward & Hyper-parameter & $r_{\max}^{\text{rand}}$ & PPO (500k) & Normalized Score & $\hat{\mathcal{I}}(R;\Theta)$ & $\hat{\mathcal{I}}(\mathcal{O};\Theta)$ \\
\midrule
L1-norm & $\alpha=1.0$ & -0.557 & -8.776 & 0.507 & 0.796 & 5.801e-3 \\
L1-norm & $\alpha=0.5$ & -0.316 & -4.245 & 0.524 & 0.798 & 5.785e-3 \\
L1-norm & $\alpha=2.0$ & -1.335 & -17.009 & 0.512 & 0.791 & 5.943e-3 \\
L1-norm & $\alpha=5.0$ & -2.940 & -42.702 & 0.524 & 0.794 & 5.965e-3 \\
\midrule
L2-norm & $\alpha=1.0$ & -0.417 & -6.604 & 0.530 & 0.837 & 5.743e-3 \\
L2-norm & $\alpha=0.5$ & -0.241 & -3.238 & 0.558 & 0.841 & 5.841e-3 \\
L2-norm & $\alpha=2.0$ & -0.969 & -13.339 & 0.522 & 0.837 & 5.815e-3 \\
L2-norm & $\alpha=5.0$ & -2.524 & -34.760 & 0.502 & 0.837 & 5.778e-3 \\
\midrule
Fraction & $(\beta, \gamma)=(0.01, 0.01)$ & 31.182 & 5.874 & 0.133 & 0.522 & 9.279e-5 \\
Fraction & $(\beta, \gamma)=(0.1, 0.1)$ & 45.758 & 23.787 & 0.306 & 0.772 & 1.270e-3 \\
Fraction & $(\beta, \gamma)=(0.01, 0.1)$ & 4.546 & 2.326 & 0.281 & 0.779 & 1.318e-3 \\
Fraction & $(\beta, \gamma)=(0.05, 0.1)$ & 22.839 & 11.673 & 0.278 & 0.772 & 1.284e-3 \\
\midrule
Sparse & $\epsilon=0.05$ & 0.000 & -44.220 & 0.111 & 0.297 & 1.348e-3 \\
Sparse & $\epsilon=0.01$ & 0.000 & -49.660 & 0.009 & 0.024 & 9.171e-6 \\
Sparse & $\epsilon=0.1$ & -10.000 & -30.520 & 0.376 & 0.541 & 0.0103 \\
Sparse & $\epsilon=0.15$ & 0.000 & -18.460 & 0.586 & 0.561 & 0.0329 \\
\bottomrule
\end{tabular}
}
\end{small}
\end{center}
\vskip -0.15in
\caption{Hyper-parameter and results of reward shaping experiments in Reacher, appeared in Section~\ref{sec:exp_reward_shaping}.}
\label{table:reward_hyper_reacher}
\end{table}

\begin{table}[htb]
\begin{center}
\begin{small}
\scalebox{0.72}{
\begin{tabular}{l|c|ccc|cc}
\toprule
Reward & Hyper-parameter & $r_{\max}^{\text{rand}}$ & PPO (500k) & Normalized Score & $\hat{\mathcal{I}}(R;\Theta)$ & $\hat{\mathcal{I}}(\mathcal{O};\Theta)$ \\
\midrule
L1-norm & $\alpha=1.0$ & -4.095 & -177.385 & 0.706 & 1.740 & 0.0232 \\
L1-norm & $\alpha=0.5$ & -1.556 & -100.371 & 0.664 & 1.806 & 0.0226  \\
L1-norm & $\alpha=2.0$ & -7.418 & -388.604 & 0.677 & 1.800 & 0.0226 \\
L1-norm & $\alpha=5.0$ & -13.431 & -1009.456 & 0.661 & 1.802 & 0.0227 \\
\midrule
L2-norm & $\alpha=1.0$ & -3.314 & -154.742 & 0.641 & 1.845 & 0.0225 \\
L2-norm & $\alpha=0.5$ & -1.983 & -81.826 & 0.621 & 1.842 & 0.0225\\
L2-norm & $\alpha=2.0$ & -5.971 & -293.556 & 0.660 & 1.836 & 0.0223 \\
L2-norm & $\alpha=5.0$ & -11.693 & -811.083 & 0.622 & 1.839 & 0.0216 \\
\midrule
Fraction & $(\beta, \gamma)=(0.01, 0.01)$ & 59.037 & 6.253 & 0.095 & 1.051 & 1.962e-5 \\
Fraction & $(\beta, \gamma)=(0.1, 0.1)$ & 123.128 & 37.607 & 0.270 & 1.443 & 2.888e-4 \\
Fraction & $(\beta, \gamma)=(0.01, 0.1)$ & 13.017 & 3.546 & 0.238 & 1.443 & 2.570e-4 \\
Fraction & $(\beta, \gamma)=(0.05, 0.1)$ & 67.266 & 18.154 & 0.236 & 1.424 & 2.357e-4 \\
\midrule
Sparse & $\epsilon=0.5$ & 0.000 & -110.000 & 0.258 & 0.258 & 0.0186\\
Sparse & $\epsilon=0.1$ & 0.000 & -146.740 & 0.019 & 0.019 & 4.339e-4 \\
Sparse & $\epsilon=0.2$ & 0.000 & -137.800 & 0.076 & 0.076 & 1.574e-3\\
Sparse & $\epsilon=1.0$ & 0.000 & -40.240 & 0.726 & 0.726 & 0.0768 \\
\bottomrule
\end{tabular}
}
\end{small}
\end{center}
\vskip -0.15in
\caption{Hyper-parameter and results of reward shaping experiments in Pointmaze, appeared in Section~\ref{sec:exp_reward_shaping}.}
\label{table:reward_hyper_pointmaze}
\end{table}

\vspace{-\baselineskip}
\begin{figure}[ht]
\centering
\includegraphics[width=0.18\linewidth]{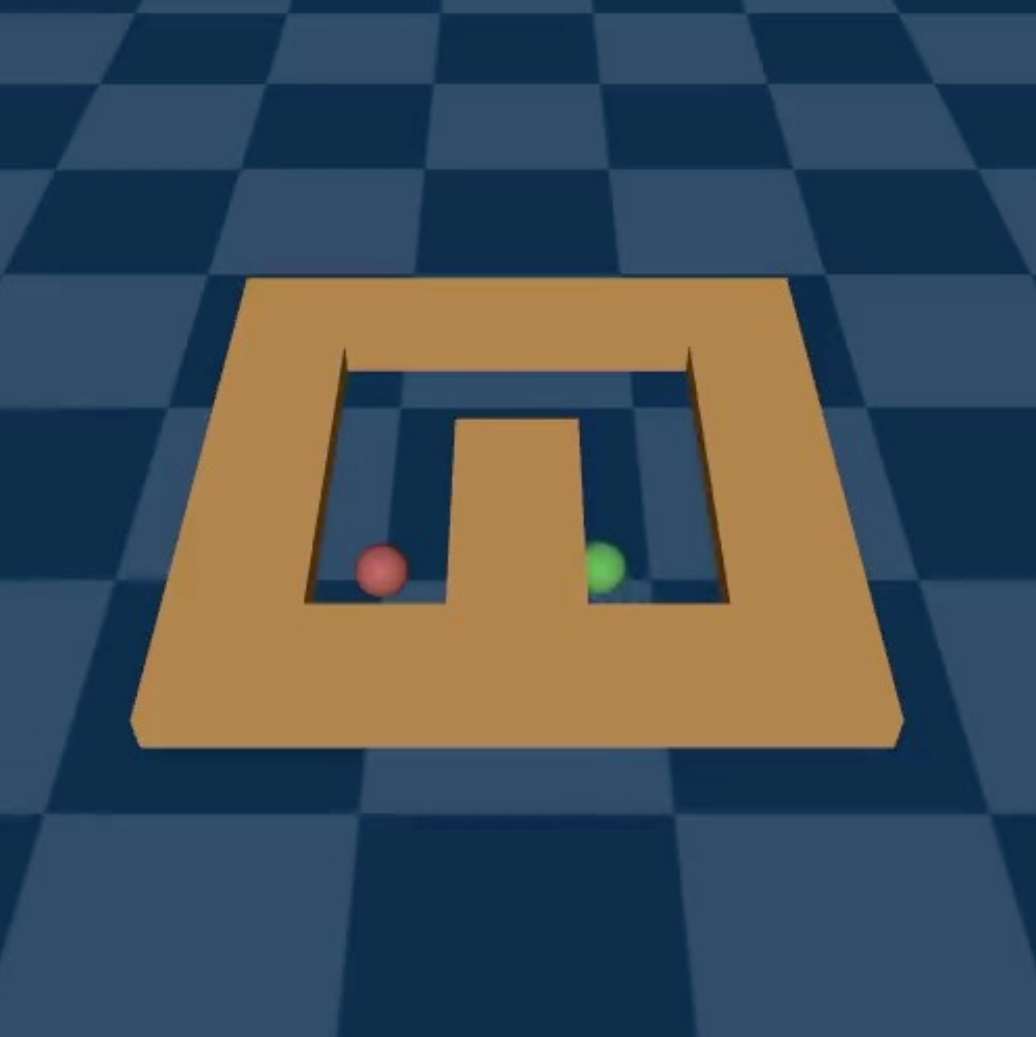}  
\vskip -0.05in
\caption{Pointmaze environment~\citep{fu2020d4rl}.}
\label{fig:pointmaze}
\end{figure}

\clearpage
\section{Evaluating the Goodness of Network Architecture and Initialization}
\label{sec:mi_with_different_archi}
Can we also use PIC and POIC to evaluate the goodness of network architecture or initialization? We investigate the correlation between PIC or POIC and the normalized score of the PPO policy with different network configurations. For the comparison, we prepare 7 policy network architectures without bias term (0 layers $+$ [1, 2] layers $\times$ [4, 32, 64] hidden units; while keeping original value networks sizes, 2 layers 64 units), and 4 initializations (normal $\mathcal{N}(0, I)$, uniform $\textit{Unif}(-1,1)$, Xavier normal and Xavier uniform). Xavier Normal and Uniform are the typical initialization methods of neural networks~\citep{glorot2010}.
First, we measure both PIC and POIC for each policy, and then train it with PPO, during 500k steps, except for 50k steps in CartPole~(see \autoref{sec:mi_with_different_archi_details} for the detailed results).

The results are shown in~\autoref{table:architecture_initialization_optimality}, and \autoref{table:architecture_initialization_reward}.
We can see valid positive correlations in CartPole, Pendulum, HalfCheetah, Hopper, and Walker2d with specific initialization, which are statistically significant with $p<0.05$, while we also observe the weak positive, negative, or no trends in other environments.
In several domains, PIC and POIC might be used for architectural tuning without extensive RL trainings.
In addition, some negative results seem consistent with empirical observations in recent RL research; the performance of many deep RL algorithms require architectural tuning for best performances~\citep{Schulman:2015uk, Schulman2017PPO, engstrom2019implementation, andrychowicz2021what}, and can be sensitive to architecture and initialization~\citep{Rajeswaran2017}.

\begin{table*}[htb]
\begin{center}
\begin{scriptsize}
\begin{tabular}{c|cccccccc}
\toprule
PIC & CartPole & Pendulum & MountainCar & MountainCarContinuous & Acrobot & HalfCheetah & Hopper & Walker2d \\
\midrule
\begin{tabular}{c} Normal \\ (p-Value) \end{tabular} & \begin{tabular}{c} \textbf{0.938} \\ (0.002) \end{tabular} &  \begin{tabular}{c} 0.121 \\ (0.796) \end{tabular} & \begin{tabular}{c} 0.681 \\ (0.092) \end{tabular} & \begin{tabular}{c} 0.124 \\ (0.791) \end{tabular} & \begin{tabular}{c} -0.455 \\ (0.304) \end{tabular} & \begin{tabular}{c} \textbf{0.891} \\ (0.007) \end{tabular} & \begin{tabular}{c} \textbf{0.839} \\ (0.018) \end{tabular} & \begin{tabular}{c} \textbf{0.782} \\ (0.038) \end{tabular} \\
\midrule
\begin{tabular}{c} Uniform \\ (p-Value) \end{tabular} & \begin{tabular}{c} 0.658 \\ (0.108) \end{tabular} &  \begin{tabular}{c} -0.172 \\ (0.712) \end{tabular} & \begin{tabular}{c} -0.783 \\ (0.037) \end{tabular} & \begin{tabular}{c} -0.021 \\ (0.965) \end{tabular} & \begin{tabular}{c} 0.502 \\ (0.251) \end{tabular} & \begin{tabular}{c} \textbf{0.881} \\ (0.009) \end{tabular} & \begin{tabular}{c} 0.649 \\ (0.115) \end{tabular} & \begin{tabular}{c} 0.617 \\ (0.140) \end{tabular} \\
\midrule
\begin{tabular}{c} Xavier (N) \\ (p-Value) \end{tabular} & \begin{tabular}{c} 0.483 \\ (0.272) \end{tabular} &  \begin{tabular}{c} 0.438 \\ (0.325) \end{tabular} & \begin{tabular}{c} -0.544 \\ (0.207) \end{tabular} & \begin{tabular}{c} -0.079 \\ (0.866) \end{tabular} & \begin{tabular}{c} -0.474 \\ (0.283) \end{tabular} & \begin{tabular}{c} 0.186 \\ (0.689) \end{tabular} & \begin{tabular}{c} -0.660 \\ (0.106) \end{tabular} & \begin{tabular}{c} -0.789 \\ (0.035) \end{tabular} \\
\midrule
\begin{tabular}{c} Xavier (U) \\ (p-Value) \end{tabular} & \begin{tabular}{c} 0.171 \\ (0.713) \end{tabular} &  \begin{tabular}{c}  0.406 \\ (0.366) \end{tabular} & \begin{tabular}{c} -0.244 \\ (0.599) \end{tabular} & \begin{tabular}{c} 0.667 \\ (0.102) \end{tabular} & \begin{tabular}{c} -0.259 \\ (0.574) \end{tabular} & \begin{tabular}{c} -0.003 \\ (0.995) \end{tabular} & \begin{tabular}{c} -0.709 \\ (0.075) \end{tabular} & \begin{tabular}{c} -0.965 \\ ($>$ 0.001) \end{tabular} \\
\bottomrule
\end{tabular}
\end{scriptsize}
\end{center}
\vskip -0.15in
\caption{Pearson correlation coefficient between the normalized score of different policies (network architecture and initialization) and PIC. 
We can see valid positive correlations in CartPole, HalfCheetah, Hopper, and Walker2d with specific initialization, which are statistically significant with $p<0.05$. The 2D-plots can be seen in \autoref{fig:architecture_initialization_reward}.
}
\label{table:architecture_initialization_reward}
\end{table*}

\begin{table*}[htb]
\begin{center}
\begin{scriptsize}
\begin{tabular}{c|cccccccc}
\toprule
POIC & CartPole & Pendulum & MountainCar & MountainCarContinuous & Acrobot & HalfCheetah & Hopper & Walker2d \\
\midrule
\begin{tabular}{c} Normal \\ (p-Value) \end{tabular} & \begin{tabular}{c} 0.392 \\ (0.385) \end{tabular} &  \begin{tabular}{c} \textbf{0.769} \\ (0.043) \end{tabular} & \begin{tabular}{c} -0.229 \\ (0.621) \end{tabular} & \begin{tabular}{c} 0.004 \\ (0.994) \end{tabular} & \begin{tabular}{c} 0.177 \\ (0.704) \end{tabular} & \begin{tabular}{c} \textbf{0.855} \\ (0.014) \end{tabular} & \begin{tabular}{c} 0.476 \\ (0.280) \end{tabular} & \begin{tabular}{c} 0.624 \\ (0.134) \end{tabular} \\
\midrule
\begin{tabular}{c} Uniform \\ (p-Value) \end{tabular} & \begin{tabular}{c} -0.087 \\ (0.852) \end{tabular} &  \begin{tabular}{c} 0.451 \\ (0.309)) \end{tabular} & \begin{tabular}{c} -0.842 \\ (0.018) \end{tabular} & \begin{tabular}{c} -0.059 \\ (0.900) \end{tabular} & \begin{tabular}{c} -0.288 \\ (0.532) \end{tabular} & \begin{tabular}{c} \textbf{0.864} \\ (0.012) \end{tabular} & \begin{tabular}{c} 0.188 \\ (0.686) \end{tabular} & \begin{tabular}{c} 0.523 \\ (0.229) \end{tabular} \\
\midrule
\begin{tabular}{c} Xavier (N) \\ (p-Value) \end{tabular} & \begin{tabular}{c} 0.127 \\ (0.787) \end{tabular} &  \begin{tabular}{c} 0.579 \\ (0.173) \end{tabular} & \begin{tabular}{c} -0.503 \\ (0.250) \end{tabular} & \begin{tabular}{c} 0.373 \\ (0.409) \end{tabular} & \begin{tabular}{c} -0.720 \\ (0.068) \end{tabular} & \begin{tabular}{c} 0.344 \\ (0.450) \end{tabular} & \begin{tabular}{c} \textbf{0.957} \\ ($>$ 0.001) \end{tabular} & \begin{tabular}{c} -0.946 \\ (0.001) \end{tabular} \\
\midrule
\begin{tabular}{c} Xavier (U) \\ (p-Value) \end{tabular} & \begin{tabular}{c} \textbf{0.930} \\ (0.002) \end{tabular} &  \begin{tabular}{c} 0.495 \\ (0.258) \end{tabular} & \begin{tabular}{c} -0.406 \\ (0.367) \end{tabular} & \begin{tabular}{c} -0.649 \\ (0.114) \end{tabular} & \begin{tabular}{c} -0.381 \\ (0.399) \end{tabular} & \begin{tabular}{c} 0.080 \\ (0.865) \end{tabular} & \begin{tabular}{c} 0.006 \\ (0.990) \end{tabular} & \begin{tabular}{c} -0.726 \\ (0.065) \end{tabular} \\
\bottomrule
\end{tabular}
\end{scriptsize}
\end{center}
\vskip -0.15in
\caption{Pearson correlation coefficient between the normalized score of different policies (network architecture and initialization) and POIC. 
We can see valid positive correlations in CartPole, Pendulum, HalfCheetah and Hopper with specific initialization, which are statistically significant with $p<0.05$. The 2D-plots can be seen in \autoref{fig:architecture_initialization_optimality}.
}
\label{table:architecture_initialization_optimality}
\end{table*}

\begin{figure*}[ht]
\centering
\includegraphics[width=\linewidth]{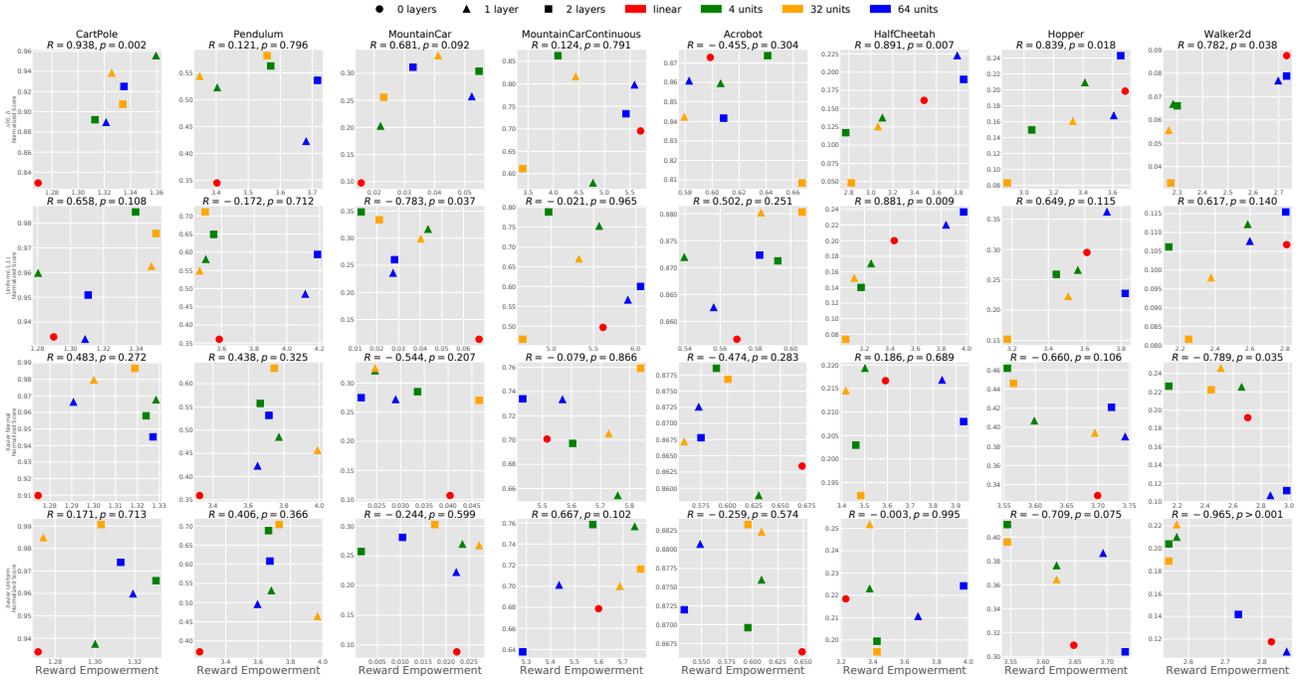}  
\vskip -0.05in
\caption{2D-Scatter plots between PIC~(x-axis) and the normalized score~(y-axis), trained during 500k steps and averaged over 5 seeds. We test combinations of 7 network architectures (0 layers, 1 layer 4 units, 2 layers 4units) and 4 kinds of network initialization (normal, uniform, Xavier normal, Xavier uniform).}
\label{fig:architecture_initialization_reward}
\end{figure*}

\begin{figure*}[ht]
\centering
\includegraphics[width=\linewidth]{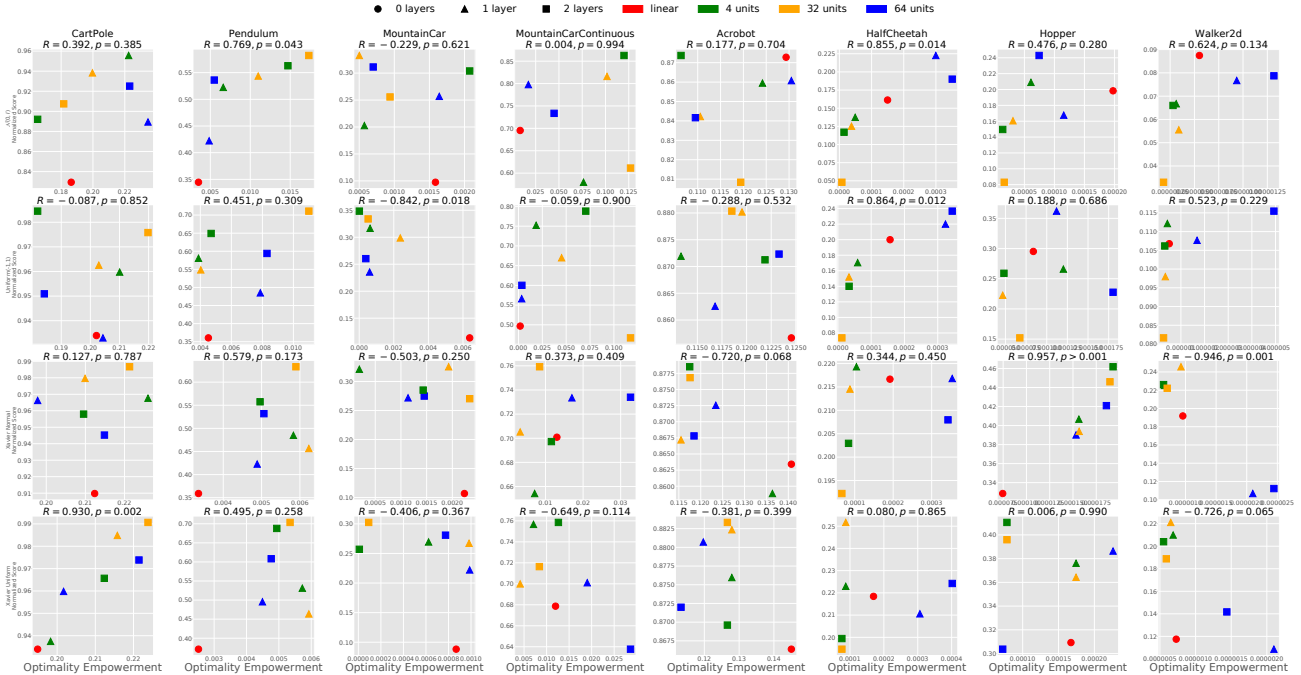}  
\vskip -0.05in
\caption{2D-Scatter plots between POIC~(x-axis) and the normalized score~(y-axis), trained during 500k steps and averaged over 5 seeds. We test combinations of 7 network architectures (0 layers, 1 layer 4 units, 2 layers 4units) and 4 kinds of network initialization (normal, uniform, Xavier normal, Xavier uniform).}
\label{fig:architecture_initialization_optimality}
\end{figure*}

\clearpage

\section{Details of Architecture and Initialization Experiments}
\label{sec:mi_with_different_archi_details}
In this section, we presents the raw scores observed through the experiments in~\autoref{sec:mi_with_different_archi}.
We train PPO with (64, 64) neural networks and 3 different discount factor $\gamma \in \{0.95, 0.99, 0.995\}$. For normalized score, we refer the algorithm-max scores in \autoref{sec:appendix_normalized_score}.
All results are summarized in \autoref{table:raw_reward_arch_init} (average cumulative rewards), \autoref{table:ns_arch_init} (normalized score), \autoref{table:re_arch_init} (PIC), and \autoref{table:oe_arch_init} (POIC).

\begin{table}[htb]
\begin{center}
\begin{small}
\scalebox{0.70}{
\begin{tabular}{l|cccccccc}
\toprule
 & CartPole & Pendulum & MountainCar & MountainCarContinuous & Acrobot & HalfCheetah & Hopper & Walker2d \\
\midrule

0L Normal & 167.38 & -1268.00 & -189.97 & 37.57& -120.18 & 154.04 & 488.97 & 705.56 \\
1L 4U Normal & 178.85 & -1134.51 & -173.60 & 56.37 & -125.43 & 1221.59 & 427.17 & 596.35 \\
1L 32U Normal & 191.51 & -962.21 & -179.17 & 13.59 & -125.97 & -221.18 & 367.72 & 743.74 \\
1L 64U Normal & 188.20 & -926.62 & -165.81 & 59.85 & -133.48 & -344.62 & 303.13 & 570.78 \\
2L 4U Normal & 185.68 & -931.49 & -168.00 & 43.72 & -133.74 & 618.83 & 435.76 & 864.46 \\
2L 32U Normal & 179.36 & -883.60 & -168.76 & 69.01 & -119.79 & -369.87 & 364.67 & 531.89 \\
2L 64U Normal & 182.30 & -852.51 & -173.71 & 19.79 & -148.21 & -404.71 & 173.00 & 294.66 \\
\midrule
0L Uniform & 187.37 & -1243.73	& -188.33 & 31.21 & -127.21	& 780.81 & 600.25 & 1051.04 \\
1L 4U Uniform & 187.19 & -1017.40 & -175.79 & 18.77 & -124.63 & 1057.68 & 591.57 & 1289.52 \\
1L 32U Uniform & 192.31 & -857.24 & -167.46 & 47.44 & -120.57 & 652.51 & 632.90 & 946.86 \\
1L 64U Uniform & 192.84 & -908.35 & -169.29 & 31.27 & -116.97 & 316.07 & 548.21 & 792.32 \\
2L 4U Uniform & 190.61 & -829.38 & -173.24 & 24.82 & -120.39 & 1361.30 & 637.16 & 810.50 \\
2L 32U Uniform & 197.04 & -737.45 & -164.19 & 54.40 & -120.86 & 364.60 & 598.17 & 922.05 \\
2L 64U Uniform & 195.37 & -629.24 & -165.67 & -8.47 & -116.90& -186.69 & 456.83 & 540.52 \\
\midrule
0L Xavier (N) & 182.75 & -1249.18 & -189.02 & 37.42 & -124.27 & 1111.14 & 1096.75 & 1170.13 \\
1L 4U Xavier (N) & 193.55 & -1132.60 & -172.06 & 43.82 & -120.29 & 1219.57 & 599.91 & 1389.66 \\
1L 32U Xavier (N) & 193.81 & -1022.30 & -166.99 & 31.09 & -126.25 & 1644.45 & 1292.61 & 1448.79 \\
1L 64U Xavier (N) & 196.11 & -1069.38 & -166.55 & 43.21 & -122.63 & 1593.45 & 1410.29 & 1403.11 \\
2L 4U Xavier (N) & 189.53 & -937.73 & -171.75 & 43.82 & -122.36 & 1048.78 & 634.12 & 1499.48 \\
2L 32U Xavier (N) & 191.96 & -893.65 & -170.69 & 36.78 & -117.69 & 1459.96 & 1297.16 & 1645.11 \\
2L 64U Xavier (N) & 197.45 & -764.92 & -172.21 & 49.70 & -118.43 & 1507.47 & 1275.05 & 1588.82 \\
\midrule
0L Xavier (U) & 187.37 & -1223.45 & -190.93 & 37.53 & -122.97 & 1117.23 & 663.19 & 1101.14 \\
1L 4U Xavier (U) & 192.31 & -1005.49 & -177.17 & 37.54 & -116.74 & 1218.68 & 611.44 & 1376.33 \\
1L 32U Xavier (U) & 188.04 & -945.72 & -172.31 & 49.55 & -118.81 & 1467.81 & 1202.81 & 1339.44 \\
1L 64U Xavier (U) & 197.10 & -1055.70 & -172.53 & 42.33 & -116.03 & 1977.95 & 1268.04 & 1297.33 \\
2L 4U Xavier (U) & 195.00 & -806.39 & -171.13 & 25.00 & -120.53 & 1370.19 & 804.70 & 1081.45 \\
2L 32U Xavier (U) & 193.43 & -670.05 & -173.58 & 49.16 & -121.57 & 1758.48 & 1168.65 & 1462.36 \\
2L 64U Xavier (U) & 198.21 & -640.18	& -168.91 & 42.80 & -115.60 & 1677.97 & 1080.10	& 1410.35 \\
\bottomrule
\end{tabular}
}
\end{small}
\end{center}
\vskip -0.15in
\caption{Cumulative rewards averaged over 3 different discount factor $\gamma \in \{0.95, 0.99, 0.995\}$ and 5 random seeds.}
\label{table:raw_reward_arch_init}
\end{table}

\begin{table}[htb]
\begin{center}
\begin{small}
\scalebox{0.70}{
\begin{tabular}{l|cccccccc}
\toprule
 & CartPole & Pendulum & MountainCar & MountainCarContinuous & Acrobot & HalfCheetah & Hopper & Walker2d \\
\midrule
0L Normal & 0.829 & 0.345 & 0.098 & 0.695 & 0.873 & 0.161 & 0.087 & 0.198 \\
1L 4U Normal & 0.889 & 0.422 & 0.257 & 0.798 & 0.861 & 0.223 & 0.077 & 0.168 \\
1L 32U Normal & 0.956 & 0.522 & 0.203 & 0.580 & 0.859 & 0.137 & 0.067 & 0.209 \\
1L 64U Normal & 0.938 & 0.544 & 0.333 & 0.816 & 0.842 & 0.125 & 0.056 & 0.161 \\
2L 4U Normal & 0.925 & 0.536 & 0.311 & 0.734 & 0.842 & 0.190 & 0.079 & 0.243 \\
2L 32U Normal & 0.892 & 0.563 & 0.304 & 0.863 & 0.874 & 0.117 & 0.066 & 0.150 \\
2L 64U Normal & 0.907 & 0.582 & 0.256 & 0.611 & 0.808 & 0.048 & 0.033 & 0.083 \\
\midrule
0L Uniform & 0.934 & 0.361 & 0.114 & 0.497 & 0.857 & 0.200 & 0.107 & 0.295 \\
1L 4U Uniform & 0.933 & 0.485 & 0.236 & 0.566 & 0.863 & 0.220 & 0.108 & 0.362 \\
1L 32U Uniform & 0.960 & 0.581 & 0.317 & 0.752 & 0.872 & 0.170 & 0.112 & 0.266 \\
1L 64U Uniform & 0.963 & 0.549 & 0.299 & 0.670 & 0.880 & 0.152 & 0.098 & 0.223 \\
2L 4U Uniform & 0.951 & 0.594 & 0.260 & 0.600 & 0.872 & 0.237 & 0.115 & 0.228 \\
2L 32U Uniform & 0.985 & 0.649 & 0.348 & 0.788 & 0.871 & 0.140 & 0.106 & 0.259 \\
2L 64U Uniform & 0.976 & 0.711 & 0.334 & 0.467 & 0.880 & 0.074 & 0.082 & 0.152 \\
\midrule
0L Xavier (N) & 0.910 & 0.359 & 0.107 & 0.701 & 0.863 & 0.217 & 0.192 & 0.329 \\
1L 4U Xavier (N) & 0.966 & 0.423 & 0.272 & 0.733 & 0.873 & 0.217 & 0.107 & 0.390 \\
1L 32U Xavier (N) & 0.968 & 0.485 & 0.321 & 0.654 & 0.859 & 0.219 & 0.225 & 0.407 \\
1L 64U Xavier (N) & 0.980 & 0.456 & 0.325 & 0.705 & 0.867 & 0.215 & 0.246 & 0.394 \\
2L 4U Xavier (N) & 0.945 & 0.532 & 0.275 & 0.734 & 0.868 & 0.208 & 0.112 & 0.421 \\
2L 32U Xavier (N) & 0.958 & 0.557 & 0.285 & 0.697 & 0.879 & 0.203 & 0.226 & 0.462 \\
2L 64U Xavier (N) & 0.987 & 0.633 & 0.270 & 0.759 & 0.877 & 0.192 & 0.222 & 0.446 \\
\midrule
0L Xavier (U) & 0.934 & 0.371 & 0.088 & 0.679 & 0.866 & 0.218 & 0.118 & 0.309 \\
1L 4U Xavier (U) & 0.960 & 0.495 & 0.222 & 0.701 & 0.881 & 0.211 & 0.109 & 0.386 \\
1L 32U Xavier (U) & 0.937 & 0.531 & 0.269 & 0.757 & 0.876 & 0.223 & 0.210 & 0.376 \\
1L 64U Xavier (U) & 0.985 & 0.464 & 0.267 & 0.700 & 0.882 & 0.252 & 0.221 & 0.364 \\
2L 4U Xavier (U) & 0.974 & 0.608 & 0.281 & 0.638 & 0.872 & 0.224 & 0.142 & 0.304 \\
2L 32U Xavier (U) & 0.966 & 0.688 & 0.257 & 0.759 & 0.870 & 0.199 & 0.204 & 0.411 \\
2L 64U Xavier (U) & 0.991 & 0.704 & 0.302 & 0.716 & 0.883 & 0.195 & 0.189 & 0.396 \\
\bottomrule
\end{tabular}
}
\end{small}
\end{center}
\vskip -0.15in
\caption{Algorithm-based normalized scores. We use the minimum values during random-sampling and the maximum values reported in \autoref{table:algorithm_results_classic_control} and \autoref{table:appendix_mujoco_algo_ave_max}.}
\label{table:ns_arch_init}
\end{table}

\begin{table}[htb]
\begin{center}
\begin{small}
\scalebox{0.70}{
\begin{tabular}{l|cccccccc}
\toprule
 & CartPole & Pendulum & MountainCar & MountainCarContinuous & Acrobot & HalfCheetah & Hopper & Walker2d \\
\midrule
0L Normal & 1.269 & 3.402 & 0.016 & 5.707 & 0.599 & 3.486 & 2.732 & 3.685 \\
1L 4U Normal & 1.321 & 3.681 & 0.052 & 5.588 & 0.583 & 3.789 & 2.698 & 3.608 \\
1L 32U Normal & 1.359 & 3.403 & 0.022 & 4.774 & 0.606 & 3.105 & 2.283 & 3.412 \\
1L 64U Normal & 1.326 & 3.348 & 0.041 & 4.433 & 0.579 & 3.065 & 2.267 & 3.330 \\
2L 4U Normal & 1.335 & 3.716 & 0.033 & 5.421 & 0.608 & 3.847 & 2.731 & 3.655 \\
2L 32U Normal & 1.313 & 3.570 & 0.055 & 4.086 & 0.641 & 2.771 & 2.300 & 3.052 \\
2L 64U Normal & 1.334 & 3.559 & 0.023 & 3.395 & 0.667 & 2.822 & 2.274 & 2.886 \\
\midrule
0L Uniform & 1.290 & 3.586 & 0.067 & 5.612 & 0.569 & 3.432 & 2.810 & 3.612 \\
1L 4U Uniform & 1.309 & 4.112 & 0.027 & 5.907 & 0.556 & 3.838 & 2.599 & 3.722 \\
1L 32U Uniform & 1.280 & 3.503 & 0.044 & 5.568 & 0.540 & 3.251 & 2.587 & 3.562 \\
1L 64U Uniform & 1.349 & 3.466 & 0.040 & 5.327 & 0.583 & 3.119 & 2.376 & 3.509 \\
2L 4U Uniform & 1.311 & 4.187 & 0.028 & 6.057 & 0.582 & 3.976 & 2.804 & 3.822 \\
2L 32U Uniform & 1.339 & 3.552 & 0.013 & 4.968 & 0.593 & 3.172 & 2.134 & 3.443 \\
2L 64U Uniform & 1.352 & 3.500 & 0.021 & 4.661 & 0.606 & 3.052 & 2.248 & 3.174 \\
\midrule
0L Xavier (N) & 1.275 & 3.326 & 0.040 & 5.516 & 0.670 & 3.592 & 2.705 & 3.699 \\
1L 4U Xavier (N) & 1.291 & 3.652 & 0.029 & 5.569 & 0.573 & 3.840 & 2.867 & 3.743 \\
1L 32U Xavier (N) & 1.329 & 3.772 & 0.025 & 5.759 & 0.630 & 3.502 & 2.660 & 3.598 \\
1L 64U Xavier (N) & 1.300 & 3.990 & 0.025 & 5.728 & 0.559 & 3.418 & 2.511 & 3.695 \\
2L 4U Xavier (N) & 1.327 & 3.716 & 0.022 & 5.432 & 0.575 & 3.934 & 2.984 & 3.721 \\
2L 32U Xavier (N) & 1.324 & 3.667 & 0.034 & 5.604 & 0.590 & 3.462 & 2.139 & 3.555 \\
2L 64U Xavier (N) & 1.319 & 3.746 & 0.047 & 5.838 & 0.600 & 3.483 & 2.443 & 3.565 \\
\midrule
0L Xavier (U) & 1.271 & 3.240 & 0.022 & 5.599 & 0.649 & 3.229 & 2.825 & 3.648 \\
1L 4U Xavier (U) & 1.319 & 3.597 & 0.022 & 5.436 & 0.549 & 3.683 & 2.867 & 3.694 \\
1L 32U Xavier (U) & 1.300 & 3.682 & 0.024 & 5.749 & 0.609 & 3.379 & 2.568 & 3.621 \\
1L 64U Xavier (U) & 1.274 & 3.967 & 0.027 & 5.687 & 0.609 & 3.379 & 2.568 & 3.621 \\
2L 4U Xavier (U) & 1.313 & 3.673 & 0.010 & 5.285 & 0.533 & 3.970 & 2.736 & 3.728 \\
2L 32U Xavier (U) & 1.331 & 3.665 & 0.001 & 5.575 & 0.596 & 3.425 & 2.547 & 3.545 \\
2L 64U Xavier (U) & 1.303 & 3.730 & 0.017 & 5.773 & 0.596 & 3.425 & 2.547 & 3.545 \\
\bottomrule
\end{tabular}
}
\end{small}
\end{center}
\vskip -0.15in
\caption{Estimated Policy Information Capacity.}
\label{table:re_arch_init}
\end{table}

\begin{table}[htb]
\begin{center}
\begin{small}
\scalebox{0.70}{
\begin{tabular}{l|cccccccc}
\toprule
 & CartPole & Pendulum & MountainCar & MountainCarContinuous & Acrobot & HalfCheetah & Hopper & Walker2d \\
\midrule
0L Normal & 0.186 & 0.003516 & 0.001586 & 0.008363 & 0.129 & 0.000150 & 0.000000551 & 0.000198 \\
1L 4U Normal & 0.235 & 0.004850 & 0.001641 & 0.017159 & 0.131 & 0.000300 & 0.000000935 & 0.000116 \\
1L 32U Normal & 0.222 & 0.006613 & 0.000573 & 0.076397 & 0.124 & 0.000049 & 0.000000312 & 0.000061 \\
1L 64U Normal & 0.200 & 0.010994 & 0.000502 & 0.101649 & 0.111 & 0.000037 & 0.000000339 & 0.000031 \\
2L 4U Normal & 0.223 & 0.005492 & 0.000699 & 0.044784 & 0.109 & 0.000351 & 0.000001318 & 0.000075 \\
2L 32U Normal & 0.165 & 0.014715 & 0.002074 & 0.119891 & 0.106 & 0.000014 & 0.000000280 & 0.000014 \\
2L 64U Normal & 0.182 & 0.017335 & 0.000939 & 0.126900 & 0.119 & 0.000008 & 0.000000182 & 0.000016 \\
\midrule
0L Uniform & 0.202 & 0.004520 & 0.006387 & 0.001799 & 0.125 & 0.000157 & 0.000000507 & 0.000084 \\
1L 4U Uniform & 0.204 & 0.007866 & 0.000603 & 0.003397 & 0.117 & 0.000323 & 0.000001682 & 0.000112 \\
1L 32U Uniform & 0.210 & 0.003882 & 0.000628 & 0.018515 & 0.113 & 0.000059 & 0.000000427 & 0.000120 \\
1L 64U Uniform & 0.203 & 0.004031 & 0.002376 & 0.045212 & 0.120 & 0.000033 & 0.000000341 & 0.000047 \\
2L 4U Uniform & 0.184 & 0.008297 & 0.000383 & 0.003576 & 0.123 & 0.000343 & 0.000004932 & 0.000180 \\
2L 32U Uniform & 0.182 & 0.004696 & 0.000007 & 0.070520 & 0.122 & 0.000034 & 0.000000319 & 0.000048 \\
2L 64U Uniform & 0.220 & 0.010994 & 0.000516 & 0.117416 & 0.119 & 0.000012 & 0.000000266 & 0.000067 \\
\midrule
0L Xavier (N) & 0.212 & 0.003354 & 0.002235 & 0.013028 & 0.141 & 0.000192 & 0.000000957 & 0.000075 \\
1L4U Xavier (N) & 0.198 & 0.004885 & 0.001138 & 0.017027 & 0.123 & 0.000353 & 0.000002099 & 0.000153 \\
1L32U Xavier (N) & 0.226 & 0.005830 & 0.000191 & 0.007107 & 0.136 & 0.000104 & 0.000000659 & 0.000156 \\
1L 64U Xavier (N) & 0.210 & 0.006233 & 0.001928 & 0.003256 & 0.115 & 0.000088 & 0.000000927 & 0.000157 \\
2L 4U Xavier (N) & 0.215 & 0.005060 & 0.001453 & 0.032625 & 0.118 & 0.000342 & 0.000002443 & 0.000186 \\
2L 32U Xavier (N) & 0.210 & 0.004961 & 0.001432 & 0.011564 & 0.117 & 0.000084 & 0.000000642 & 0.000193 \\
2L 64U Xavier (N) & 0.221 & 0.005896 & 0.002337 & 0.008424 & 0.118 & 0.000067 & 0.000000702 & 0.000190 \\
\midrule
0L Xavier (U) & 0.195 & 0.002517 & 0.000868 & 0.011988 & 0.145 & 0.000169 & 0.000000728 & 0.000167 \\
1L 4U Xavier (U) & 0.202 & 0.004492 & 0.000988 & 0.019024 & 0.120 & 0.000306 & 0.000002099 & 0.000226 \\
1L 32U Xavier (U) & 0.198 & 0.005724 & 0.000623 & 0.007098 & 0.128 & 0.000088 & 0.000000685 & 0.000174 \\
1L 64U Xavier (U) & 0.216 & 0.005925 & 0.000983 & 0.004143 & 0.128 & 0.000088 & 0.000000652 & 0.000174 \\
2L 4U Xavier (U) & 0.221 & 0.004763 & 0.000774 & 0.028628 & 0.113 & 0.000402 & 0.000001439 & 0.000072 \\
2L 32U Xavier (U) & 0.212 & 0.004936 & 0.000009 & 0.012626 & 0.127 & 0.000076 & 0.000000551 & 0.000078 \\
2L 64U Xavier (U) & 0.224 & 0.005346 & 0.000092 & 0.008398 & 0.127 & 0.000076 & 0.000000589 & 0.000078 \\
\bottomrule
\end{tabular}
}
\end{small}
\end{center}
\vskip -0.15in
\caption{Estimated Policy-Optimal Information Capacity.}
\label{table:oe_arch_init}
\end{table}

\clearpage

\section{Evaluating the Dynamics and Initialization Noises}
\label{sec:appendix_dynamics_initialization}
In this section, we evaluate one of our information capacity metrics' properties: can we use PIC and POIC for tuning the noise levels in MDPs that might help to learn? To answer this question, we design the experiments to observe the correlations between each metric (POIC, optimality marginal entropy, optimality conditional entropy, PIC, reward marginal entropy, reward conditional entropy, and variance) and algorithm-based normalized scores.

For the experiments, we prepare 12 cartpole environments with 3 initialization and 4 dynamics noises. In these environments, we initialize 4-dimensional states with uniform distribution $\textit{Unif}(-u_{\text{init}}, u_{\text{init}})$, where $u_{\text{init}}$ is a tunable parameter $u_{\text{init}} \in \{0.05, 0.1. 0.15\}$, and add 1-dimensional transition noise to angular velocity with uniform distribution $\textit{Unif}(-u_{\text{dyn}}, u_{\text{dyn}})$, where $u_{\text{dyn}}$ is also a tunable parameter $u_{\text{dyn}} \in \{0.0, 0.03, 0.05, 0.1\}$.
As in Section \ref{sec:mi_on_benchmark}, we prepare the bag-of-architectures to measure PIC and POIC: ([0] layers $+$ [1, 2] layers $\times$ [4, 32, 64] hidden units) $\times$ [Gaussian prior $\mathcal{N}(0, I)$, Uniform prior $\textit{Unif}(-1,1)$, Xavier Normal, Xavier Uniform] $\times$ [w/ bias, w/o bias], and bag-of-algorithms (PPO, ES, DQN with different hyper-parameters) to compute the algorithm-based normalized scores (see \autoref{table:algorithm_results_noise} for each raw score).

The results are shown in \autoref{table:task_arch_init}. The POIC seems the best metric that positively correlates to the algorithm-based normalized scores ($R=0.860$; statistically significant with $p<0.001$), but PIC doesn't show such a trend. This suggests that POIC might be used for tuning initialization and dynamics noises that help to learn.

\begin{table}[htb]
\begin{center}
\begin{small}
\scalebox{0.90}{
\begin{tabular}{l|cc|ccccccc|c}
\toprule
Environment & Initial Noise & Dynamics Noise & $\hat{\mathcal{I}}(\mathcal{O};\Theta)$ & $\hat{\mathcal{H}}(\mathcal{O})$ & $\hat{\mathcal{H}}(\mathcal{O}|\Theta)$ & $\hat{\mathcal{I}}(R;\Theta)$ & $\hat{\mathcal{H}}(R)$ & $\hat{\mathcal{H}}(R|\Theta)$ & Variance & Score(A)\\
\midrule
\multirow{12}{*}{CartPole} & 0.05 & 0.0 & 0.207 & 0.292 & 0.085 & 1.315 & 3.632 & 2.317 & 17.586 & 0.886 \\
 & 0.05 & 0.03 & 0.172 & 0.252 & 0.080 & 1.268 & 3.696 & 2.428 & 15.540 & 0.848 \\
 & 0.05 & 0.05 & 0.151 & 0.222 & 0.071 & 1.233 & 3.723 & 2.490 & 14.024 & 0.856 \\
 & 0.05 & 0.1 & 0.111 & 0.172 & 0.061 & 1.151 & 3.746 & 2.595 & 11.302 & 0.827 \\
 & 0.1 & 0.0 & 0.140 & 0.200 & 0.060 & 1.110 & 3.823 & 2.713 & 12.102 & 0.849 \\
 & 0.1 & 0.03 & 0.130 & 0.185 & 0.055 & 1.095 & 3.843 & 2.748 & 11.430 & 0.849 \\
 & 0.1 & 0.05 & 0.120 & 0.172 & 0.052 & 1.076 & 3.853 & 2.776 & 10.826 & 0.847 \\
 & 0.1 & 0.1 & 0.093 & 0.143 & 0.050 & 1.019 & 3.866 & 2.847 & 9.343 & 0.820 \\
 & 0.15 & 0.0 & 0.110 & 0.160 & 0.050 & 0.922 & 3.886 & 2.964 & 9.710 & 0.850 \\
 & 0.15 & 0.03 & 0.105 & 0.150 & 0.046 & 0.913 & 3.895 & 2.981 & 9.326 & 0.848 \\
 & 0.15 & 0.05 & 0.098 & 0.142 & 0.044 & 0.901 & 3.899 & 2.999 & 8.951 & 0.828 \\
 & 0.15 & 0.1 & 0.077 & 0.120 & 0.043 & 0.859 & 3.902 & 3.042 & 7.891 & 0.824 \\
\midrule
\multicolumn{3}{l|}{Correlation Coefficient: Score(A)} & 0.860 & 0.824 & 0.698 & 0.613 & -0.625 & -0.625 & 0.783 & -- \\
\multicolumn{3}{l|}{p-Value: Score(A)} & $>$0.001 & $>$0.001 & 0.012 & 0.034 & 0.030 & 0.030 & 0.003 & -- \\
\bottomrule
\end{tabular}
}
\end{small}
\end{center}
\vskip -0.15in
\caption{PIC and POIC under noisy initialization and dynamics.}
\label{table:task_arch_init}
\end{table}

\begin{landscape}
\begin{table}[p]
\setlength{\tabcolsep}{3pt}
\begin{center}
\begin{small}
\scalebox{0.95}{
\begin{tabular}{l|c|cccccccccccc}
\toprule
\multicolumn{2}{l|}{} & \multicolumn{12}{c}{(Initialization Noise, Dynamics Noise)} \\
\midrule
Algorithm & Hyper-Parameters & (0.05, 0.0) & (0.05, 0.03) & (0.05, 0.05) & (0.05, 0.1) & (0.1, 0.0) & (0.1, 0.03) & (0.1, 0.05) & (0.1, 0.1) & (0.15, 0.0) & (0.15, 0.03) & (0.15, 0.05) & (0.15, 0.1) \\
\midrule
PPO & (64, 64), $\gamma=0.995$ & 200.0 & 200.0 & 200.0 & 200.0 & 200.0 & 200.0 & 200.0 & 200.0 & 199.6 & 200.0 & 197.8 & 199.4 \\
PPO & (64, 64), $\gamma=0.99$ & 200.0 & 199.7 & 200.0 & 200.0 & 200.0 & 199.2 & 196.5 & 200.0 & 200.0 & 200.0 & 196.2 & 197.2 \\
PPO & (64, 64), $\gamma=0.95$ & 194.0 & 200.0 & 197.2 & 192.5 & 193.7 & 196.9 & 182.4 & 193.2 & 189.8 & 186.5 & 200.0 & 193.5 \\
PPO & (128, 64), $\gamma=0.995$ & 200.0 & 200.0 & 200.0 & 197.2 & 200.0 & 200.0 & 200.0 & 200.0 & 199.6 & 199.9 & 200.0 & 199.9 \\
PPO & (128, 64), $\gamma=0.99$ & 200.0 & 200.0 & 200.0 & 200.0 & 200.0 & 200.0 & 193.1 & 198.5 & 200.0 & 196.8 & 198.3 & 193.8 \\
PPO & (128, 64), $\gamma=0.95$ & 198.1 & 196.1 & 199.3 & 192.2 & 196.1 & 191.0 & 195.3 & 196.1 & 197.4 & 182.5 & 189.3 & 184.9 \\
PPO & (128, 128), $\gamma=0.995$ & 200.0 & 199.6 & 200.0 & 200.0 & 197.8 & 200.0 & 200.0 & 199.7 & 191.1 & 196.5 & 200.0 & 195.9 \\
PPO & (128, 128), $\gamma=0.99$ & 200.0 & 200.0 & 200.0 & 200.0 & 200.0 & 200.0 & 197.8 & 200.0 & 198.6 & 200.0 & 199.3 & 199.4 \\
PPO & (128, 128), $\gamma=0.95$ & 197.1 & 195.2 & 196.4 & 197.5 & 192.9 & 193.4 & 196.5 & 191.7 & 195.6 & 193.4 & 191.2 & 185.8 \\
\midrule
ES & (16, 16), $\sigma=0.1$ & 188.1 & 165.5 & 166.5 & 147.1 & 185.3 & 196.3 & 179.1 & 130.2 & 170.3 & 167.2 & 155.1 & 144.6 \\
ES & (16, 16), $\sigma=0.1$, rand & 197.1 & 161.7 & 175.1 & 175.9 & 186.4 & 170.7 & 185.8 & 140.9 & 167.7 & 173.0 & 159.8 & 167.3 \\
ES & (64, 64), $\sigma=0.1$ & 168.6 & 159.6 & 187.2 & 170.0 & 160.9 & 150.2 & 171.9 & 115.7 & 173.3 & 150.0 & 170.4 & 141.5 \\
ES & (64, 64), $\sigma=0.1$, rand & 168.0 & 177.0 & 170.8 & 112.7 & 161.4 & 174.3 & 169.4 & 131.1 & 173.4 & 166.6 & 138.9 & 128.2 \\
\midrule
DQN & (100, 100), $\gamma=0.95$ & 133.2 & 127.1 & 131.4 & 131.5 & 123.4 & 102.5 & 87.3 & 128.8 & 116.0 & 134.8 & 130.7 & 117.8 \\
DQN & (100, 100), $\gamma=0.99$ & 185.2 & 169.6 & 114.7 & 145.0 & 163.1 & 174.3 & 177.2 & 193.1 & 169.0 & 157.8 & 163.6 & 171.4 \\
DQN & (200, 200), $\gamma=0.95$ & 132.2 & 97.2 & 73.9 & 118.6 & 96.2 & 135.6 & 113.0 & 128.8 & 106.0 & 107.3 & 117.0 & 122.6 \\
DQN & (200, 200), $\gamma=0.99$ & 195.5 & 196.0 & 179.7 & 194.9 & 195.9 & 149.4 & 172.2 & 182.4 & 184.8 & 183.2 & 167.3 & 189.4 \\
DQN & (50, 50), $\gamma=0.95$ & 161.3 & 123.5 & 166.4 & 91.5 & 98.6 & 101.3 & 156.2 & 113.3 & 132.4 & 127.5 & 121.0 & 90.2 \\
DQN & (50, 50), $\gamma=0.99$ & 176.3 & 153.7 & 169.4 & 172.8 & 170.0 & 179.8 & 163.6 & 137.0 & 168.8 & 155.2 & 150.8 & 167.9 \\
DQN & (50, 50, 50), $\gamma=0.95$ & 87.7 & 102.7 & 103.0 & 101.6 & 84.4 & 114.5 & 64.8 & 134.2 & 108.1 & 119.4 & 93.1 & 111.0 \\
DQN & (50, 50, 50), $\gamma=0.99$ & 159.8 & 164.1 & 192.6 & 164.7 & 185.6 & 163.3 & 180.2 & 160.5 & 153.9 & 188.6 & 165.2 & 185.3 \\
\midrule
$r_{\text{ave}}^{\text{algo}}$ & -- & 178.2 & 170.9 & 172.6 & 166.9 & 171.0 & 171.1 & 170.6 & 165.5 & 171.2 & 170.8 & 166.9 & 166.0 \\
\midrule
$r_{\max}^{\text{algo}}$ & -- & 200.0 & 200.0 & 200.0 & 200.0 & 200.0 & 200.0 & 200.0 & 200.0 & 200.0 & 200.0 & 200.0 & 199.9 \\
\midrule
$r_{\min}^{\text{rand}}$ & -- & 8.742 & 8.752 & 8.759 & 8.789 & 8.137 & 8.139 & 8.161 & 8.184 & 7.479 & 7.482 & 7.482 & 7.476 \\
\bottomrule
\end{tabular}
}
\end{small}
\end{center}
\vskip -0.15in
\caption{Performance of a variety of algorithms in CartPole with initialization and dynamics noises. The results are averaged over 5 random seeds. We change 2 hyper-parameters, architecture of neural networks and discount factor $\gamma$.}
\label{table:algorithm_results_noise}
\end{table}
\end{landscape}

\clearpage

\section{Our Information Capacity Metrics based on Channel Capacity}
In previous works~\citep{Klyubin2005empower,tishby2011information}, \textit{empowerment} is originally defined with channel capacity between future state and n-step action sequence, instead of mutual information.

Following these definition, we extensively evaluate the variants of our PIC and POIC, based on \textit{channel capacity}:
\begin{equation}
\begin{split}
&\text{Policy Information Capacity: }~\underset{p(\theta)}{\max}~\mathcal{I}(R; \Theta), \\
&\text{Policy-Optimal Information Capacity: }~\underset{p(\theta)}{\max}~\mathcal{I}(\mathcal{O}; \Theta).
\end{split}
\label{eq:channel_capacity}
\end{equation}

The experimental settings are similar to Section~\ref{sec:mi_on_benchmark}; we investigate the correlation between channel-capacity-based information capacity metrics, and algorithm- or random-sampling-based normalized score.
To take the maximum over parameter distribution in~\autoref{eq:channel_capacity}, we prepare a ``bag-of-policy-architectures'' in practice: ([0] layers $+$ [1, 2] layers $\times$ [4, 32, 64] hidden units) $\times$ [Gaussian prior $\mathcal{N}(0, I)$, Uniform prior $\textit{Unif}(-1,1)$, Xavier Normal, Xavier Uniform] $\times$ [w/ bias, w/o bias], which amounts to 56 different parameter distributions $p(\theta)$.

The results are shown in \autoref{fig:peason_correlation_coefficient_channel_capacity} (also shown in~\autoref{fig:peason_correlation_coefficient_channel_capacity_sub} and \autoref{table:channel_capacity_full}).
POIC seems to positively correlate to the algorithm-based normalized score ($R=0.707$; statistically significant with $p<0.01$), which can be regarded as the more realistic of the two task difficulty scores.
On the other hand, PIC also shows a weak positive correlation with the random-sampling-based normalized score ($R=0.421$).
These results show the similar trends with our definition of PIC and POIC, based on mutual information.

\begin{figure*}[ht]
\centering
\includegraphics[width=\linewidth]{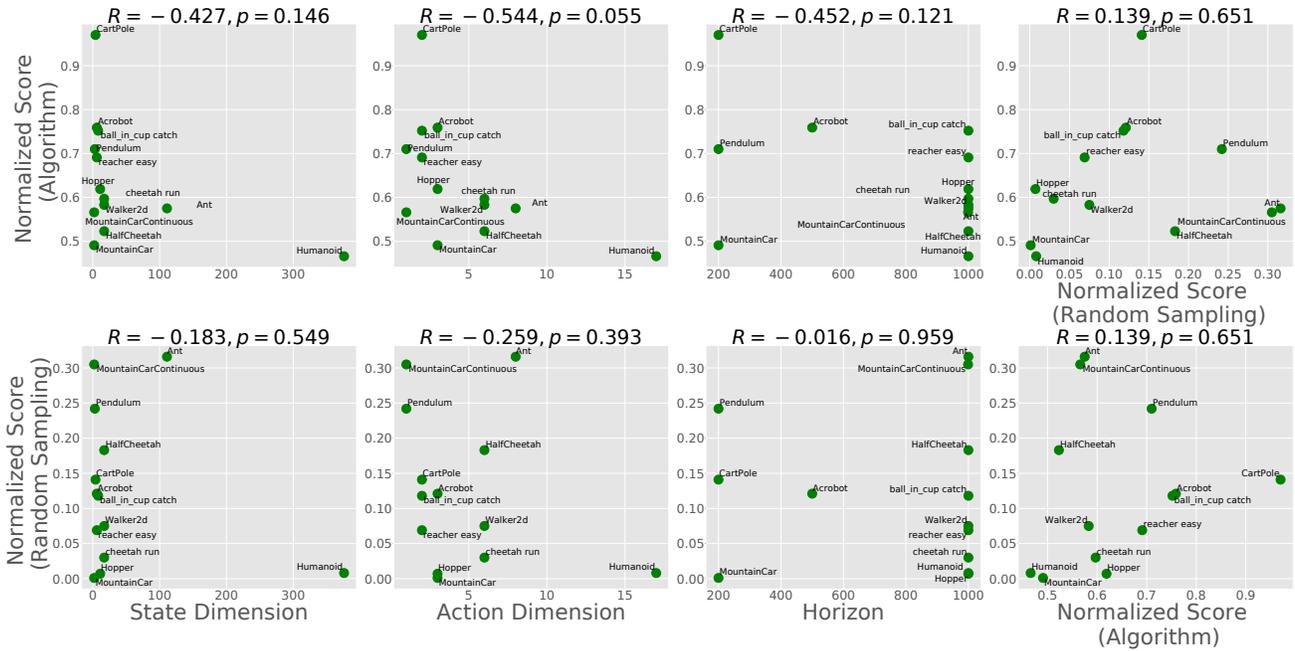}  
\vskip -0.05in
\caption{\textbf{[Channel Capacity version]}: 2D-Scatter plots between each metric (State dimension, action dimension, episode horizon, and normalized scores; x-axis) and each normalized score (algorithm-based (top) and random-sampling-based (bottom); y-axis).}
\label{fig:peason_correlation_coefficient_channel_capacity_sub}
\end{figure*}

\begin{landscape}
\begin{figure*}[p]
\centering
\includegraphics[width=\linewidth]{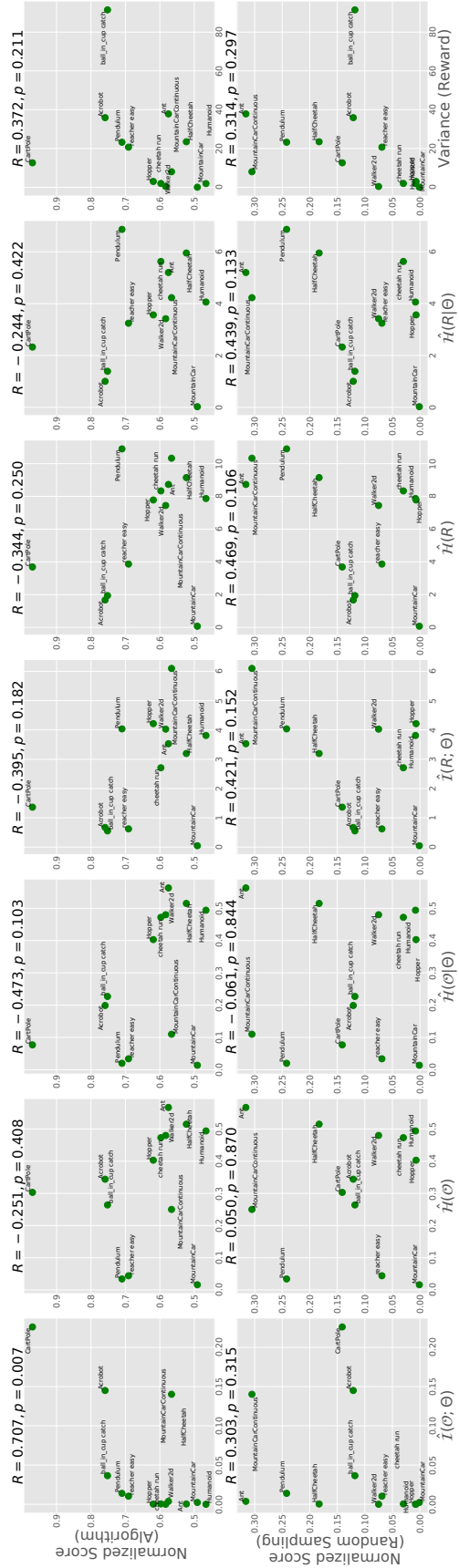}  
\vskip -0.05in
\caption{\textbf{[Channel Capacity version]}: 2D-Scatter plots between each metric (x-axis) and normalized scores (algorithm-based (top) and random-sampling-based (bottom) ; y-axis); see~\autoref{table:channel_capacity_full} for the details. Variance (last column) approximately corresponds to the metric proposed by~\citet{oller2020analyzing}.
POIC positively correlates with algorithm-based normalized score ($R=0.707$; statistically significant with $p<0.01$), the more realistic of the two task difficulty scores.}
\label{fig:peason_correlation_coefficient_channel_capacity}
\end{figure*}
\end{landscape}

\begin{landscape}
\begin{table}[p]
\setlength{\tabcolsep}{3pt}
\begin{center}
\begin{small}
\begin{tabular}{l|c|c|cccccccccc}
\toprule
Environment & Score(A) & Score(R) & $\hat{\mathcal{I}}(\mathcal{O};\Theta)$ & $\hat{\mathcal{H}}(\mathcal{O})$ & $\hat{\mathcal{H}}(\mathcal{O}|\Theta)$ & $\hat{\mathcal{I}}(R;\Theta)$ & $\hat{\mathcal{H}}(R)$ & $\hat{\mathcal{H}}(R|\Theta)$ & Variance & State dim & Action dim & Horizon \\
\midrule
CartPole & 0.970 & 0.141 & 0.225996 & 0.303 & 0.077 & 1.363 & 3.687 & 2.324 & 12.610 & 4 & 2 & 200 \\
Pendulum & 0.710 & 0.242 & 0.013944 & 0.034 & 0.020 & 4.036 & 10.901 & 6.865 & 23.223 & 3 & 1 & 200 \\
MountainCar & 0.491 & 0.001 & 0.002337 & 0.016 & 0.014 & 0.047 & 0.070 & 0.023 & 0.054 & 2 & 3 & 200 \\
MountainCarContinuous & 0.566 & 0.305 & 0.140193 & 0.250 & 0.110 & 6.100 & 10.331 & 4.230 & 8.022 & 2 & 1 & 999 \\
Acrobot & 0.759 & 0.121 & 0.145020 & 0.344 & 0.199 & 0.670 & 1.672 & 1.002 & 35.891 & 6 & 3 & 500 \\
Ant & 0.575 & 0.316 & 0.003742 & 0.567 & 0.563 & 3.529 & 8.730 & 5.201 & 37.830 & 111 & 8 & 1000 \\
HalfCheetah & 0.523 & 0.183 & 0.000402 & 0.515 & 0.515 & 3.195 & 9.148 & 5.953 & 23.468 & 17 & 6 & 1000 \\
Hopper & 0.619 & 0.007 & 0.000226 & 0.403 & 0.403 & 4.214 & 7.780 & 3.565 & 3.006 & 11 & 3 & 1000 \\
Walker2d & 0.583 & 0.075 & 5.1e-05 & 0.480 & 0.480 & 4.027 & 7.445 & 3.418 & 0.403 & 17 & 6 & 1000 \\
Humanoid & 0.466 & 0.008 & 2.9e-05 & 0.494 & 0.494 & 3.810 & 7.870 & 4.061 & 1.917 & 376 & 17 & 1000 \\
cheetah run & 0.597 & 0.030 & 0.000483 & 0.473 & 0.472 & 2.708 & 8.335 & 5.627 & 1.972 & 17 & 6 & 1000 \\
reacher easy & 0.691 & 0.069 & 0.010335 & 0.044 & 0.034 & 0.622 & 3.866 & 3.244 & 20.669 & 6 & 2 & 1000 \\
ball\_in\_cup catch & 0.752 & 0.118 & 0.036467 & 0.264 & 0.227 & 0.553 & 1.942 & 1.389 & 91.454 & 8 & 2 & 1000 \\
\midrule
Correlation Coefficient: Score(A) & -- & 0.139 & 0.707 & -0.251 & -0.473 & -0.395 & -0.344 & -0.244 & 0.372 & -0.427 & -0.544 & -0.452 \\
p-Value: Score(A) & -- & 0.651 & 0.007 & 0.408 & 0.103 & 0.182 & 0.250 & 0.422 & 0.211 & 0.146 & 0.055 & 0.121 \\
Correlation Coefficient: Score(R) & 0.139 & -- & 0.303 & 0.050 & -0.061 & 0.421 & 0.469 & 0.439 & 0.314 & -0.183 & -0.259 & -0.016 \\
p-Value: Score(R) & 0.651 & -- & 0.315 & 0.870 & 0.844 & 0.152 & 0.106 & 0.133 & 0.297 & 0.549 & 0.393 & 0.959 \\
\bottomrule
\end{tabular}
\end{small}
\end{center}
\vskip -0.15in
\caption{\textbf{[Channel Capacity version]}: POIC $\hat{\mathcal{I}}(\mathcal{O};\Theta)$, optimality marginal entropy $\hat{\mathcal{H}}(\mathcal{O})$, optimality conditional entropy $\hat{\mathcal{H}}(\mathcal{O}|\Theta)$, PIC $\hat{\mathcal{I}}(R;\Theta)$, reward marginal entropy $\hat{\mathcal{H}}(R)$, reward conditional entropy $\hat{\mathcal{H}}(R|\Theta)$, normalized variance of return in standard RL benchmark environments, Pearson correlation coefficient to the algorithm-based normalized score~(Score(A)), and the random-sampling-based normalized score~(Score(R)). We prepare the bags of policy architectures, 56 variants in total; ([0] layers $+$ [1, 2] layers $\times$ [4, 32, 64] units) $\times$ [Gaussian prior $\mathcal{N}(0, I)$, Uniform prior $\textit{Unif}(-1,1)$, Xavier Normal, Xavier Uniform] $\times$ [w/ bias, w/o bias]. Since we consider the definition based on channel capacity here, we take the maximum value of each information capacity metrics over these policy architectures. The results suggest that POIC seems to positively correlate with algorithm-based normalized score ($R=0.707$; statistically significant with $p<0.01$) better than all other alternatives including variance of returns~\citep{oller2020analyzing}.}
\label{table:channel_capacity_full}
\end{table}
\end{landscape}

\end{document}